\documentclass[10pt]{article}
\usepackage[preprint]{tmlr}
\setcitestyle{numbers,square,sort&compress}

\usepackage{amsmath,amsfonts,bm}

\def\eqref#1{equation~\ref{#1}}

\def\1{\bm{1}}

\DeclareMathAlphabet{\mathsfit}{\encodingdefault}{\sfdefault}{m}{sl}
\SetMathAlphabet{\mathsfit}{bold}{\encodingdefault}{\sfdefault}{bx}{n}

\newcommand{\Cov}{\mathrm{Cov}}

\usepackage{hyperref}
\usepackage{url}
\usepackage{microtype}
\usepackage{graphicx}
\usepackage{booktabs}
\usepackage{tabularx}
\usepackage{longtable}
\usepackage{array}
\usepackage{makecell}
\usepackage{enumitem}

\usepackage{amsmath,amssymb,amsthm}
\usepackage{mathtools}
\usepackage{mathrsfs}
\usepackage{bbm}
\usepackage{bm}

\usepackage{algorithm}
\usepackage{algcompatible}

\usepackage{subfig}
\usepackage{float}
\usepackage{csvsimple}
\usepackage{colortbl}
\usepackage[dvipsnames]{xcolor}
\usepackage{pifont}

\usepackage[english]{babel}
\usepackage[autostyle, english=american]{csquotes}

\hypersetup{hidelinks}

\DeclareMathOperator{\rand}{rand}

\usepackage[capitalize,noabbrev]{cleveref}

\theoremstyle{plain}
\newtheorem{theorem}{Theorem}[section]
\newtheorem{proposition}[theorem]{Proposition}
\newtheorem{lemma}[theorem]{Lemma}

\theoremstyle{definition}

\newtheorem{exam}[theorem]{Example}

\theoremstyle{remark}

\title{Separation-Utility Pareto Frontier:\\ An Information-Theoretic Characterization}

\author{Shizhou Xu\\
Department of Mathematics\\
University of California Davis\\
Davis, CA 95616, USA \\}

\def\month{MM}  
\def\year{YYYY} 
\def\openreview{\url{https://openreview.net/forum?id=XXXX}} 

\begin{document}

\maketitle

\begin{abstract}
We study the Pareto frontier between predictive utility and separation, a fairness criterion requiring predictive independence from sensitive attributes conditional on the true outcome. Through an information-theoretic lens, we characterize the achievable separation–utility region, prove that the revealed randomized frontier is the concave closure of the deterministic frontier, and clarify how the marginal cost of separation varies along the frontier. We further identify sufficient conditions under which the trade-off is strict, providing theoretical guidance for interpreting empirical frontiers and choosing operating points. Motivated by this characterization, we develop a direct empirical regularizer based on conditional mutual information (CMI) for discrete target and sensitive variables. By estimating CMI directly from sample statistics, the resulting plug-in regularizer avoids reliance on adversarial or variational proxy losses, is compatible with deep models trained by gradient-based optimization, and provides a scalar monitor of residual separation violation with finite-sample guarantees. Experiments on COMPAS, UCI Adult, UCI Bank, CelebA, and ACS show that the proposed method traces stable separation-utility frontiers, substantially reduces separation violations, and achieves competitive trade-offs relative to established baselines, often improving the low-violation region. This study offers a principled, stable, and flexible framework for navigating separation–utility trade-offs in deep learning.
\end{abstract}

\section{Introduction}\label{sec:introduction}

Automated decision systems are increasingly deployed in high-stakes domains such as finance, criminal justice, hiring, and healthcare, making verifiable fairness constraints an essential component of reliable machine learning \cite{barocas2016big}. A central challenge is that fairness criteria often conflict with model utility \cite{chouldechova2017fair,kleinberg2016inherent}, creating a critical need for a framework that provides practitioners with a provably optimal and practically implementable trade-off.

In this work, we focus on \textbf{separation} (or equalized odds in binary classification), which requires conditional independence between the model output $\hat{Y}$ and a sensitive attribute $Z$ given the true label $Y$ \cite{hardt2016equality}:
\[
\hat{Y}\ \perp\ Z\ \mid\ Y.
\]
Any dependence of $\hat{Y}$ on $Z$ must be justified by the overlapping information between $Z$ and the target $Y$. Separation is particularly relevant when base rates differ across groups, as it rules out unjustified group-dependent errors while remaining compatible with perfect prediction.\footnote{We adopt separation as the fairness definition throughout. Comparing separation to alternative notions (e.g., independence, calibration) is beyond the scope of this paper.}

\textbf{Information-plane view and the Pareto frontier.}
We take an information-theoretic perspective on the separation-utility trade-off by quantifying (i) separation violation $v$ by the conditional mutual information (CMI), and (ii) predictive utility $u$ by mutual information (MI):
\[
v(\hat{Y}) \ :=\ I(\hat{Y};Z\mid Y), \qquad u(\hat{Y}) \ :=\ I(\hat{Y};Y).
\]
Here, $v$ characterizes separation at $v(\hat{Y}) = 0$, whereas $u$ is equivalent to minimizing the Bayes-optimal conditional log-loss, yields Fano-type necessary conditions for small classification error in finite-label settings, and connects to general distortion losses through rate-distortion theory \cite{CoverThomas2006}.

\noindent\textbf{Novelty relative to prior CMI-based fairness work.}
While the equivalence $I(\hat{Y};Z\mid Y)=0 \iff \hat{Y}\perp Z\mid Y$ follows directly from the definition of conditional mutual information, our contributions are (i) a \emph{population-level characterization} of the achievable separation-utility region and the \emph{optimal randomized Pareto frontier}, showing that revealed randomization is necessary in general and that mixing between at most two deterministic predictors suffices to attain the randomized frontier, (ii) conditions under which the trade-off becomes \emph{strict}, and (iii) an empirical CMI regularizer with finite-sample monitoring guarantees that reliably traces the frontier in practice.

This formulation yields a model-agnostic feasibility region on the information plane:
$$\mathcal{R}\ :=\ \bigl\{(v(\hat{Y}),u(\hat{Y})):\ \hat{Y}\ \text{feasible from a predictor}\bigr\}.$$
Crucially, this \emph{frontier characterization} is distributional and holds for arbitrary laws of $(X,Y,Z)$ on Polish spaces, and for predictors \(U\) satisfying the measurability and finiteness conditions stated in Section~\ref{sec:theory-regularization}. In particular, it applies to continuous, discrete, or mixed variables under these conditions. Our first goal is to characterize the shape of the Pareto frontier of $\mathcal{R}$, thereby establishing the ``price of separation'' in terms of utility on this information plane.

\textbf{From theory to practice: An empirical regularizer.}
While our theoretical frontier characterization is general, practical optimization requires estimating MI and CMI. In continuous settings, this often necessitates complex auxiliary density models or variational bounds due to the inherent statistical challenges of density estimation in high-dimensional spaces. However, we observe that for the widespread case of discrete fairness tasks, ranging from standard classification to high-cardinality output spaces, such complexity is unnecessary. To bridge the gap between theory and practice, we propose a simple in-processing regularizer that directly penalizes the empirical plug-in estimator of CMI. Unlike learning-based proxies, this sample-statistics approach is stable, statistically interpretable, traces the theoretical Pareto frontier more robustly, and seamlessly scales to multi-class and multi-group settings without the optimization overhead of more sophisticated baselines, as we demonstrate empirically in Section \ref{sec:experiments}.

\subsection{Related Works}\label{subsec:separation}
Existing approaches for enforcing separation generally fall into three paradigms, each with distinct limitations regarding stability and scalability.

\textbf{Rate-constraints and reductions.} 
Reductions convert fairness constraints into cost-sensitive classification problems \cite{agarwal2018reductions,cotter2019wellgen}, while convex frameworks employ covariance-based surrogates \cite{celis2019classification,zafar2017}. Although offering provable guarantees, they typically match only lower-order moments (mostly due to their focus on binary classification for which lower-order moments are sufficient) rather than the full distribution, and constraint complexity scales linearly with the product of target and sensitive feature cardinality, limiting scalability.

\textbf{Adversarial and information-theoretic proxies.} 
Adversarial debiasing uses min-max games to infer $Z$ from predictions \citep{zhang2018mitigating,madras2018learning}. Similarly, information-theoretic approaches penalize Conditional Mutual Information (CMI) via learning-based approximations, such as auxiliary density models \citep{Steinberg2020} or variational bounds \citep{roh2020frtrain,kang2022infofair}. While popular, these methods suffer from lack of theoretical guarantee, optimization instability (e.g., vanishing gradients, non-convexity), and reliance on the expressivity of the adversarial or auxiliary model, often introducing estimation errors \citep{song2020learningcontrollablefairrepresentations}.

\textbf{Reweighting and robustness.} 
This line of works enforce fairness through external interventions on the data. \emph{Distributionally Robust Optimization (DRO)}~\cite{sagawa2020distributionally} targets the \emph{worst-performing} subgroup, dynamically prioritizing high-loss regions. In contrast, \emph{resampling} techniques~\cite{kamiran2012data} address group \emph{imbalance} via static re-weighting, while \emph{data transformation} methods~\cite{romano2021fair} introduce randomized variables or modify features to break statistical dependencies. Although model-agnostic, these methods can be unstable or degrade utility due to their aggressive manipulation of the training distribution.

\subsection{Contributions}\label{subsec:contributions}
Despite recent progress, three critical gaps remain. First, while fairness-utility trade-offs are well-documented in binary classification settings \citep{pleiss2017fairness,chouldechova2017fair,corbett-davies2017algorithmic,menon2018cost}, the field lacks a general model-agnostic characterization of Pareto frontier beyond binary classification setting, particularly on its existence, shape, and finite-sample feasibility. Second, applicability is limited: most methods are tailored to binary tasks \citep{hardt2016equality,agarwal2018reductions}, becoming intractable in modern high-cardinality settings (e.g., LLMs) where constraint complexity explodes \citep{gallegos2024bias,chu2024taxonomic,blodgett2020language}. Finally, training lacks transparent guarantees: standard proxies (adversaries, variational bounds) neither directly certify separation \citep{zhang2018mitigating,madras2018learning,roh2020frtrain} nor ensure optimization stability \citep{agarwal2018reductions,namkoong2017variance,sagawa2020distributionally}.

We address these challenges through practice-inspired theory, theory-guided method, and empirical verification:

\begin{itemize}[leftmargin=*]
    \item \textbf{Model-agnostic theoretical characterization of the Pareto frontier.}
    We provide a general, model-agnostic theoretical foundation by characterizing the separation-utility region on an information-theoretic plane. We prove that the optimal randomized frontier equals the concave closure of the deterministic frontier (Theorem~\ref{thm:timesharing-frontier-polish}). This result holds for continuous, discrete, or mixed variables under the measurability and finiteness assumptions stated in Section~\ref{sec:theory-regularization}, clarifying which trade-offs are information-theoretically feasible independent of any optimization surrogate. This provides a theoretical common ground for existing and future work that targets separation via information-theoretic regularization, validating the objective regardless of how it is estimated.

    \item \textbf{CMI as a principled separation quantification.}
    We show that $I(\hat{Y};Z\mid Y)$ provides an exact characterization of separation and a uniform upper bound controlling the conditional statistical dependence of \emph{all} bounded functionals of the prediction and sensitive attribute (Theorem~\ref{th:CMI_bound}). This positions CMI as a principled scalar quantification for separation violation.

    \item \textbf{Direct empirical CMI regularization.}
    Bridging the gap between optimization and theoretical guarantees, we demonstrate that complex learning-based proxies are often unnecessary when the target and sensitive variables are discrete, even in high-cardinality settings. For discrete separation tasks, the target information-theoretic quantity admits a direct plug-in estimator derived from empirical frequencies. This enables transparent statistical guarantees under standard i.i.d. assumptions (Proposition \ref{prop:estimation_error}), while avoiding the optimization instability of adversarial min-max optimization and the auxiliary estimation errors introduced by variational bounds.

    \item \textbf{Empirical verification.}
    We validate the proposed empirical CMI regularizer on standard fairness benchmarks and a multi-class/multi-group task derived from the ACS dataset. Across datasets, it traces stable empirical separation-utility frontiers and often improves over reduction-based, adversarial, robustness-based, and learning-based proxy methods in both the information plane and deployment metrics such as AUROC and Accuracy. The experiments further connect to the theoretical results: on Adult, utility beyond the empirical $X$-only level is attained only at strictly positive separation violation, illustrating the necessary-tradeoff result in Theorem~\ref{thm:necessary-tradeoff}. Furthermore, the ACS experiment confirms that the same CMI formulation extends naturally beyond binary targets and binary sensitive attributes to discrete variables with large cardinalities.
    \footnote{Code available at \url{https://github.com/xushizhou/Separation-Utility-Pareto-Frontier}.}
\end{itemize}

\section{Methods: Regularization with Theoretical Guarantees}
\label{sec:theory-regularization}

We present a principled regularization framework for enforcing \emph{separation} during training. To maintain consistency and clarity, we first introduce the setting and notation, then develop the theoretical results in the following order:
\textbf{Section \ref{sub:utility-sep-tradeoff}}: The characterization of the separation-utility Pareto frontier;
\textbf{Section \ref{sub:CMI}}: The justification for Conditional Mutual Information (CMI) as a rigorous quantification of separation violation;
\textbf{Section \ref{sub:when-tradeoff}}: The conditions under which the trade-off becomes strict.

\textbf{Setting and notation.}
Let $(\mathcal X,\mathcal{B}_{\mathcal X})$, $(\mathcal Y,\mathcal{B}_{\mathcal Y})$, and $(\mathcal Z,\mathcal{B}_{\mathcal Z})$ be Polish spaces equipped with Borel $\sigma$-algebras. The random variables $(X,Y,Z)$ are jointly distributed according to $P_{XYZ}$. Since $(\mathcal X,\mathcal B_{\mathcal X})$, $(\mathcal Y,\mathcal B_{\mathcal Y})$, and $(\mathcal Z,\mathcal B_{\mathcal Z})$ are Polish spaces equipped with their Borel $\sigma$-algebras, the relevant regular conditional laws exist. Here, $X$, $Y$, and $Z$ denote the input, target, and sensitive variables respectively.

A deterministic predictor is a measurable map $f:\mathcal X\times\mathcal Z\to\mathcal{Y}$, producing the output $\hat{Y}=f(X,Z)$.
A randomized predictor is defined by an external independent noise variable $N$ (with $N \!\perp (X,Y,Z)$) and a measurable map $f$, producing the augmented variable $\widetilde{U} = (N, \hat{Y})$ where $\hat{Y} = f(X,Z,N)$.

To simplify notation, we use a generic prediction $U$ to represent either $\hat{Y}$ or $\widetilde{U}$, with the generic space $\mathcal{U}$ representing either the outcome space $\mathcal{Y}$ or the product space $\mathcal{N} \times \mathcal{Y}$. We define \emph{predictive utility} ($u$) and \emph{separation violation} ($v$) for the generic predictor variable $U$ as:
\[
u(U) := I(U;Y), \qquad v(U) := I(U;Z\mid Y).
\]
To keep the analysis clean and straightforward, we restrict attention to admissible predictors $U$ for which
\[
I(U;Y)<\infty,\qquad I(U;Z\mid Y)<\infty.
\]
All mutual information and conditional mutual information quantities are interpreted in the standard measure-theoretic sense via relative entropy. Whenever entropy quantities such as $H(Y)$ or $H(Z\mid Y)$ are used later, they are understood only in the corresponding discrete setting.

\subsection{The Separation-Utility Pareto Frontier}
\label{sub:utility-sep-tradeoff}

We begin by characterizing the set of feasible (violation, utility) pairs. This characterizes the theoretical limit of what is achievable. Define the deterministic attainable set $\mathcal{S}_{\det}$ and the randomized attainable set $\mathcal{S}_{\rand}$ as follows:
\begin{itemize}[leftmargin = *]
    \item Let $u_f:=I(f(X,Z);Y)$ and $v_f:=I(f(X,Z);Z\mid Y)$. The deterministic set is:
    \[ 
        \mathcal{S}_{\det} \;:=\; \bigl\{(v_f, u_f) : f \text{ is measurable and } I(f(X,Z);Y),\,I(f(X,Z);Z\mid Y)<\infty\bigr\} \;\subset\; \mathbb{R}_{\ge 0}^2.
    \]
    \item Let $\widetilde U := (N, f_N(X,Z))$ be a revealed randomized predictor where $N \!\perp\! (X,Y,Z)$. The randomized set is:
    \[ 
        \mathcal{S}_{\rand} \;:=\; \bigl\{(I(\widetilde U;Z \mid Y), \, I(\widetilde U;Y)) : \widetilde U \text{ is randomized and } I(\widetilde U;Y),\,I(\widetilde U;Z\mid Y)<\infty\bigr\}.
    \]
\end{itemize}
We define the corresponding Pareto frontiers $U^\star_{\det}(v)$ and $U^\star_{\rand}(v)$ as the (closure of the) maximum utility attainable for a given violation constraint $v$:
\begin{align*}
    U^\star_{\det}(v)
    &:=
    \sup_{\substack{(v',u')\in \overline{\mathcal{S}_{\det}}\\ v'\le v}}
    u',
    \qquad
    U^\star_{\rand}(v)
    :=
    \sup_{\substack{(v',u')\in \overline{\mathcal{S}_{\rand}}\\ v'\le v}}
    u'.
\end{align*}
Similarly, for any attainable set 
\(\mathcal S\subseteq \mathbb R_{\ge 0}^2\), we define the upper frontier as
\[
    U_{\overline{\mathcal S}}(v)
    :=
    \sup\Bigl\{
        u' : (v',u')\in \overline{\mathcal S},\ v'\le v
    \Bigr\}.
\]
Then
\[
    U^\star_{\det}(v)=U_{\overline{\mathcal S_{\det}}}(v),
    \qquad
    U^\star_{\rand}(v)=U_{\overline{\mathcal S_{\rand}}}(v).
\]
Before analyzing the effect of randomization and the shape of the randomized Pareto frontier, we first note a basic order property of the deterministic frontier.
\begin{proposition}[Deterministic frontier is non-decreasing]\label{prop:det-monotone}
For any \(0\le v_1\le v_2\), we have
\[
U^\star_{\det}(v_1)\le U^\star_{\det}(v_2).
\]
\end{proposition}
See proof in Appendix \ref{append:det-monotone}. This implies that decreasing separation violation can potentially harm utility, but when only deterministic predictors are allowed, the optimal trade-off need not be concave and may fail to be smooth.

Now, we prove that the randomized frontier is concave:
\begin{theorem}[Randomized frontier equals the concave closure]
\label{thm:timesharing-frontier-polish}
The set $\mathcal{S}_{\rand}$ contains the convex hull of $\mathcal{S}_{\det}$ and is contained in its closure:
\begin{equation}\label{eq:Sstar-equals-closure-conv-polish}
\mathrm{conv}\bigl(\mathcal{S}_{\det}\bigr)\;\subset\;\mathcal{S}_{\rand}\;\subset\;\overline{\mathrm{conv}}\bigl(\mathcal{S}_{\det}\bigr).
\end{equation}
Consequently, the randomized upper frontier is the concave closure of the deterministic frontier:
\begin{equation}\label{eq:frontier-equals-concave-envelope-polish}
U^\star_{\rand}(v) = U_{\overline{\operatorname{conv}}(\mathcal{S}_{\det})}(v), \qquad \forall v\ge 0.
\end{equation}
\end{theorem}
See a proof in Appendix \ref{append:timesharing-frontier-polish}.
The result shows two key messages: (1) The randomized Pareto frontier is concave, meaning the \textit{marginal cost} of reducing separation, when the frontier is differentiable, is non-decreasing in terms of utility loss. (2) Any point on the randomized upper frontier can be achieved exactly when it lies on a line segment of $\operatorname{conv}(\mathcal S_{\det})$, and in general can be approximated arbitrarily well by randomizing between at most two deterministic predictors (e.g., via Bernoulli randomization).

\subsection{Information-Theoretic Quantification of Separation Violation}
\label{sub:CMI}

Having characterized the frontier on the information plane, we now justify why conditional mutual information (CMI) is the correct quantification for separation violation ($v$).

To start, we formalize the well-known result that zero CMI characterizes the separation definition:
\begin{proposition}[CMI characterizes separation]\label{prop:CMI_characterization}
\[
I(U;Z\mid Y)=0 \quad\Longleftrightarrow\quad U\perp Z\mid Y.
\]
\end{proposition}
See a proof in Appendix~\ref{a:CMI_characterization}. Beyond the above characterization, we now show that CMI further provides a strong statistical guarantee: it controls the conditional dependence for \emph{all} bounded test functions (e.g., auditors or queries). 
To that end, we define the normalized covariance for bounded measurable maps $h:\mathcal U\to\mathbb R^d$ and $g:\mathcal Z\to\mathbb R^d$:
\begin{equation}
\label{eq:norm-cov}
\rho(h,g)
:=\frac{\big|\langle h(U)-\mathbb{E}(h(U)),g(Z)-\mathbb{E}(g(Z))\rangle_{L^2}\big|}
     {\|h(U)\|_{L^\infty}\,\|g(Z)\|_{L^\infty}}\,.
\end{equation}

Now, we first show that the worst-case downstream ($L^{\infty}$-)normalized covariance is upper bounded by the mutual information.

\begin{lemma}[Mutual information controls dependence]\label{l:MI_bound}
Assume \(I(U;Z)<\infty\). For any bounded measurable $h, g$, we have
\begin{equation}
\rho(h,g) \;\le\;\sqrt{2\,I(U;Z)}.
\end{equation}
\end{lemma}
See a proof in Appendix~\ref{a:proof_MI_bound}. Extending the above upper bound to the conditional setting yields our main guarantee:

\begin{theorem}[CMI controls average conditional dependence]\label{th:CMI_bound}
Let \(h:\mathcal U\to\mathbb R^d\) and \(g:\mathcal Z\to\mathbb R^d\) be bounded measurable functions satisfying
\[
    \mathbb E[h(U)\mid Y]=0,
    \qquad
    \mathbb E[g(Z)\mid Y]=0,
    \quad\text{almost surely}.
\]
Then:
\begin{equation}
\sup_{h,g}\;
\frac{\mathbb E\!\left[\left|\mathbb E\!\left[\langle h(U),g(Z)\rangle\mid Y\right]\right|\right]}
     {\|h\|_{L^\infty}\,\|g\|_{L^\infty}}
\;\le\;\sqrt{2\,I(U;Z\mid Y)}.
\end{equation}
\end{theorem}
See a proof in Appendix \ref{a:proof_CMI_bound}. The above result shows that minimizing CMI uniformly controls the expected conditionally centered covariance between bounded downstream tests \(h(U)\) of the prediction and \(g(Z)\) of the sensitive attribute. In particular, small CMI limits the expected conditional covariance that any fixed bounded auditor can detect between the prediction and the sensitive attribute. Standard concentration inequalities further convert this population control into high-probability control of the corresponding sample conditional covariance.

\subsection{When Is a Trade-off Necessary?}
\label{sub:when-tradeoff}
Now, we examine the conditions under which increasing utility \emph{necessarily} forces an increase in separation violation. Let $U$ denote the generic learning outcome. Recall our primary coordinates: utility $u=I(U;Y)$ and separation violation $v=I(U;Z\mid Y)$.

We first establish that the sum of utility and violation is constrained by a universal information budget:
\begin{lemma}[Budget identity and universal bounds]\label{lem:budget}
Given data $(X,Y,Z)$ and any generic prediction outcome $U$, we have
\[
    u+v
    =
    I(U;(Y,Z))
    \le
    \underbrace{I((X,Z);(Y,Z))}_{\text{total budget}}
    = 
    \underbrace{I((X,Z);Y)}_{\text{utility budget}}
    +
    \underbrace{I((X,Z);Z\mid Y)}_{\text{separation budget}},
\]
\[
    0\le u\le I((X,Z);Y),
    \qquad
    0\le v\le I((X,Z);Z\mid Y).
\]
Furthermore, if \(Z\) is discrete, then $I((X,Z);Z\mid Y)=H(Z\mid Y)$. If \(Y\) is discrete and finite, then $I((X,Z);Y) \le H(Y)\le \log |\mathcal Y|$.
\end{lemma}

See a proof in Appendix \ref{a:proof_budget}. This budget identity reveals that, in full generality, there is no necessary strict trade-off. There are degenerate cases where perfect utility and perfect fairness coexist. For instance, if \(Y=(X,Z)\) and \(Y\) is discrete, then \(I((X,Z);Y)=H(Y)\) while \(H(Z\mid Y)=0\). That is, the information budget is entirely allocated to utility without incurring any separation violation.

To rule out such degeneracy and analyze the realistic setting where conflict often arises, we impose the following non-degeneracy condition for the remainder of the analysis:
\begin{equation}\label{eq:ci}
X \perp Z \mid Y, \qquad \text{and} \qquad P_Y\otimes P_Z\ll P_{Y,Z}.
\end{equation}
\begin{itemize}[leftmargin=*]
    \item $X \perp Z \mid Y$: There is no $X$-$Z$ dependence beyond $Y$. While this may not hold for raw data, related optimal-transport and information-theoretic constructions suggest that one can, under suitable conditions, transform \((X,Z)\) into a representation \((X',Z)\) that preserves predictive information about \(Y\) while enforcing \(X' \perp Z \mid Y\) \cite{JMLR:v24:22-0005,xu2025machine}. We adopt this assumption here to simplify the analysis by decoupling the information overlap between \(X\) and \(Z\) beyond \(Y\).
    \item $P_Y\otimes P_Z\ll P_{Y,Z}$: This is an overlap condition ruling out support-level shortcuts. It requires that any measurable subset of $\mathcal Y\times\mathcal Z$ with positive product measure also has positive joint measure under $P_{Y,Z}$. In particular, it excludes degenerate situations in which certain $Y$-values occur only with a null set of $Z$-values (or vice versa), which could otherwise allow $Z$ to improve prediction of $Y$ without inducing conditional dependence beyond $Y$.
\end{itemize}

Under assumptions \eqref{eq:ci}, any predictor that achieves perfect separation ($v=0$) cannot “use” $Z$ in any way that alters its conditional law given $Y$. The following lemma formalizes this intuition: a predictor satisfying perfect separation must, up to its joint law with \(Y\), be equivalent to a predictor that ignores \(Z\).

Let us define the optimal utility achievable using only $X$ versus using both $X$ and $Z$. Let $U_X \in \{f(X), (N,f(X,N))\}$ denote the generic prediction that only explicitly uses $X$ or both $X$ and $N$, but not $Z$. Define
\begin{align*}
    u_X^\star :=\sup_{U_X} I(U_X;Y), \qquad
    u_{XZ}^\star :=\sup_{U} I(U;Y).
\end{align*}
Also, for \(z_0\in\mathcal Z\), define the frozen predictor
\[
    U_{z_0}:=
    \begin{cases}
    f(X,z_0), & \text{if } U=f(X,Z),\\
    (N,f(X,z_0,N)), & \text{if } U=(N,f(X,Z,N)).
    \end{cases}
\]

\begin{lemma}[Conditional law matching]
\label{lem:cond-law-matching}
Let \((\mathcal X,\mathcal Y,\mathcal Z,\mathcal U)\) be Polish spaces, and let
\((X,Y,Z)\) satisfy
\[
    X\perp Z\mid Y,
    \qquad
    P_Y\otimes P_Z\ll P_{Y,Z}.
\]
Also, let $U \in \{f(X,Z), (N,f(X,Z,N))\}$ be a generic predictor satisfying perfect separation: $I(U;Z\mid Y)=0$. Then there exists a measurable set \(\mathcal Z_0\subseteq\mathcal Z\) with
\(P_Z(\mathcal Z_0)=1\) such that, for every \(z_0\in\mathcal Z_0\),
\[
    \mathcal L(U_{z_0}\mid Y=y)
    =
    \mathcal L(U\mid Y=y)
\]
for \(P_Y\)-almost every \(y\). Consequently, $\mathcal L(U_{z_0},Y)=\mathcal L(U,Y)$ for every \(z_0\in\mathcal Z_0\), and hence
\[
    I(U_{z_0};Y)=I(U;Y).
\]
\end{lemma}

See proof in Appendix \ref{append:cond-law-matching}. This leads to our main
necessity result: in the ideally decoupled and overlap setting, perfect
separation restricts the achievable utility to the best \(X\)-only utility.
Therefore, any attempt to obtain utility beyond \(u_X^\star\) necessarily incurs positive separation violation.

\begin{theorem}[Necessary trade-off beyond \(u_X^\star\)]
\label{thm:necessary-tradeoff}
Let \((\mathcal X,\mathcal Y,\mathcal Z,\mathcal U)\) be Polish spaces, and let
\((X,Y,Z)\) satisfy
\[
    X\perp Z\mid Y,
    \qquad
    P_Y\otimes P_Z\ll P_{Y,Z}.
\]
Then the maximum utility compatible with perfect separation is exactly the
utility of the best \(X\)-only predictor:
\[
    \sup_U
    \Big\{
        I(U;Y): I(U;Z\mid Y)=0
    \Big\}
    =
    u_X^\star,
\]
where the supremum is over generic predictors. Therefore, any generic
predictor \(U\) satisfying
\[
    I(U;Y)>u_X^\star
\]
must have strictly positive separation violation:
\[
    I(U;Z\mid Y)>0.
\]
Furthermore, if \(u_{XZ}^\star>u_X^\star\), then there exists a generic predictor
\(U\) such that
\[
    I(U;Y)>u_X^\star
    \quad\text{and}\quad
    I(U;Z\mid Y)>0.
\]
\end{theorem}

See proof in Appendix \ref{append:necessary-tradeoff} and empirical evidence on the Adult dataset in Section \ref{subsubsec:consistency}. Without \eqref{eq:ci}, the conclusion can fail. For instance, if \(Y=Z\), then \(U=Y\) gives \(v=0\), while \(u\) may exceed \(u_X^\star\). Assumption \eqref{eq:ci} excludes such support-level shortcuts. Under this condition, any \(v=0\) predictor is, with respect to its joint law with \(Y\), equivalent to an \(X\)-only predictor obtained by freezing \(Z\). Finally, the condition \(u_{XZ}^\star>u_X^\star\) can be understood as the operational condition that access to \(Z\) provides exploitable predictive information beyond \(X\) for the chosen predictor class and output space.

The results above provide useful guidance for interpreting the separation-utility trade-off in practice. Since the ideal condition \(X\perp Z\mid Y\) need not hold exactly, an \(X\)-only predictor may have non-zero violation. Thus, given a selected empirical \(X\)-only reference predictor \(U_X^{\mathrm{ref}}\), with coordinates
$(v_X^{\mathrm{ref}},u_X^{\mathrm{ref}}):=\bigl(I(U_X^{\mathrm{ref}};Z\mid Y), I(U_X^{\mathrm{ref}};Y)\bigr)$, one can compare it with an unconstrained reference predictor \(U_{XZ}^{\mathrm{ref}}\), with coordinates $(v_{XZ}^{\mathrm{ref}},u_{XZ}^{\mathrm{ref}}):=\bigl(I(U_{XZ}^{\mathrm{ref}};Z\mid Y), I(U_{XZ}^{\mathrm{ref}};Y)\bigr)$. The secant line between these two empirical reference points provides a simple linear estimate of the local price of separation, with slope $\frac{u_{XZ}^{\mathrm{ref}}-u_X^{\mathrm{ref}}}{v_{XZ}^{\mathrm{ref}}-v_X^{\mathrm{ref}}}$. By the concavity of the randomized Pareto frontier in Theorem~\ref{thm:timesharing-frontier-polish}, extending this secant line to the stricter region \(v<v_X^{\mathrm{ref}}\) gives a conservative upper-bound extrapolation for the frontier there. That provides a quick diagnostic for the \textbf{least} marginal utility cost of further reducing separation violation.

To conclude the current section, we have characterized the separation-utility Pareto frontier on the \((v,u)\) information plane, established its concavification under revealed randomization and hence non-decreasing marginal cost (in the generalized concave sense) of separation, justified CMI as a separation-violation quantification through uniform control of bounded downstream tests, and identified conditions under which utility improvements beyond \(u_X^\star\) necessarily incur separation violation. These results ground the use of CMI as a scalar regularizer for navigating the separation-utility trade-off in practice.

\section{Training with Direct CMI Regularization}
\label{sec:algorithm}

\begin{algorithm}[t]
\caption{CMI-regularized training (surrogate for EO) with grad-norm balancing}
\label{alg:cmi_reg_exact}
\begin{algorithmic}[1]
\REQUIRE Dataset $\mathcal{D}=\{(x_i,y_i,z_i)\}$, model $f_\theta = g_\theta\!\circ\!\phi_\theta$, mixing weight $\lambda\!\in[0,1]$, optimizer (Adam), $\epsilon>0$
\ENSURE Trained parameters $\theta$
\FOR{epoch $=1,\dots,T$}
  \FOR{mini-batch $\mathcal{B}=\{(x,y,z)\}$}
    \STATE \textbf{Compose input:} $x_{\mathrm{in}} \gets (x,z)$ \ \ (if using sensitive input; else $x_{\mathrm{in}}\!=\!x$)
    \STATE \textbf{Forward:} $(l,h) \gets (g_\theta(\phi_\theta(x_{\mathrm{in}})),\ \phi_\theta(x_{\mathrm{in}}))$
    \STATE \textbf{Raw losses:} 
    \STATE \hspace{2.85em} $\mathcal L_{\mathrm{task}} \gets \mathrm{CE}(l,y)$
    \STATE \hspace{2.85em} $\widehat I_{\mathrm{CMI}} \gets \textsc{SoftCMI}(l,z,y)$ \quad (computed from $\mathrm{softmax}(l)$)
    \STATE \textbf{Grad-norms on features (detached):}
    \STATE \hspace{2.85em}
      $n_{\mathrm{task}} \gets \mathbb{E}_i\big\|\partial_{h_i}\mathcal L_{\mathrm{task}}\big\|_2\ \text{(stop-grad)}$
    \STATE \hspace{2.85em}
      $n_{\mathrm{cmi}} \gets \mathbb{E}_i\big\|\partial_{h_i}\widehat I_{\mathrm{CMI}}\big\|_2\ \text{(stop-grad)}$
    \STATE \textbf{Balanced objective:}
      $\mathcal L_{\mathrm{final}} \gets (1-\lambda)\dfrac{\mathcal L_{\mathrm{task}}}{n_{\mathrm{task}}+\epsilon}
      + \lambda \dfrac{\widehat I_{\mathrm{CMI}}}{n_{\mathrm{cmi}}+\epsilon}$
    \STATE \textbf{Update:} take an optimizer step on $\nabla_\theta \mathcal L_{\mathrm{final}}$ (Adam)
  \ENDFOR
\ENDFOR
\STATE \textbf{return} $\theta$
\end{algorithmic}
\end{algorithm}

Based on the theoretical insights from Section~\ref{sec:theory-regularization}, we now leverage a sample estimator of conditional mutual information (CMI) to steer the learned predictor toward the separation-utility Pareto frontier.

Given a dataset $\mathcal{D}=\{(x_i,y_i,z_i)\}_{i=1}^N$, we train a model $f_\theta$ consisting of a feature extractor $h = \phi(x, z)$ and a classifier $g(h)$. To ensure optimization stability, particularly when balancing competing objectives with different gradient scales, we employ gradient normalization. We minimize a dynamic objective where each term is scaled by the inverse norm of its gradient with respect to the features $h$:
\begin{equation}
\label{eq:erm-cmi-objective}
\mathcal{L}_{\mathrm{total}}
\;=\;
(1-\lambda)\,\frac{\mathcal{L}_{\mathrm{task}}}{\|\nabla_h \mathcal{L}_{\mathrm{task}}\|}
\;+\;
\lambda\;\frac{\widehat{I}_{\mathrm{CMI}}}{\|\nabla_h \widehat{I}_{\mathrm{CMI}}\|},
\end{equation}
where $\lambda \in [0, 1]$ controls the trade-off. This normalization ensures that both the utility signal and the fairness penalty contribute equally to the update magnitude, preventing one from dominating the other due to arbitrary scaling.

\subsection{Differentiable Soft-Plug-in Estimator}
\label{sub:estimator}

Before introducing the estimator, we emphasize that our direct plug-in estimation approach is intended for the discrete-\(Y\), discrete-\(Z\) setting. This covers the standard separation/equalized-odds fairness setting, where \(Y\) is a discrete target variable and \(Z\) is a finite sensitive attribute such as gender, race, or group membership. In contrast, when \(Y\) or \(Z\) is continuous, especially high-dimensional, direct density-based plug-in estimation becomes statistically challenging: naive nonparametric density estimation typically suffers from rates that deteriorate exponentially with dimension \citep{kraskov2004estimating,paninski2003estimation}. In such settings, adversarial, variational, or otherwise parameterized estimators of mutual information and conditional mutual information are often used, and developing stable estimators remains an active area of research \cite{belghazi2018mine, oord2018cpc, poole2019variational, song2020understanding}. Our focus here is complementary: we show that in the common discrete fairness setting, direct sample-statistics estimation is simple, differentiable, statistically interpretable, and admits finite-sample guarantees, while avoiding the optimization instability and auxiliary-model error that can arise in variational or adversarial approaches.

To backpropagate through \(I(U;Z\mid Y)\) in this discrete setting, we implement a differentiable soft plug-in estimator. Let $p_i = \text{softmax}(f_\theta(x_i)) \in \Delta^{|\mathcal{Y}|}$ be the predicted probability vector for sample $i$.

The estimator $\widehat{I}_{\mathrm{CMI}}$ is computed on a mini-batch $B$ as follows:
\begin{enumerate}[leftmargin=*]
    \item \textit{Stratification:} Partition the batch indices by target class: $\mathcal{I}_y = \{i \in B : y_i = y\}$.
    \item \textit{Soft Joint Distribution:} For each class $y$ and sensitive group $z$, compute the average soft prediction:
    \[ \hat{P}(U|y, z) = \frac{1}{|\{i \in \mathcal{I}_y : z_i=z\}|} \sum_{i \in \mathcal{I}_y, z_i=z} p_i \]
    \item \textit{Divergence Aggregation:} The conditional MI is computed as the weighted sum of KL divergences between the sensitive-conditional distributions and the marginal distribution within each class $y$:
    \[ \widehat{I}_{\mathrm{CMI}} = \sum_{y} \frac{|\mathcal{I}_y|}{|B|} \sum_{z} \hat{P}(z|y) \, D_{\mathrm{KL}}\big( \hat{P}(U|y, z) \,\|\, \hat{P}(U|y) \big) \]
\end{enumerate}

\subsection{Statistical Consistency and Explicit Bias Analysis}
\label{sub:consistency}

A critical feature of the discrete plug-in CMI estimator is that its finite-sample behavior can be monitored during training. Unlike adversarial discriminators or variational lower-bound estimators, which may systematically underestimate dependence when the auxiliary model is misspecified or under-optimized, the ordinary discrete plug-in estimator has a positive leading-order asymptotic bias. Thus, in the large-sample regime, it tends to be conservative in expectation, while also admitting explicit concentration guarantees.

\begin{proposition}[Bias and concentration of hard-count plug-in CMI]
\label{prop:estimation_error}
Let \((U,Y,Z)\) take values in finite sets with cardinalities
\(K_U,K_Y,K_Z\). Let \(\widehat I_B\) be the ordinary plug-in estimator of
\(I(U;Z\mid Y)\) computed from \(|B|\) i.i.d. samples of \((U,Y,Z)\). Assume that $P(Y=y)\ge p_{\min}>0$ and $P(u,z\mid y)\ge q_{\min}>0$ for all \(u,z,y\).

\begin{enumerate}[leftmargin=*]
    \item \textbf{Asymptotic bias.}
    As \(|B|\to\infty\), the estimator has positive leading-order bias:
    \[
        \mathbb E[\widehat I_B]-I(U;Z\mid Y)
        =
        \frac{K_Y(K_U-1)(K_Z-1)}{2|B|}
        +
        O(|B|^{-2}).
    \]
    \item \textbf{Concentration.}
    For any \(\delta\in(0,1)\), if \(|B|\) is sufficiently large, then with
    probability at least \(1-\delta\),
    \[
        \left|
        \widehat I_B-\mathbb E[\widehat I_B]
        \right|
        \le
        C
        \sqrt{
            \frac{\log(2/\delta)}{|B|}
        },
    \]
    where
    \(C=C(K_U,K_Y,K_Z,p_{\min},q_{\min})\).
\end{enumerate}
\end{proposition}

See a proof in Appendix \ref{Append:algorithm}. We note that the above result is stated for the hard-count plug-in estimator computed from i.i.d. samples of \((U,Y,Z)\). In practice, our differentiable estimator uses the same plug-in construction but replaces hard counts of \(U\) by soft counts \(p_i(u)\), yielding a smooth relaxation suitable for backpropagation. This estimator reduces to the hard-count plug-in estimator in the one-hot limit \(p_i(u)=\mathbf 1\{U_i=u\}\), and under suitable regularity assumptions it is expected to inherit analogous consistency and \(O(|B|^{-1/2})\) concentration behavior. However, the explicit Miller-Madow bias coefficient in Proposition~\ref{prop:estimation_error} should be interpreted here as the classical hard-count reference formula.

\section{Numerical Experiment}
\label{sec:experiments}

This section provides empirical evidence for the separation-utility theory in Section~\ref{sec:theory-regularization} and evaluates the proposed (gradient-normalized) plug-in CMI regularization in Section~\ref{sec:algorithm}. We evaluate on both classic binary fairness benchmarks (Adult, COMPAS, Bank, CelebA) and an ACS occupation-prediction benchmark with multi-class outcomes and multi-group sensitive attributes. The experiments are designed to test five contributions of the paper:

\begin{enumerate}[leftmargin=*]
    \item \textbf{Frontier geometry: theory leads to observable prediction.}
    We empirically approximate the theoretically predicted separation-utility frontier and show that revealed randomization (between predictions) produces a smoother and typically higher envelope than the deterministic frontier, both for randomized-policy predictions and for deterministic-policy predictions. This is the empirical counterpart of the concavification result in Theorem~\ref{thm:timesharing-frontier-polish}. The Adult experiment (Section~\ref{subsubsec:consistency}) provides an empirical counterpart of the necessary-tradeoff mechanism in Theorem~\ref{thm:necessary-tradeoff}: mutual-information utility beyond the empirical \(X\)-only level is achieved only at positive CMI separation violation.

    \item \textbf{Stable travel on the frontier: algorithmic contribution.}
    Tuning the trade-off parameter yields a stable approximation and a largely monotone traversal of the Pareto frontier under CMI regularization. The resulting curve exhibits smaller cross-validation variance, especially in the strict regime near zero separation violation, where baselines often show backtracking or variance blow-up. This behavior is consistent with the statistics-based stability suggested by Proposition~\ref{prop:estimation_error}, in contrast to learning-based proxy objectives that introduce additional estimation and optimization error.

    \item \textbf{Operational transfer: information-plane improvements generalize to deployment metrics.}
    Improvements in the information plane, i.e., lower $I(U;Z\mid Y)$ at comparable $I(U;Y)$, transfer to improved Equalized Odds (EO) gap at comparable Accuracy/AUROC on cross-validation test folds. This provides empirical evidence that CMI is a useful separation quantification and effective regularization, consistent with the uniform downstream covariance control in Theorem~\ref{th:CMI_bound}.

    \item \textbf{Representation-level separation vs.\ post-hoc thresholding.}
    On CelebA and ACSOccupation, posterior-level randomized evaluation separates methods more clearly than deterministic evaluation. CMI reduces sensitive dependence in the predictive posterior itself, while several baselines appear competitive only after deterministic thresholding. This illustrates that post-hoc decision repair can mask posterior-level separation violations that remain in the model.

    \item \textbf{Beyond binary fairness: multi-class and multi-group separation.}
    We further evaluate on ACS occupation prediction, where \(Y\) is a multi-class occupation label and \(Z\) is a multi-group sensitive attribute. This directly tests a central advantage of the information-theoretic formulation: \(I(U;Z\mid Y)\) applies without modification as both a separation quantification and a regularization objective beyond binary targets and binary sensitive attributes.
\end{enumerate}

\subsection{Setup, Metrics, and Baselines}

\textbf{Datasets.}
We evaluate on four standard fairness benchmarks: Adult~\cite{adult1996}, COMPAS~\cite{angwin2016machine}, Bank~\cite{bank_marketing_222}, and CelebA~\cite{liu2015deep}, covering both tabular and image modalities with varying sample sizes and class/group imbalances. We additionally evaluate on an ACS task derived from the American Community Survey~\cite{ACS} to test a genuinely non-binary setting. Throughout, \(X\) denotes the input features, \(Y\) the target variable, and \(Z\) the sensitive attribute. Table~\ref{tab:dataset_summary} summarizes the benchmark information.

\begin{table}[t]
\centering
\caption{
\textbf{Summary of experimental datasets.}
Here $d_X$ denotes the input feature dimension after preprocessing. For tabular datasets, this is the dimension after numerical features are retained and categorical features are one-hot encoded. For CelebA, $d_X=512$ is the dimension of the fixed ResNet18 image embedding. ACSOccupation is the only benchmark in which both the outcome and sensitive attribute are non-binary. In the main experiment, we use a conservative complete-case preprocessing rule that removes rows with missing or structurally unavailable covariate values, yielding a smaller but clean multi-class/multi-group benchmark. In Appendix~\ref{append:acs_large}, we additionally report a large-scale ACS preprocessing variant that retains missing or structurally unavailable covariate values as explicit categorical levels.
}
\label{tab:dataset_summary}
\begin{tabular}{lccccc}
\toprule
Dataset & Sample size & $d_X$ & $|\mathcal Y|$ & $|\mathcal Z|$ & Task type \\
\midrule
Adult & $33$K & $16$ & $2$ & $2$ & binary tabular \\
COMPAS & $ 5$K & $8$ & $2$ & $2$ & binary tabular \\
Bank & $ 45$K & $51$ & $2$ & $2$ & binary tabular \\
CelebA & $203$K & $512$ & $2$ & $2$ & binary image \\
ACSOccupation & $2.5$K & $55$ & $20$ & $9$ & multi-class / multi-group tabular \\
ACSOccupation-Large & $576$K & $55$ & $20$ & $9$ & multi-class / multi-group tabular \\
\bottomrule
\end{tabular}
\end{table}

\textbf{Metrics.}
We report utility using mutual information $I(U;Y)$ and Accuracy/AUROC as practical deployment metrics. We report separation violation using conditional mutual information $I(U;Z\mid Y)$ and, for binary tasks, Equalized Odds gap as the corresponding operational fairness metric. For binary tasks, AUROC denotes the standard area under the receiver operating characteristic curve, computed from the positive-class score. For the multi-class ACSOccupation task, we report macro-AUROC, computed by averaging one-vs-rest AUROC over all occupation classes:
\[
    \mathrm{AUROC}_{\mathrm{macro}}
    =
    \frac{1}{|\mathcal Y|}
    \sum_{y\in\mathcal Y}
    \mathrm{AUROC}\!\left(
        \mathbf 1\{Y=y\},
        Q_\theta(y\mid X,Z)
    \right).
\]
Thus, binary AUROC measures ranking quality for a single positive class, whereas macro-AUROC measures average one-vs-rest ranking quality across classes. The CMI-MI plots evaluate the information-theoretic frontier, whereas the Accuracy/AUROC-EO plots evaluate whether information-plane improvements transfer to deployment metrics. In the multi-class ACS experiment, we also report expected randomized accuracy and top-1 accuracy of the induced hard classifier. See Appendix~\ref{append:metrics} for formal definitions of all metrics.

\textbf{Randomized-policy and deterministic-policy predictions.}
For randomized policies, we use
\[
    \hat Y \sim Q_\theta(\cdot\mid X,Z),
\]
where $Q_\theta(\cdot\mid X,Z)$ is the predicted class-probability vector. Information quantities such as $I(\hat Y;Y)$ and $I(\hat Y;Z\mid Y)$ are computed from the empirical
soft-count joint law
\[
    \widehat P(\hat Y,y,z)
    =
    \frac1n\sum_{i=1}^n
    Q_\theta(\hat Y\mid x_i,z_i)
    \mathbf 1\{y_i=y,z_i=z\}.
\] The induced deterministic classifier is
\[
    \hat Y_{\mathrm{hard}}
    =
    \arg\max_y Q_\theta(y\mid X,Z).
\]
For hard-decision information quantities, we replace
\(Q_\theta(\cdot\mid x_i,z_i)\) by the one-hot vector
\(\mathbf 1\{\hat Y_{\mathrm{hard},i}=\cdot\}\). For binary tasks, deterministic hard decisions are obtained by thresholding the predicted positive-class score; the threshold is selected by the deterministic post-processing rule described below. AUROC is reported for the soft prediction distribution, since AUROC is a ranking metric and less meaningfully evaluated from one-hot hard labels.

Finally, we note that randomized-policy predictions are distinct from the external revealed randomization used in the theoretical randomized frontier. Each trained model $Q_\theta$ defines a single predictor and is evaluated as one point on the empirically approximated deterministic Pareto frontier. By contrast, the randomized frontier in our theoretical characterization is obtained by time-mixing among distinct predictors using an independent revealed random variable.

\textbf{Baselines.}
We compare CMI against representative baselines from four algorithmic families:
(1) constraint-based methods: \emph{EG Reductions} (ExpGrad)~\cite{agarwal2018reductions};
(2) adversarial methods: \emph{Adversarial Debiasing} (Adversarial)~\cite{madras2018learning} and \emph{Fair Dummy} (FairDummy)~\cite{romano2021fair};
(3) information-theoretic or representation-level proxies: \emph{FR-Train}~\cite{roh2020frtrain} and \emph{InfoFair}~\cite{kang2022infofair};
and (4) robustness-based methods: \emph{FairDRO}~\cite{fairdro2024}.
All neural methods use a unified two-hidden-layer MLP backbone. We report both randomized policies, which evaluate posterior-level separation, and deterministic policies, which evaluate the induced decision rule. See Appendix~\ref{appen:methods} for method details and the rationale behind our choice of comparison methods.

\textbf{Fair comparison protocol.}
To isolate the effect of the regularization objective, we standardize the training pipeline. All neural methods use the same MLP architecture, optimizer family, number of epochs, and cross-validation splits. For randomized evaluation, we use the raw posterior probability vector. For deterministic evaluation in binary tasks, we use hard predictions obtained by a training-fold-tuned threshold $t^*$ chosen to maximize $\mathrm{Accuracy}-\lambda\cdot\mathrm{EO\ gap}$, simulating deployment-time decision selection. We report mean $\pm$ one standard deviation over five folds.

\subsection{Results and Analysis}

We analyze the Pareto frontiers generated by tuning the fairness trade-off parameter for each method. We begin with Bank and CelebA, which highlight two different experimental focuses: a rapidly increasing price of separation on Bank and high-dimensional image input features on CelebA. We then discuss Adult and COMPAS as consistency checks across additional tabular domains, before turning to the ACSOccupation experiment with multi-class targets and multi-group sensitive attributes.

\subsubsection{Bank Dataset: Frontier Shape and Operational Transfer}
\label{subsubsec:bank}

\begin{figure*}[t]
    \centering
    \setlength{\tabcolsep}{1pt}
    
    \subfloat{\includegraphics[width=0.49\textwidth]{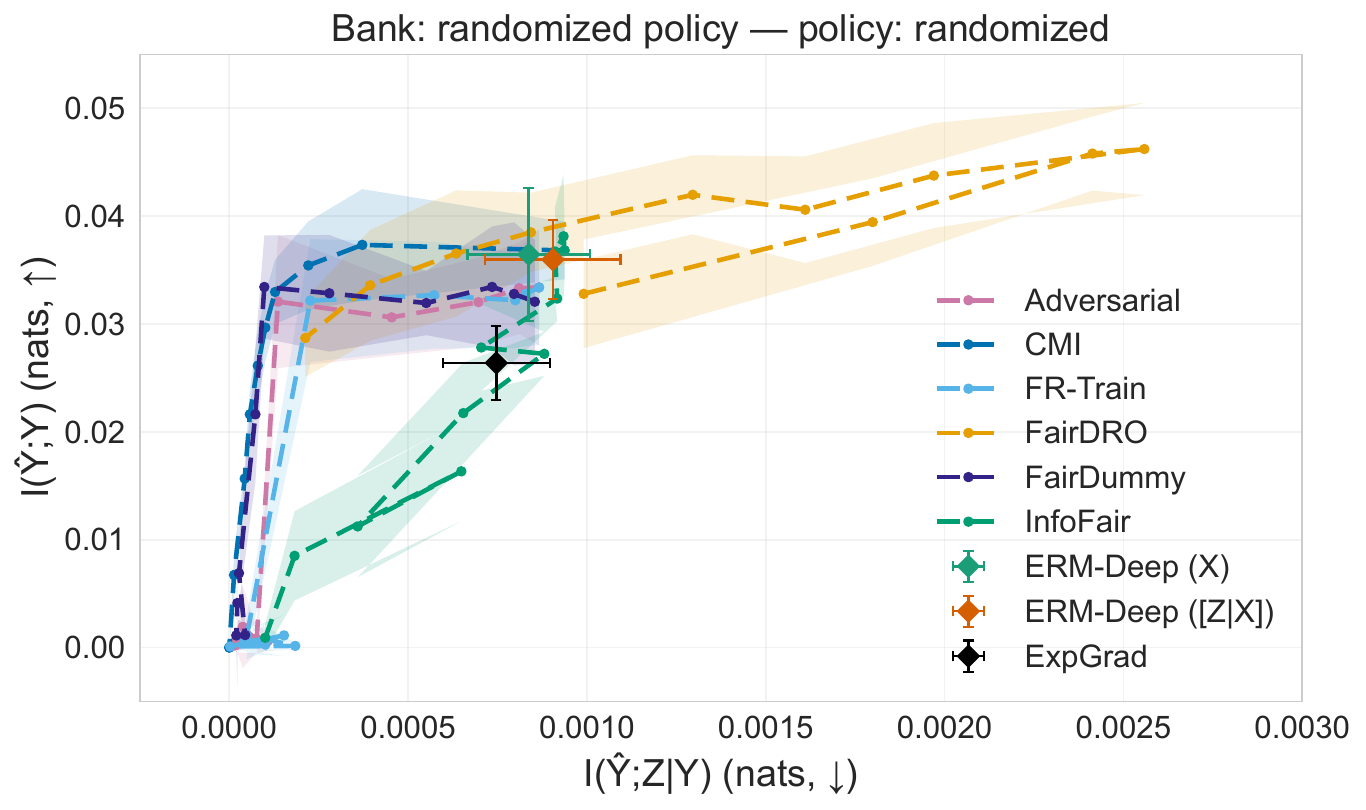}}
    \subfloat{\includegraphics[width=0.49\textwidth]{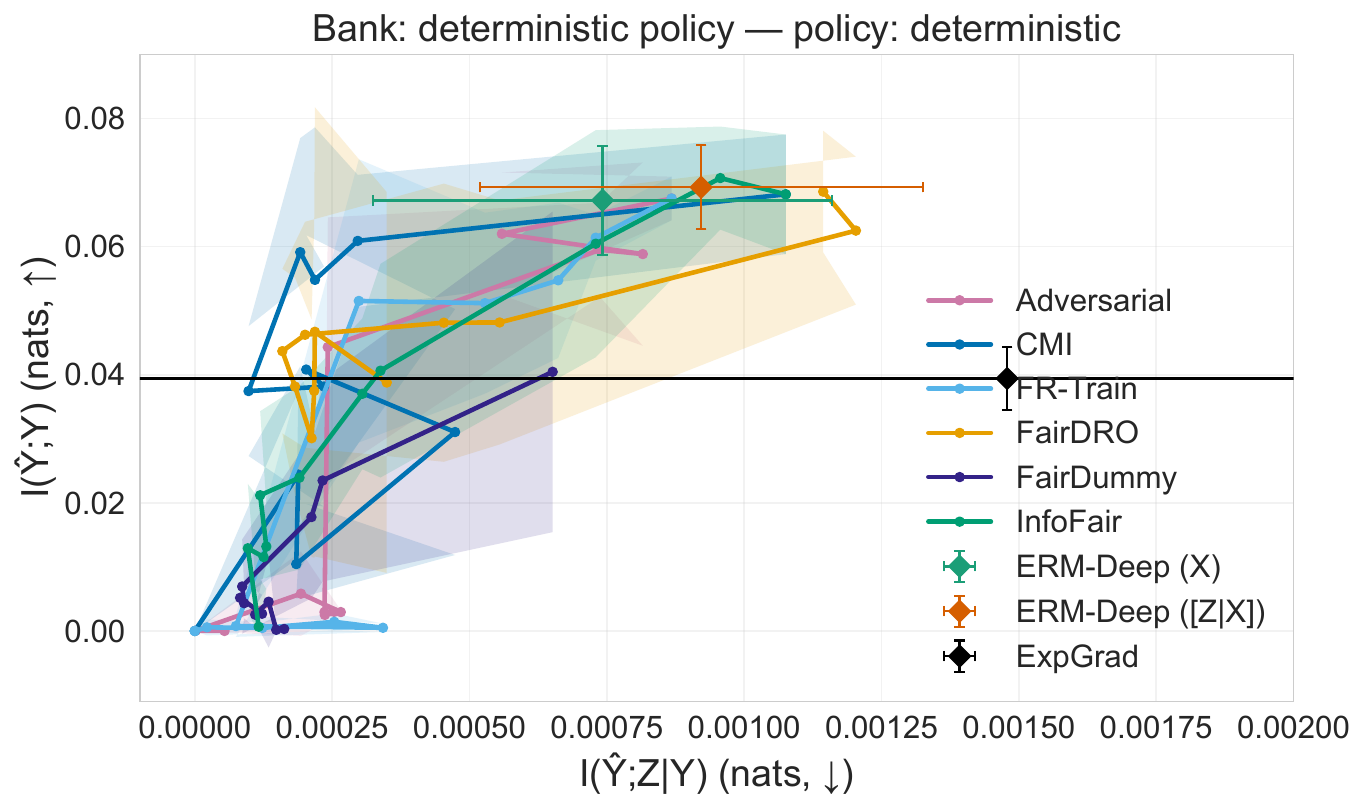}} \\
    \subfloat{\includegraphics[width=0.49\textwidth]{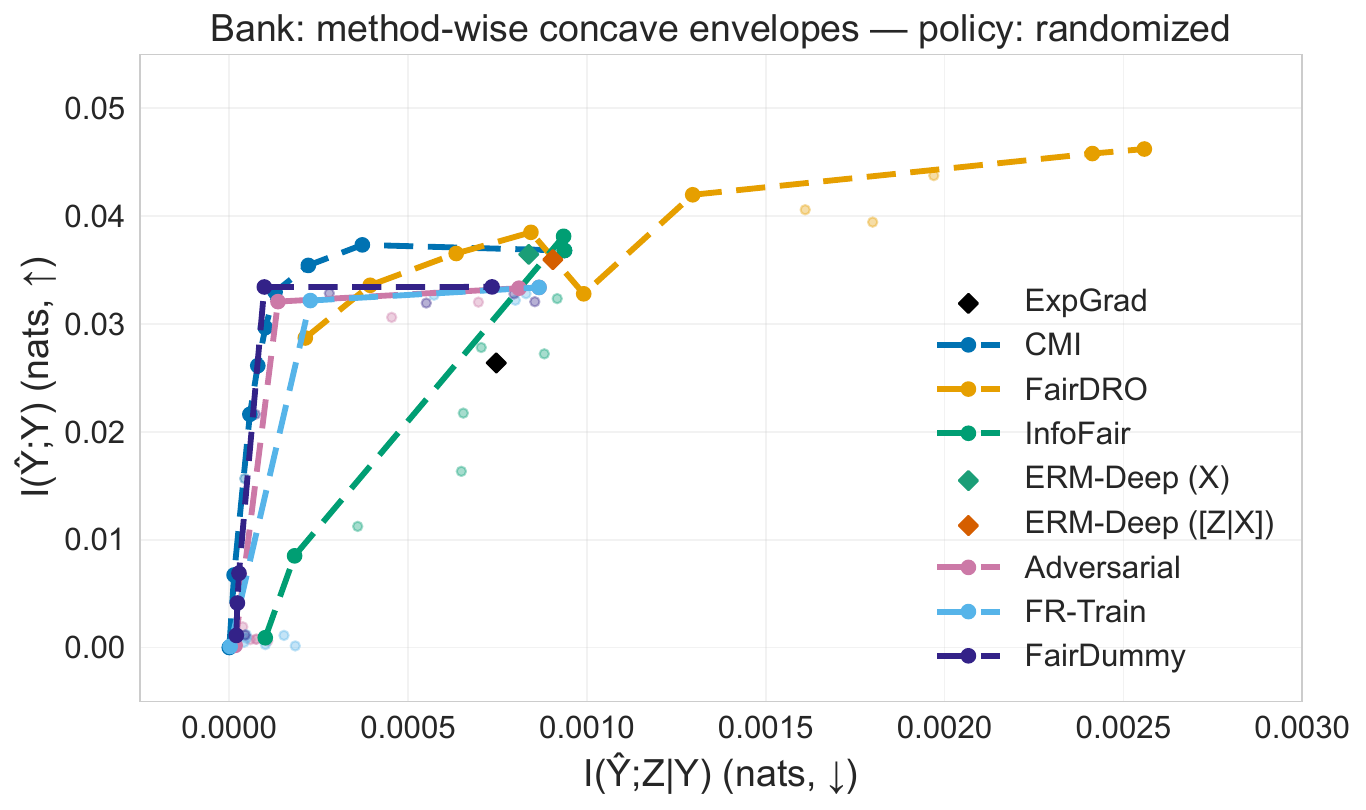}}
    \subfloat{\includegraphics[width=0.49\textwidth]{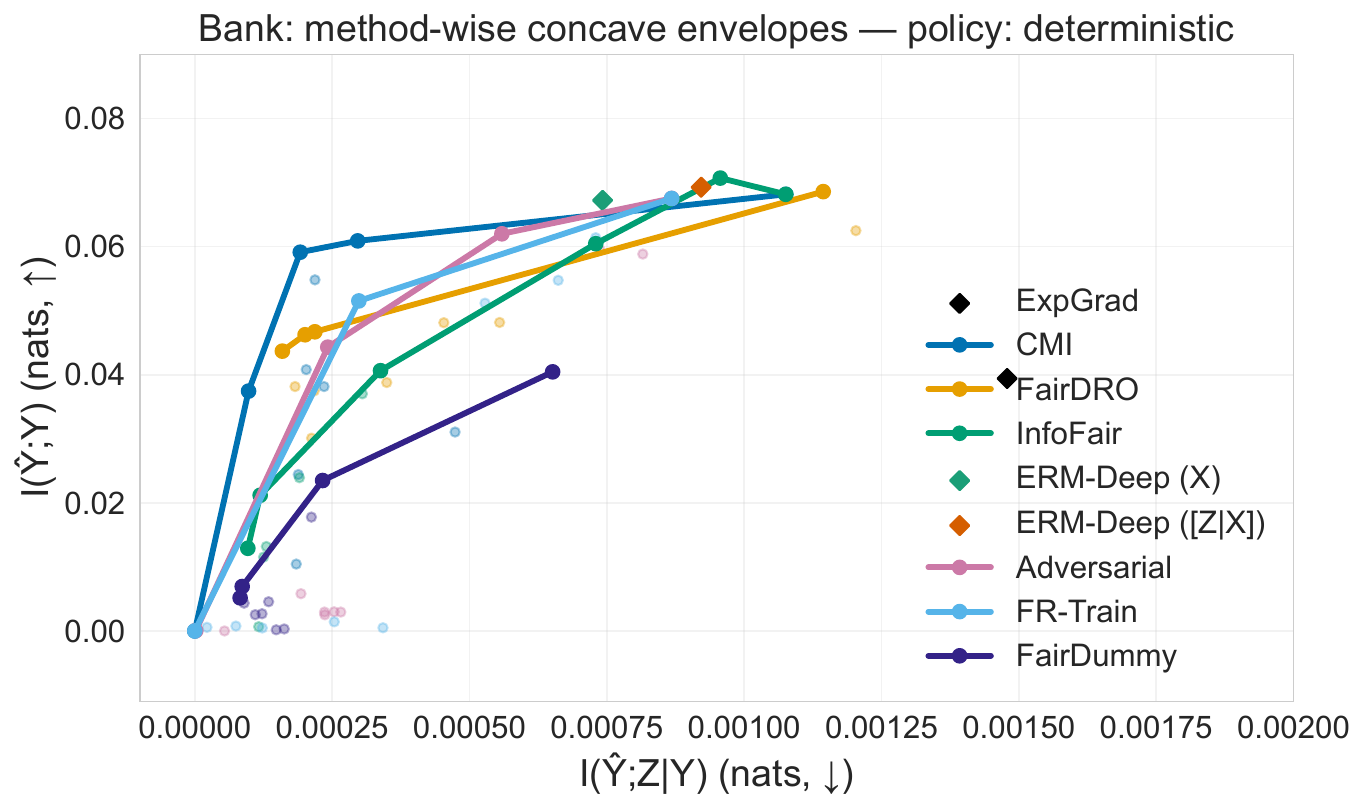}} \\
    \subfloat{\includegraphics[width=0.49\textwidth]{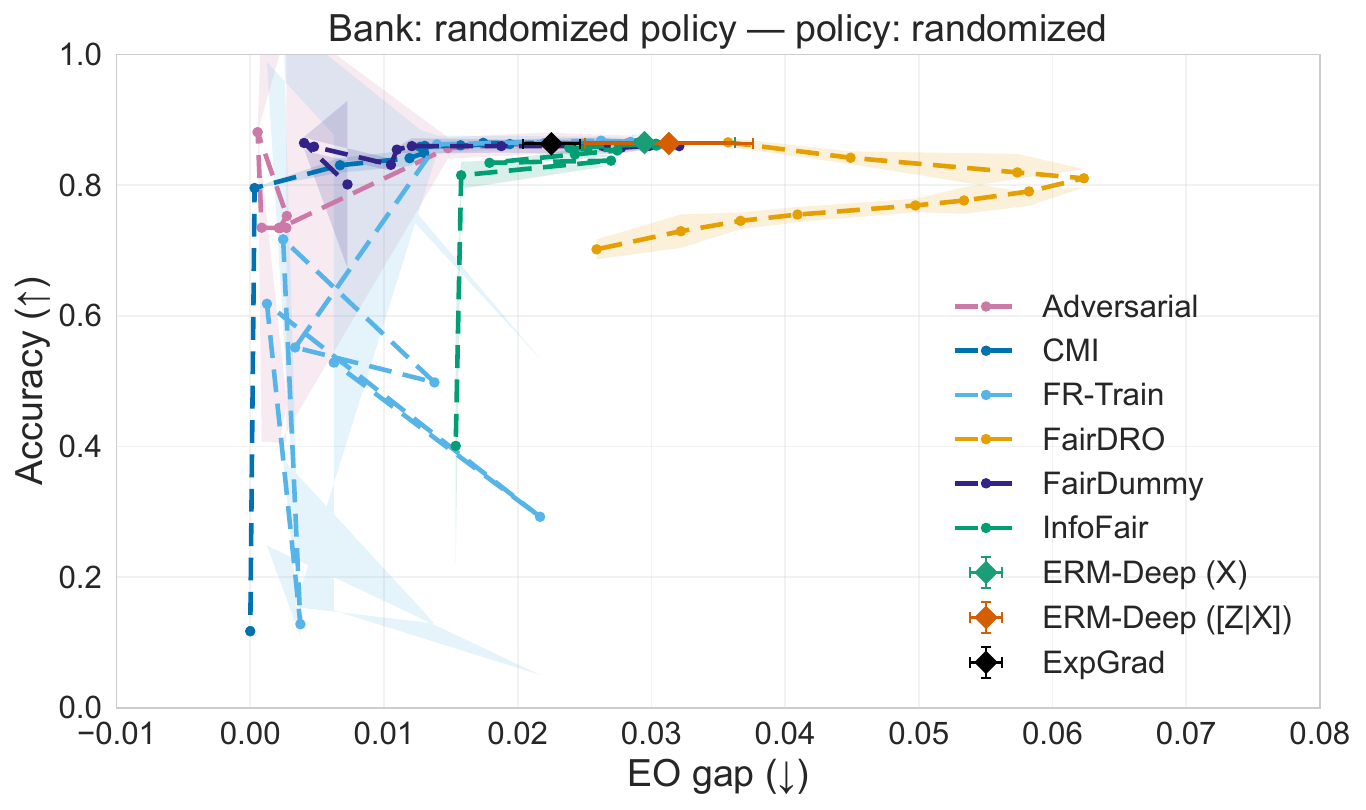}}
    \subfloat{\includegraphics[width=0.49\textwidth]{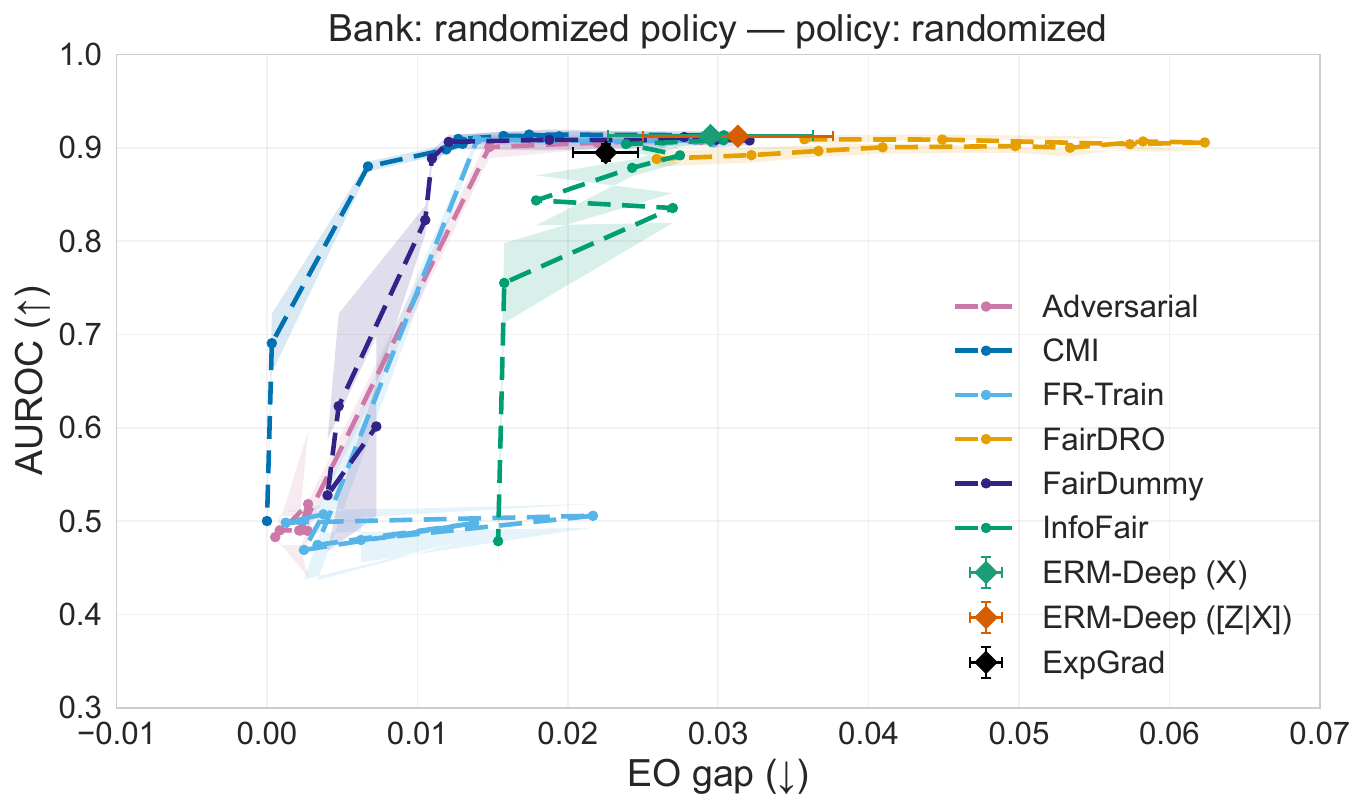}}
    
    \caption{\textbf{Bank Results: Frontier Shape, Method-wise Concave Envelopes, and Operational Transfer.}
    \textbf{Top:} Raw information-plane frontier approximations for posterior-level predictions (left) and deterministic predictions (right).
    \textbf{Middle:} Method-wise upper concave envelopes of the observed information-plane points, summarizing empirical frontiers achievable by revealed randomization among trained predictors of the same method.
    \textbf{Bottom:} Accuracy and AUROC as functions of EO gap, evaluating whether the information-plane frontier approximation translates to deployment metrics.}
    \label{fig:Bank_combined}
    \vspace{-10pt}
\end{figure*}

Figure~\ref{fig:Bank_combined} shows that Bank has a sharply nonuniform separation-utility trade-off. The raw information-plane frontier approximations reveal a steep ``knee'' as $I(\hat Y;Z\mid Y)$ approaches zero, indicating a rapidly increasing marginal cost of separation in the strict-separation regime. When the separation violation is not too small, high utility can be maintained with little additional cost. However, as the violation approaches zero, further reduction becomes much more expensive. Thus, exact or near-exact separation is costly, while permitting a very small violation recovers near-saturated utility.

We organize the Bank results around two points.

\begin{itemize}[leftmargin=*]
    \item \textbf{Frontier shape and empirical envelopes.}
    The proposed plug-in CMI regularizer provides a smooth traversal of the high-utility, low-violation region. In contrast, Adversarial, FR-Train, and FairDummy reduce utility near the unconstrained ERM $([Z|X])$ endpoint. FairDRO exhibits backtracking, while InfoFair shows utility collapse in parts of the path. The middle row further reports the method-wise upper concave envelopes of the observed information-plane points for posterior-level predictions on the left and deterministic predictions on the right. These envelopes summarize the empirical frontiers achievable by revealed randomization among trained predictors of the same method. The relevant fairness-utility frontier is the portion moving from the unconstrained ERM endpoint toward smaller separation violation. CMI gives the clearest empirical path toward near-separation while preserving high utility in both randomized and deterministic policy settings.

    \item \textbf{Operational transfer.}
    The bottom row evaluates whether each method's information-plane frontier approximation translates to deployment metrics. CMI maintains relatively high Accuracy and AUROC as $I(\hat Y;Z\mid Y)$ approaches zero, and this behavior transfers effectively to the low-EO-gap region. In contrast, other baseline methods show unstable or non-monotone behavior when transferring the information-theoretic frontier traversal to the Accuracy-EO-gap plane, and they incur larger AUROC degradation to achieve comparable EO-gap reduction. This further supports the practical relevance of directly controlling $I(\hat Y;Z\mid Y)$ as a separation violation.
\end{itemize}

\subsubsection{CelebA Dataset: Scalability and Posterior Separation}
\label{subsubsec:celeba}

\begin{figure*}[t]
    \centering
    \setlength{\tabcolsep}{1pt}
    
    \subfloat{\includegraphics[width=0.49\textwidth]{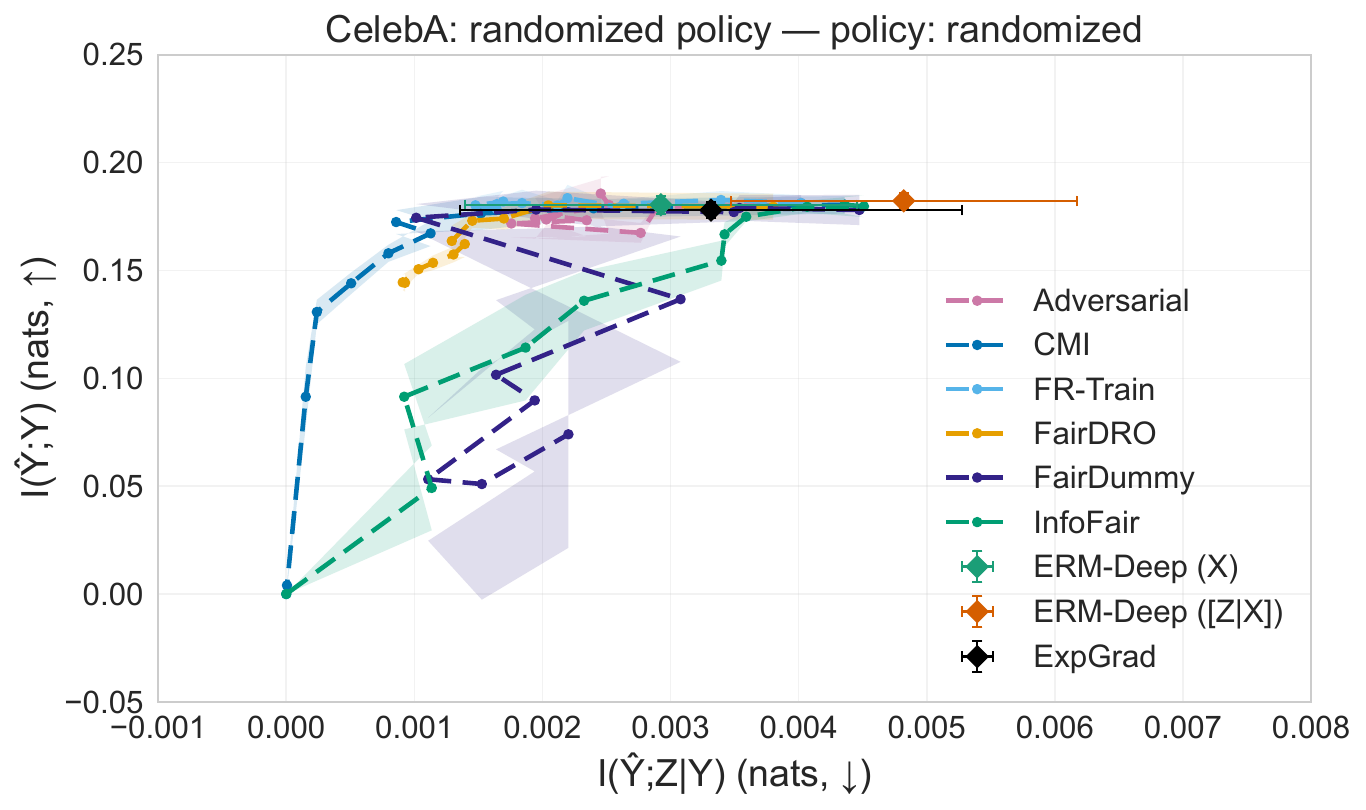}} \hfill
    \subfloat{\includegraphics[width=0.49\textwidth]{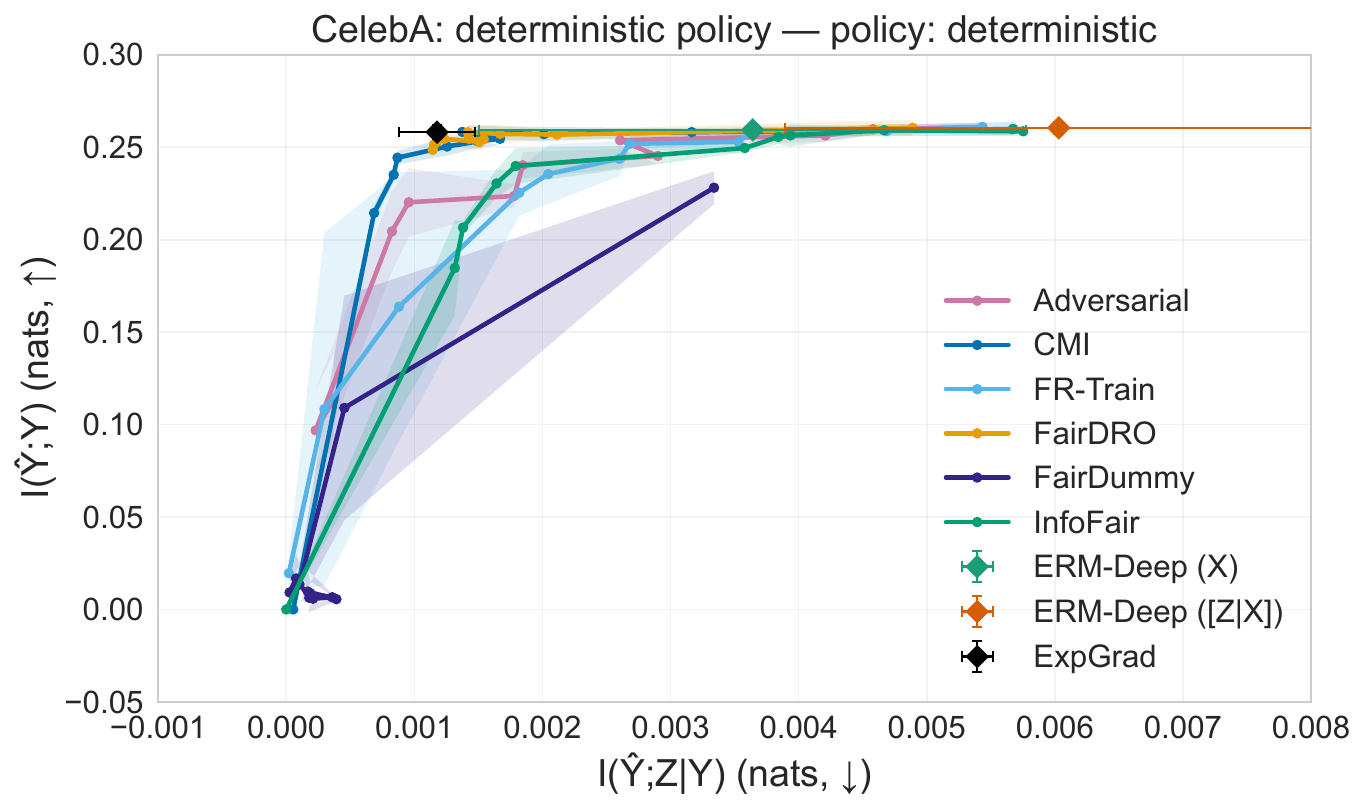}} \\
    \subfloat{\includegraphics[width=0.49\textwidth]{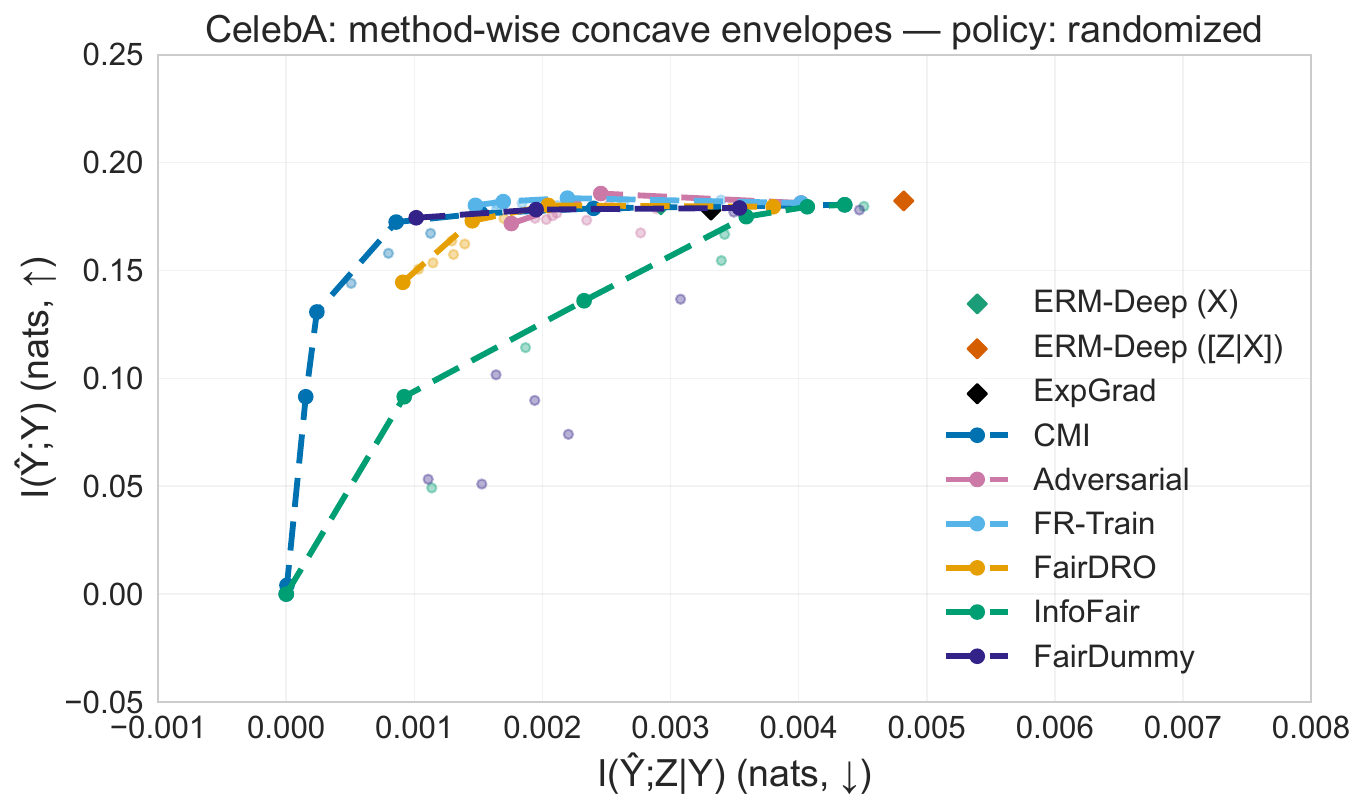}} \hfill
    \subfloat{\includegraphics[width=0.49\textwidth]{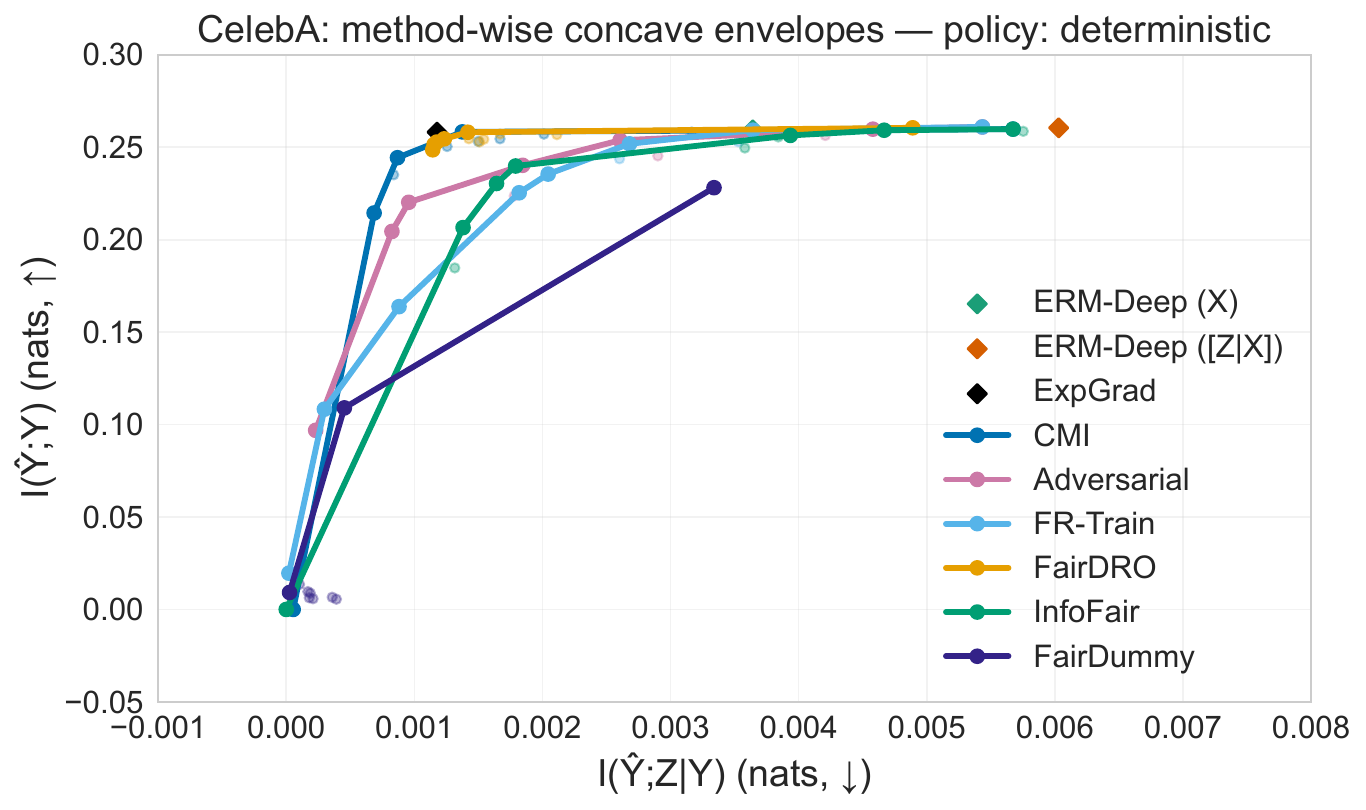}} \\
    \subfloat{\includegraphics[width=0.49\textwidth]{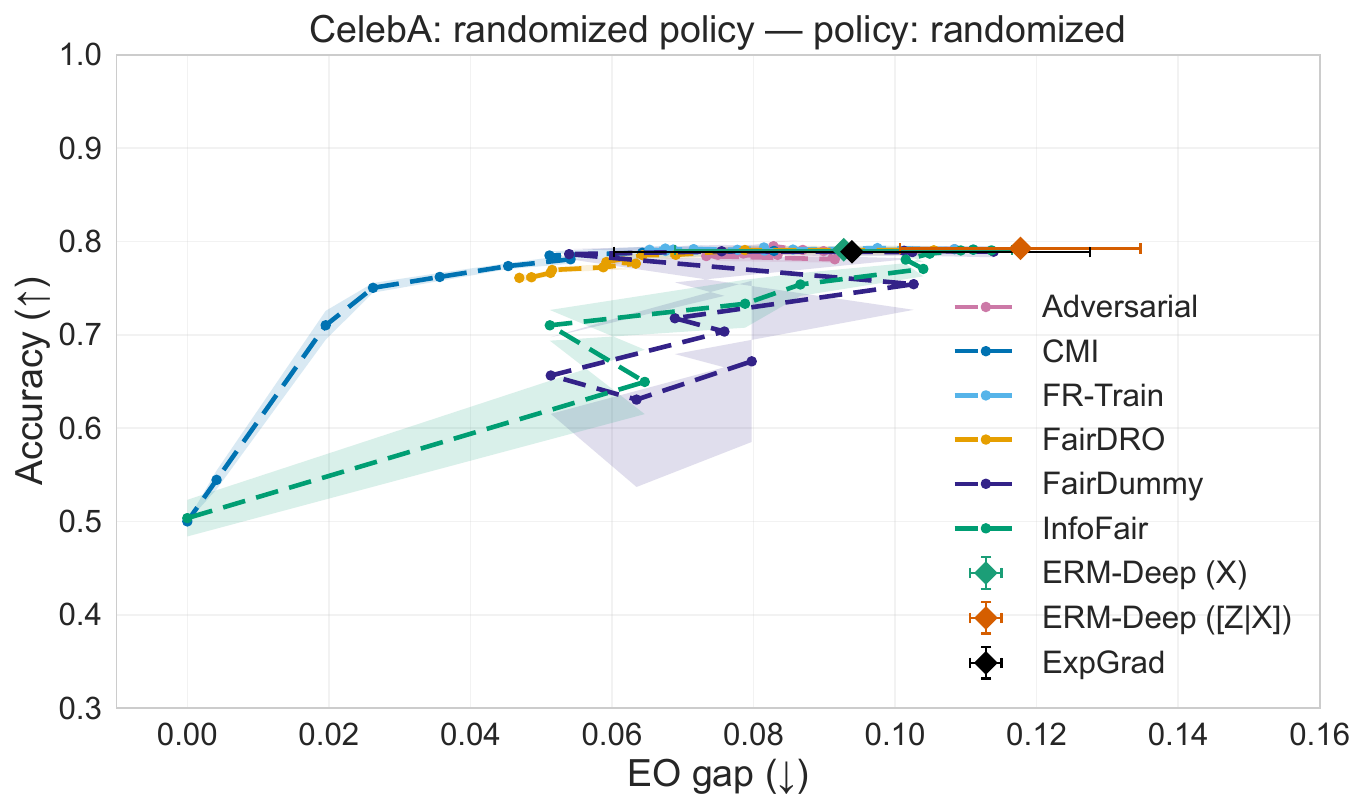}} \hfill
    \subfloat{\includegraphics[width=0.49\textwidth]{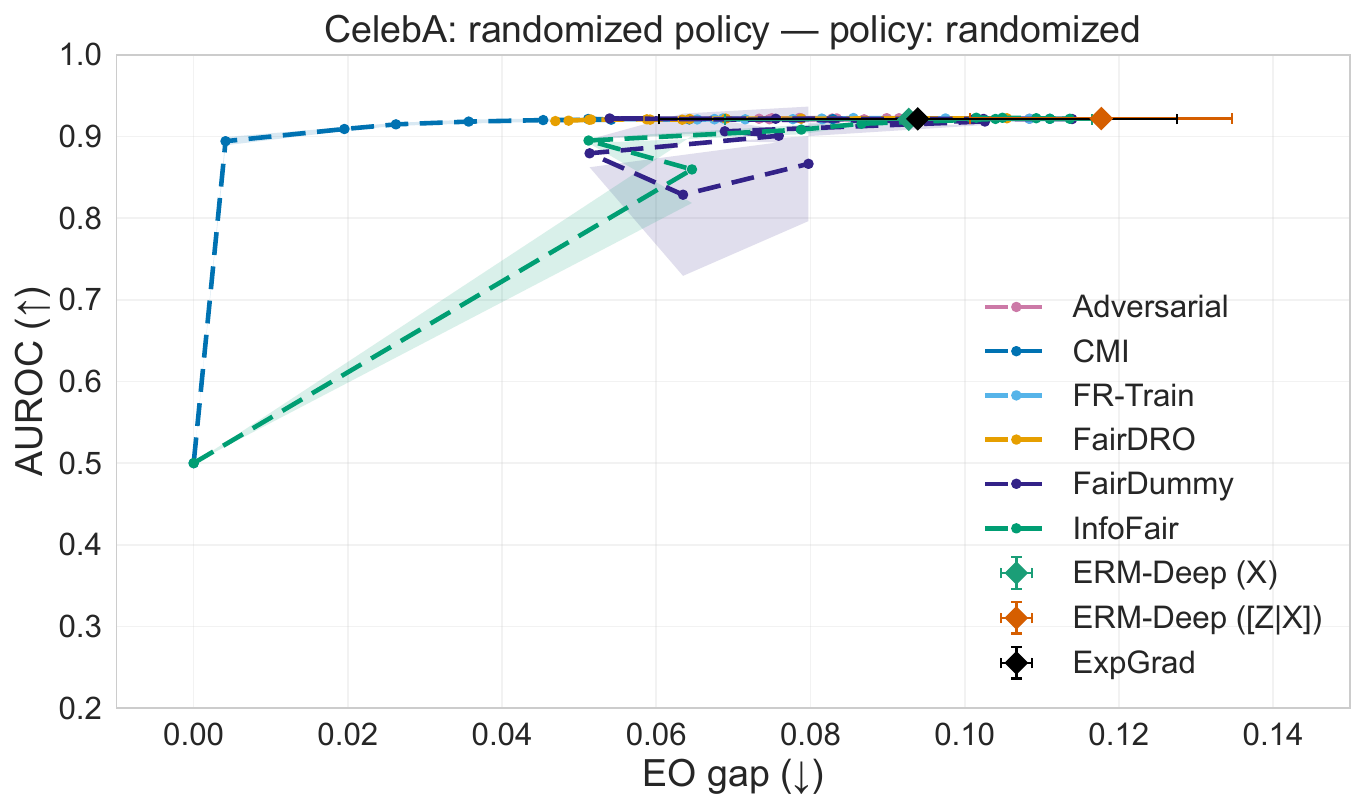}}
    
    \caption{\textbf{CelebA Results: Scalability, Empirical Envelopes, and Posterior Separation.} 
    \textbf{Top:} In the randomized-policy posterior-level view (left), the gradient-normalized CMI is able to maintain relatively high utility while reducing $I(\hat Y;Z\mid Y)$ close to zero, avoiding the range collapse, backtracking, and instability observed in other baselines on high-dimensional image-input setting. The deterministic-policy plot on the right shows that hard thresholding can make several methods appear similarly fair at high utility, even when posterior-level dependence on the sensitive attribute remains in the learned predictive distribution.
    \textbf{Middle:} We plot the method-wise upper concave envelopes of the observed information-plane points. These envelopes represent empirical frontiers achievable by revealed randomization among the trained predictors of each method.
    \textbf{Bottom:} Operational metrics mirror the posterior-level frontier. CMI maintains high Accuracy and AUROC while reducing the EO gap close to zero.}
    \label{fig:CelebA_combined}
    \vspace{-10pt}
\end{figure*}

Figure~\ref{fig:CelebA_combined} uses CelebA to test the framework on high-dimensional image inputs. While the prediction target remains binary, the optimization challenge lies in controlling separation for complex, high-dimensional input data. This benchmark evaluates each method's scalability to complex input structures, distinguishing methods that remain effective on image input data from those that degrade under high input dimensionality.

\begin{itemize}[leftmargin=*]
    \item \textbf{Scalability to high-dimensional input.}
    In the randomized/posterior-level information plane (first and second rows, left), CMI traces a stable Pareto-frontier path and reaches high utility with small $I(\hat Y;Z\mid Y)$. It avoids the range collapse, backtracking, and utility collapse observed in several comparison baselines. This provides empirical evidence that, unlike adversarial or variational objectives that can struggle in high-dimensional input settings such as CelebA, the direct CMI estimator provides a stable and well-conditioned training signal for posterior-level separation-utility frontier approximation.

    \item \textbf{Masking effect.} The discrepancy between the randomized-policy (posterior-level) frontiers (first and second rows, left) and deterministic-policy frontiers (first and second rows, right) is most pronounced on CelebA. Several baselines fail to reduce posterior-level separation violation while retaining meaningful utility, yet their deterministic-policy predictions have significantly reduced separation violation after thresholding. This indicates that hard decisions can mask posterior-level dependence on the sensitive attribute that remains in the learned predictive distribution. CMI, by contrast, directly reduces sensitive dependence in the learned predictive posterior, supporting its effectiveness for intrinsic bias mitigation in complex data modalities.

    \item \textbf{Operational transfer.}
    The operational plots (third row) confirm that the proposed CMI method's randomized-policy information-theoretic frontier traversal transfers effectively to deployment-metric trade-offs. In particular, CMI maintains high Accuracy and AUROC while reducing the EO gap close to zero.

\end{itemize}

\subsubsection{Adult and COMPAS: Necessary Trade-off and Cross-Domain Consistency}
\label{subsubsec:consistency}

\begin{figure*}[!p]
    \centering
    \setlength{\tabcolsep}{1pt}
    
    \subfloat{\includegraphics[width=0.49\textwidth]{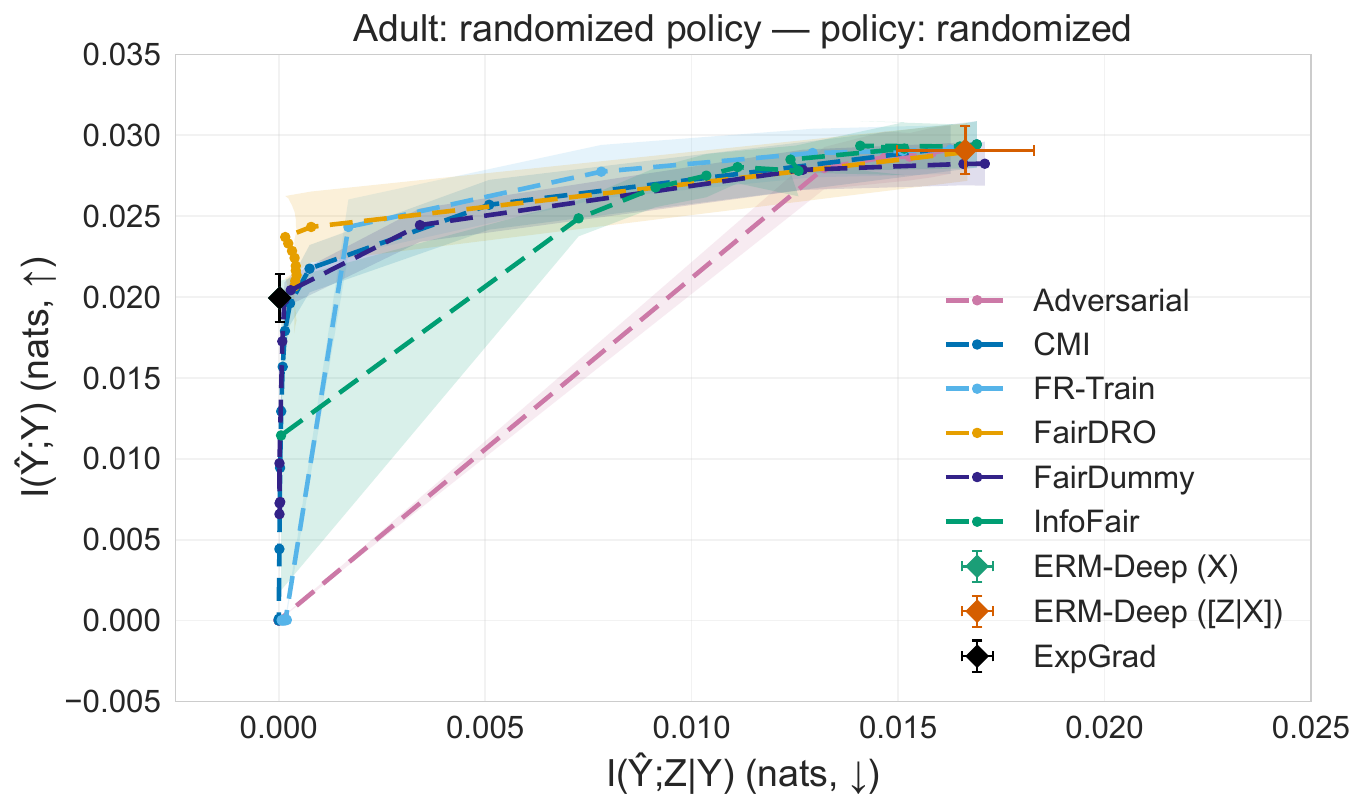}}
    \subfloat{\includegraphics[width=0.49\textwidth]{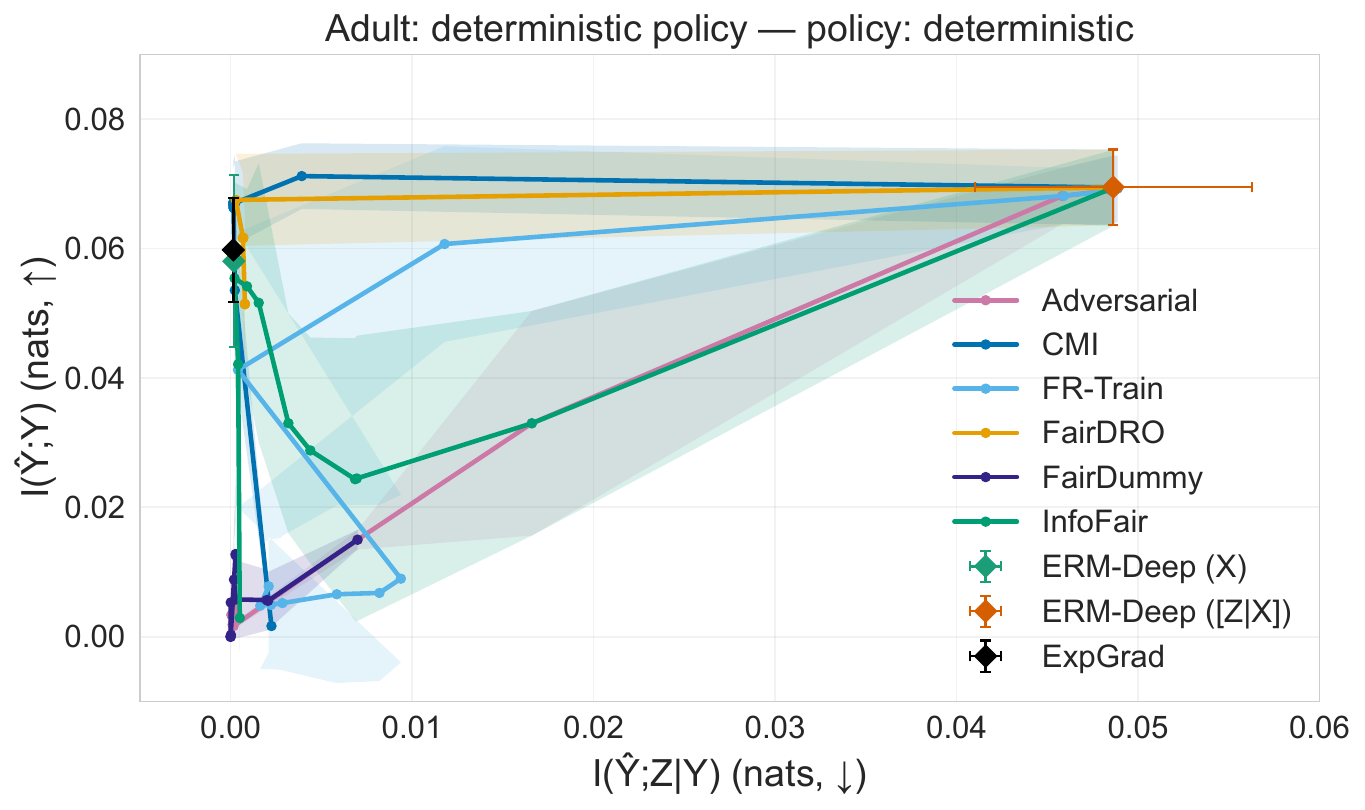}} \\
    \subfloat{\includegraphics[width=0.49\textwidth]{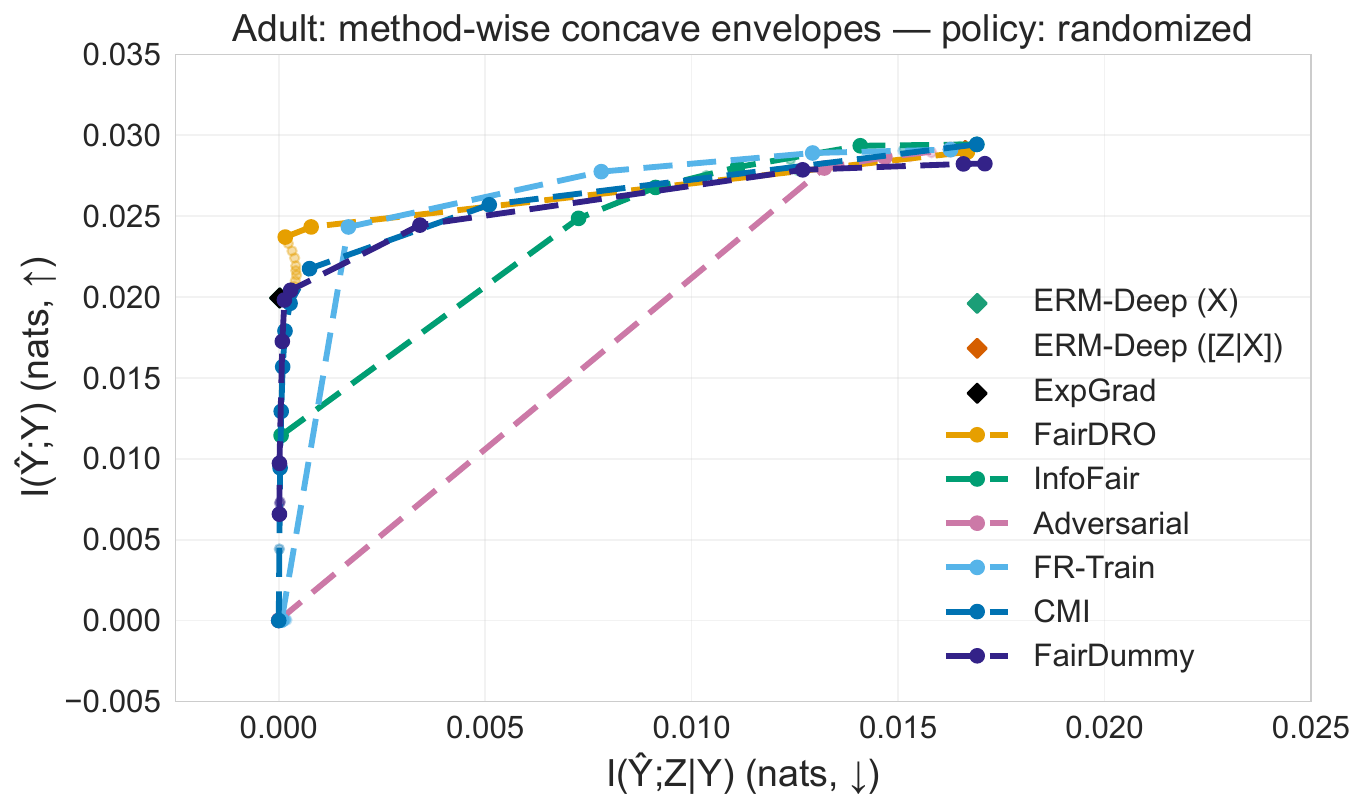}}
    \subfloat{\includegraphics[width=0.49\textwidth]{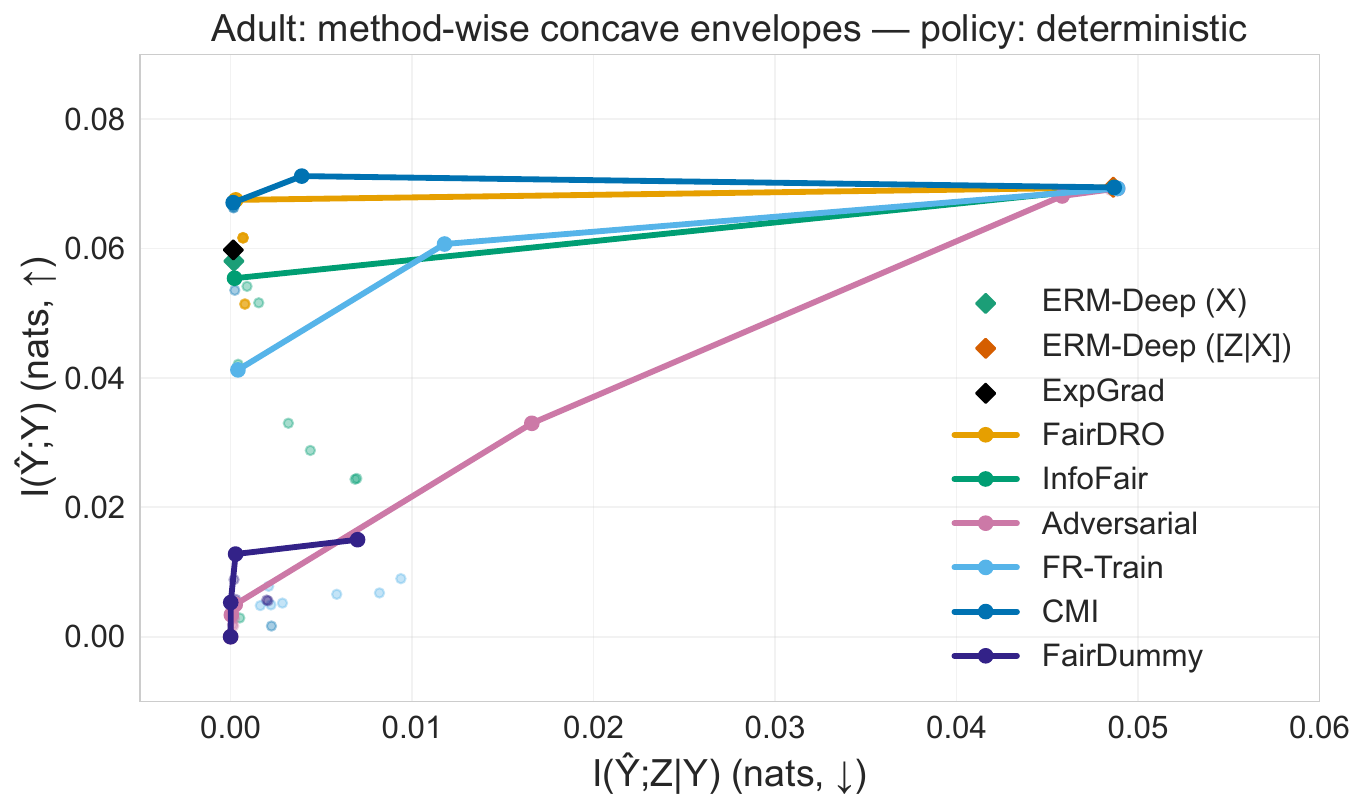}} \\
    \subfloat{\includegraphics[width=0.49\textwidth]{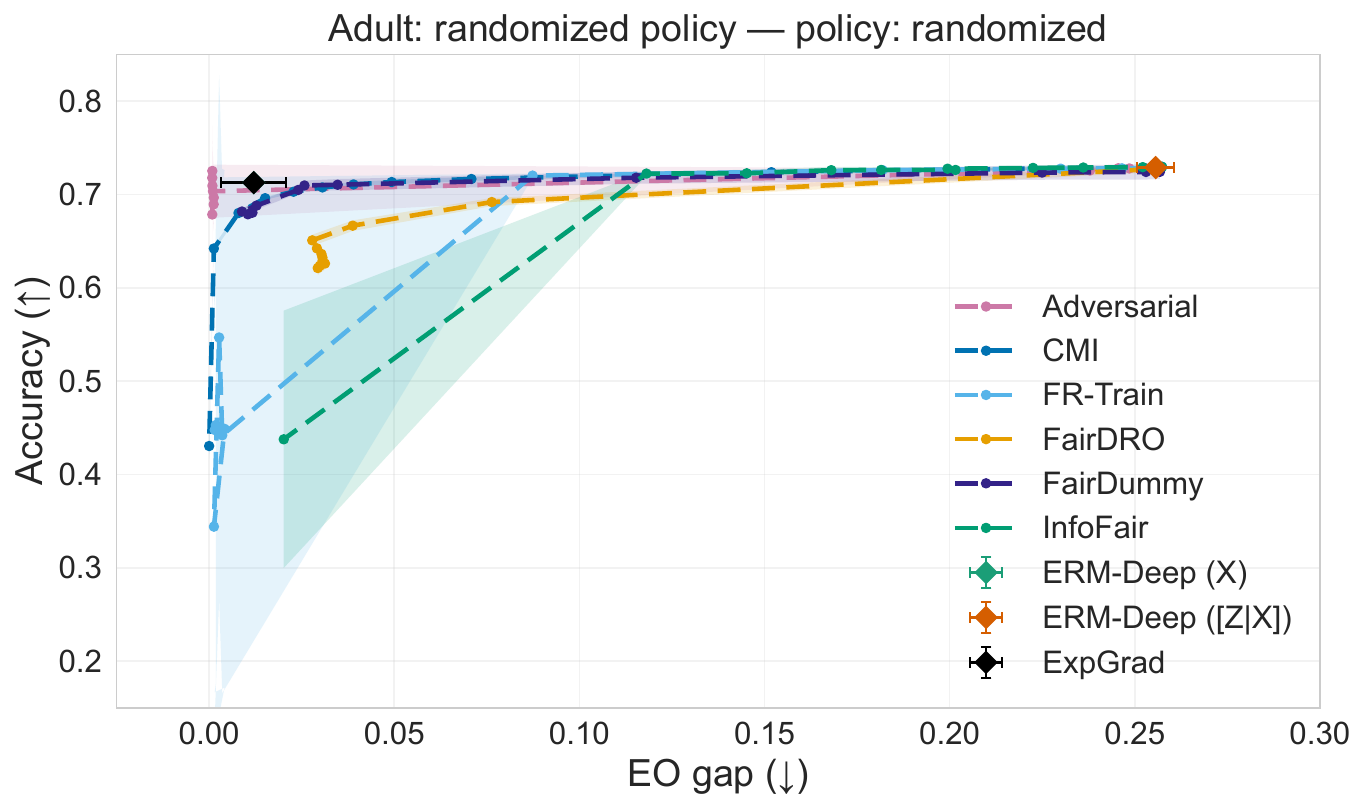}}
    \subfloat{\includegraphics[width=0.49\textwidth]{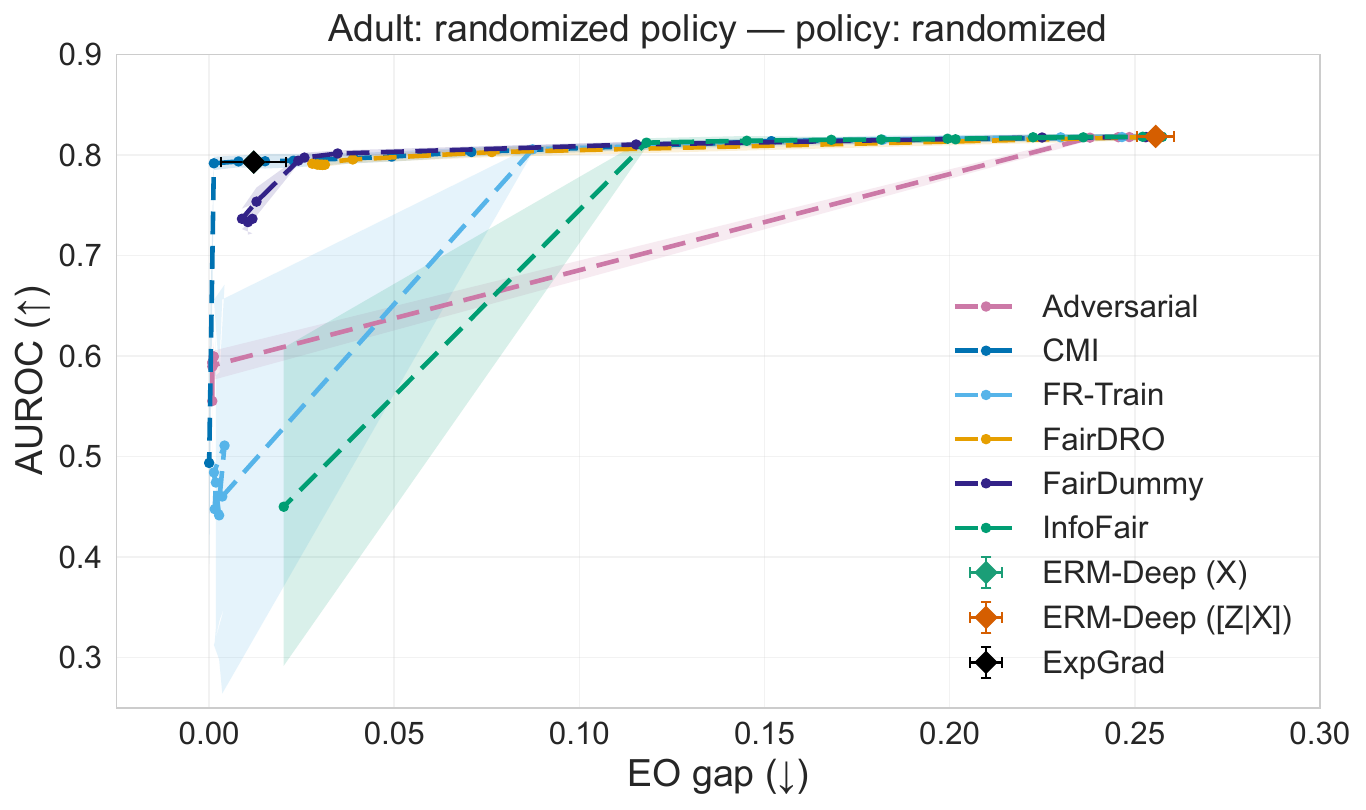}}
    \caption{\textbf{Adult Results: Low Marginal Cost, Necessary Trade-off, and Operational Transfer.}
    \textbf{Top and Middle:} Adult exhibits a relatively low marginal cost of separation compared with Bank: utility decreases gradually as $I(\hat Y;Z\mid Y)$ is reduced toward zero. That is, the marginal cost of separation increases very slowly. The randomized-policy information plane also illustrates the necessary-tradeoff mechanism in Theorem~\ref{thm:necessary-tradeoff}: the ERM-Deep $(X)$ baseline lies near the strict-separation axis, which provides empirical evidence for the condition $X \perp Z \mid Y$, while ERM-Deep $([Z|X])$ achieves higher utility only at positive separation violation. Among the comparison methods, FairDRO gives the strongest randomized-policy information-plane frontier, while CMI is strongest in the deterministic-policy frontier, both by a small margin.
    \textbf{Bottom:} The operational plots evaluate whether the information-plane frontier translates to deployment metrics. CMI shows effective transfer to the Accuracy-EO-gap and AUROC-EO-gap planes: it maintains competitive Accuracy and near-optimal AUROC at negligible EO gaps, with a stable and monotone traversal. In contrast, FairDRO's strong randomized-policy information-plane performance does not translate as effectively to the operational metrics.}
    \label{fig:Adult_combined}
\end{figure*}

\begin{figure*}[t]
    \centering
    \setlength{\tabcolsep}{1pt}
    
    \subfloat{\includegraphics[width=0.49\textwidth]{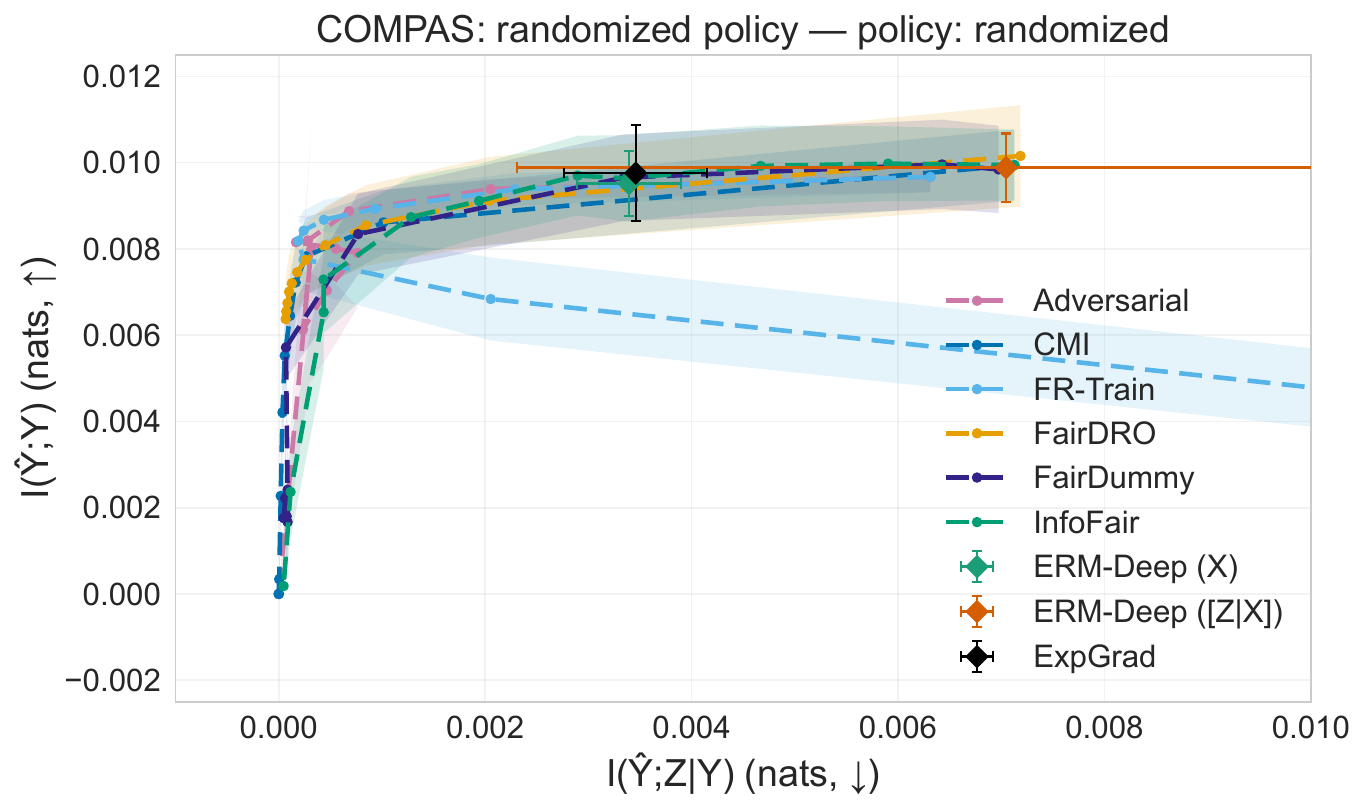}}
    \subfloat{\includegraphics[width=0.49\textwidth]{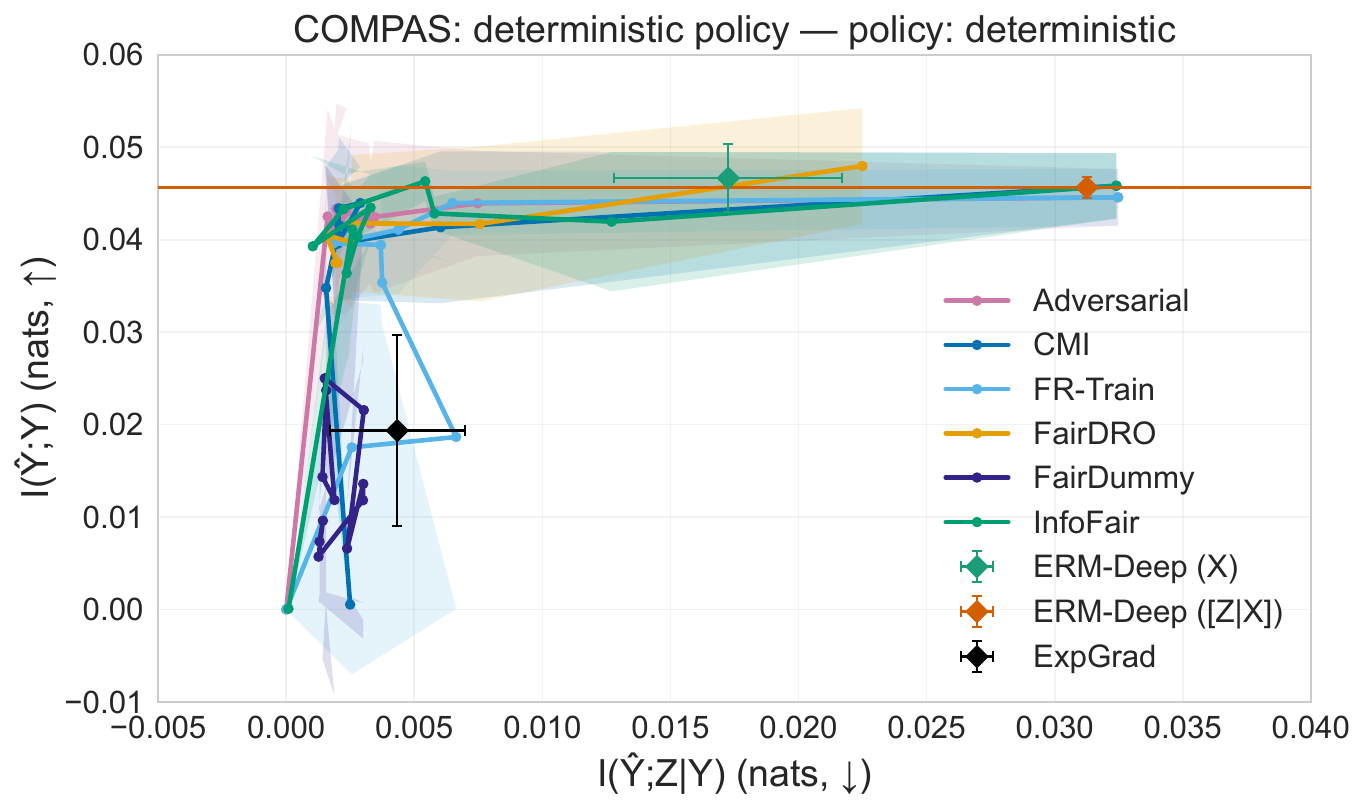}} \\
    \subfloat{\includegraphics[width=0.49\textwidth]{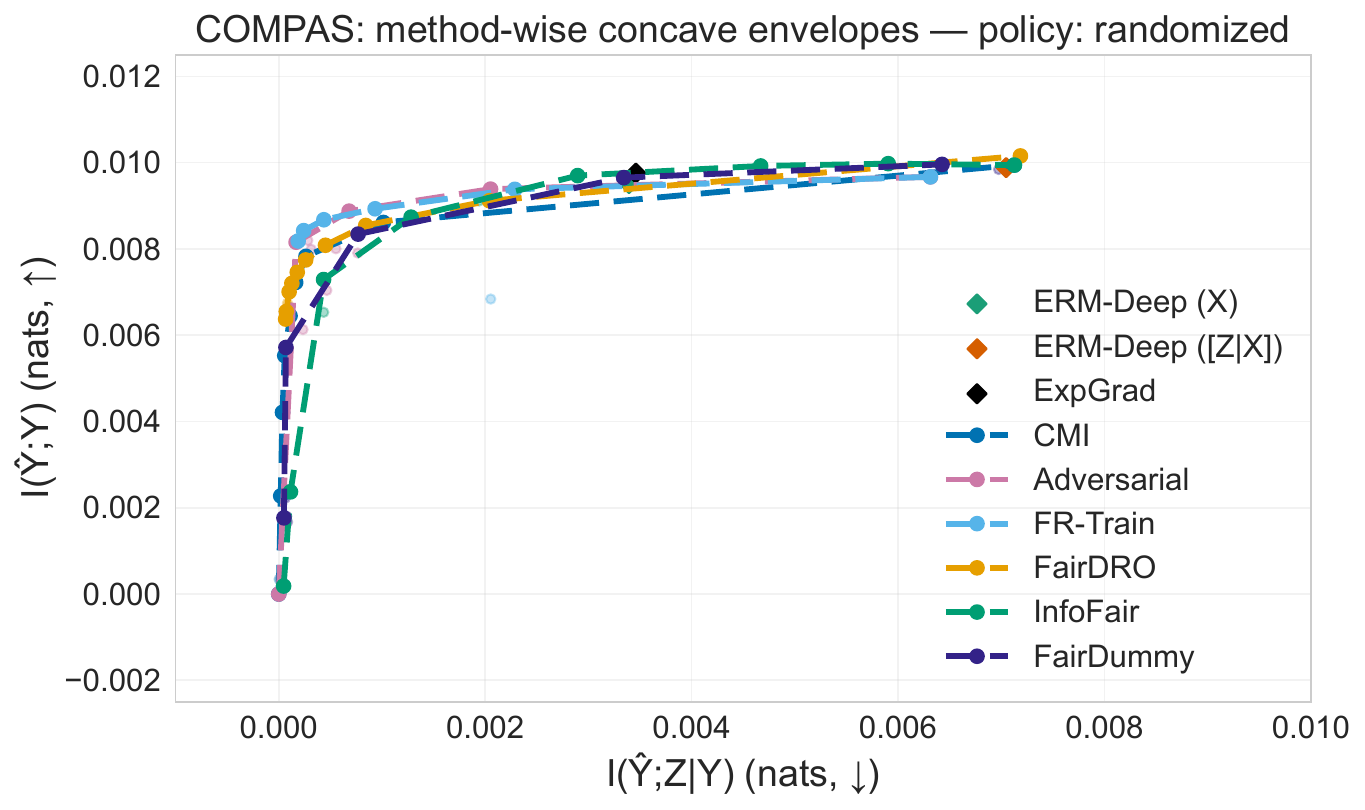}}
    \subfloat{\includegraphics[width=0.49\textwidth]{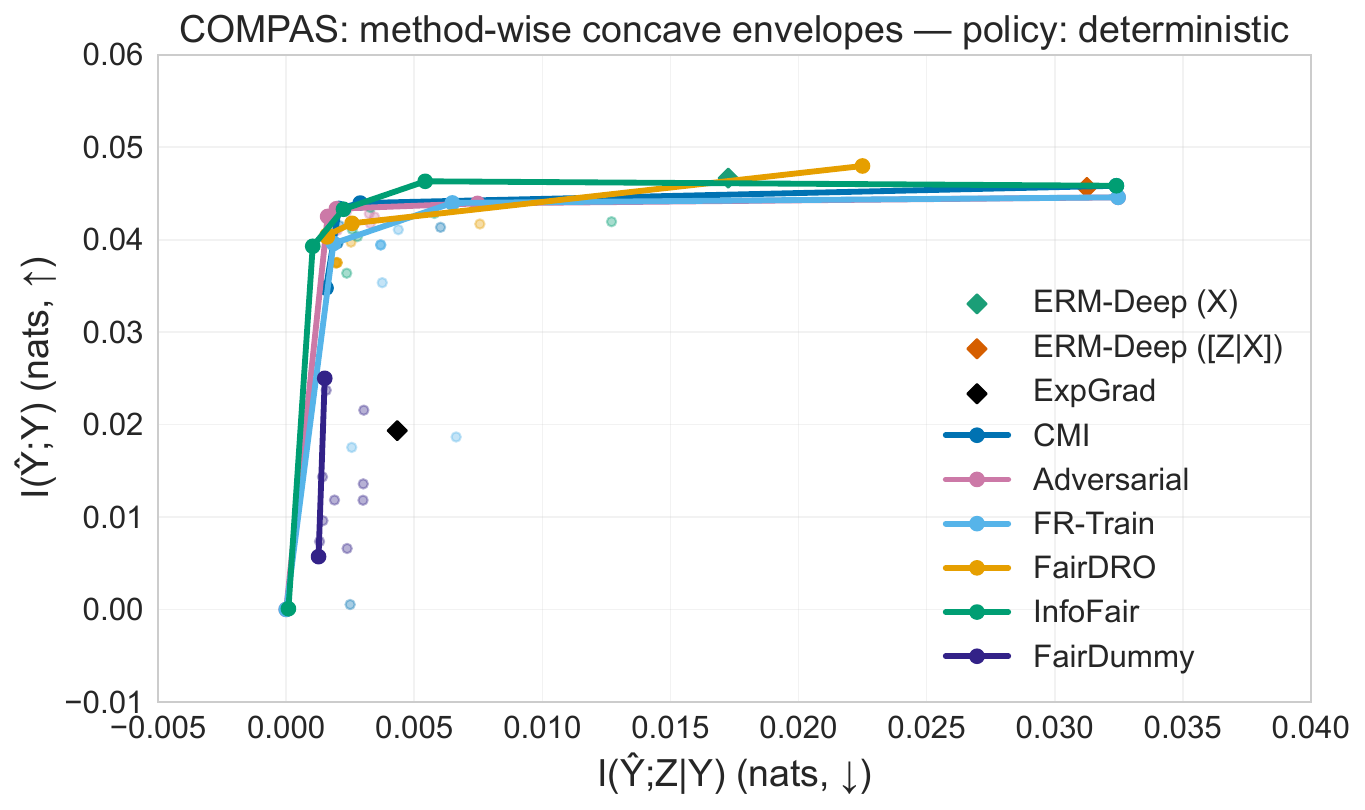}} \\
    \subfloat{\includegraphics[width=0.49\textwidth]{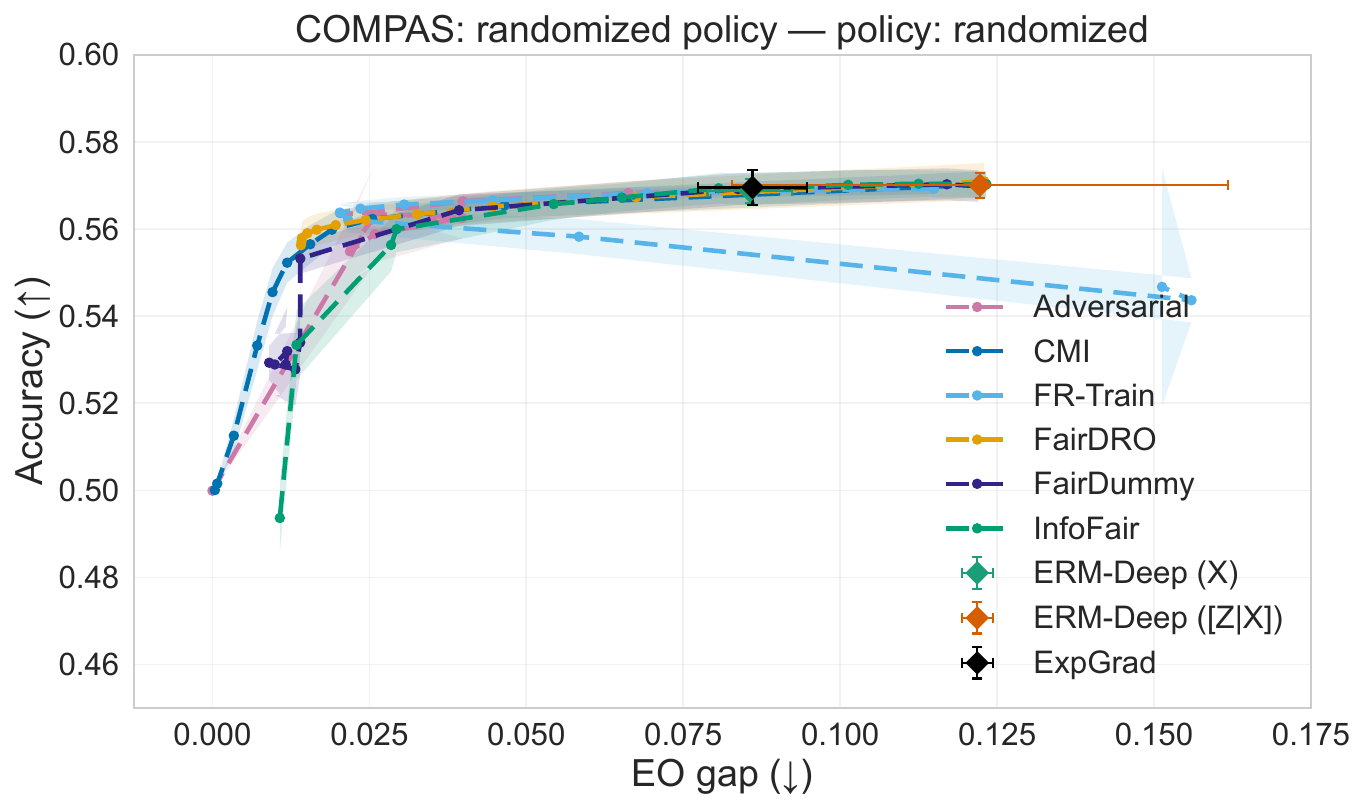}}
    \subfloat{\includegraphics[width=0.49\textwidth]{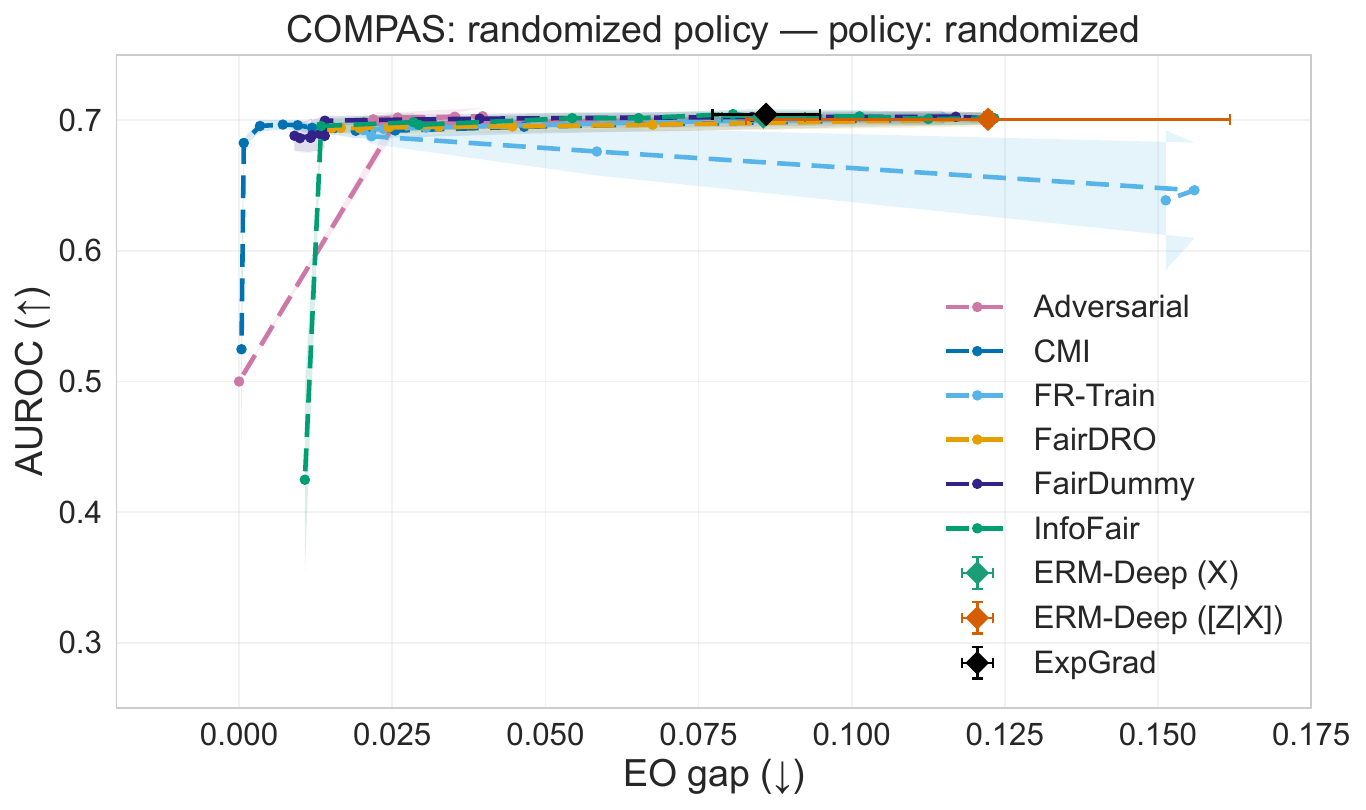}}
    \caption{\textbf{COMPAS Results: Low Marginal Cost of Separation and Operational Transfer.}
    \textbf{Top and Middle:} COMPAS exhibits a relatively low marginal cost of separation. In the randomized-policy setting (left), CMI and FairDRO are both top-performing methods by generating a stable and monotone Pareto frontier approximation and traversal as $I(\hat{Y};Z \mid Y)$ goes to zero. In contrast, FR-Train and Adversarial achieve comparable or slightly better utility when $I(\hat{Y};Z \mid Y)$ is not too small, but start to show backtracking and utility collapse near the strict-separation region. In the deterministic-policy setting (right), most methods achieve similar high-utility frontiers, with InfoFair slightly leading and CMI and Adversarial following closely, while FairDummy remains noticeably weaker.
    \textbf{Bottom:} The operational plots evaluate whether the information-theoretic frontier traversal transfers to deployment metrics. CMI shows the clearest operational transfer: it reduces the EO gap close to zero while retaining high AUROC and competitive Accuracy. In comparison, several baselines either degrade more sharply or do not reduce the EO gap as effectively, despite strong information-plane performance.}
    \label{fig:COMPAS_combined}
\end{figure*}

Figures~\ref{fig:Adult_combined} and~\ref{fig:COMPAS_combined} provide two complementary consistency checks beyond Bank and CelebA. Adult gives the clearest empirical illustration of the necessary-tradeoff mechanism in Theorem~\ref{thm:necessary-tradeoff}, while COMPAS serves as an additional tabular benchmark for testing whether the observed information-plane behavior transfers across domains.

\textbf{Empirical evidence for the necessary trade-off.}
The Adult randomized-policy information plane in Figure~\ref{fig:Adult_combined} provides a direct empirical illustration of Theorem~\ref{thm:necessary-tradeoff}. The ERM-Deep $(X)$ baseline lies near the strict-separation axis, with empirical violation
\[
    v_X^{\mathrm{emp}}
    =
    I(U_X;Z\mid Y)
    \approx 0,
\]
providing empirical evidence consistent with the condition $X \perp Z \mid Y$. Its utility
\[
    u_X^{\mathrm{emp}}
    =
    I(U_X;Y)
\]
serves as an empirical approximation of the best $X$-only utility level $u_X^*$. In contrast, the ERM-Deep $([Z|X])$ baseline attains higher utility, but only at strictly positive conditional dependence:
\[
    I(U_{X,Z};Z\mid Y)>0.
\]
Thus, the learned Pareto frontier mirrors the theorem's qualitative prediction: utilities above the empirical $X$-only level are attained only after incurring positive separation violation.

\textbf{Frontier geometry and low marginal cost.}
In contrast to the steep small-violation trade-off observed on Bank, both Adult and COMPAS exhibit a lower and more stable marginal cost of separation. On Adult, utility decreases gradually as $I(\hat Y;Z\mid Y)$ is reduced toward zero, as shown by FairDRO in the randomized-policy setting and CMI in the deterministic-policy setting. On COMPAS, CMI and FairDRO both trace strong randomized-policy frontier approximations near the strict-separation regime. These results suggest that the information-plane frontier provides useful guidance by revealing dataset-specific frontier geometry.

\textbf{Operational transfer.}
The bottom rows of Figures~\ref{fig:Adult_combined} and~\ref{fig:COMPAS_combined} show that information-plane improvements can transfer to deployment metrics. On Adult, CMI maintains competitive Accuracy and near-optimal AUROC at negligible EO gaps, despite FairDRO's stronger randomized-policy information-plane frontier. On COMPAS, CMI reduces the EO gap close to zero while retaining high AUROC and competitive Accuracy. Across both datasets, this supports the practical relevance of directly controlling $I(U;Z\mid Y)$ as a separation violation, since the information-theoretic frontier traversal translates into favorable EO-Accuracy and EO-AUROC trade-offs.

\subsubsection{ACS Benchmark: Multi-Class and Multi-Group Separation}
\label{subsubsec:acs}

To assess whether the proposed separation-utility frontier extends beyond legacy binary benchmarks, we further evaluate on an ACS task derived from the American Community Survey~\cite{ACS}. ACS provides large-scale tabular prediction problems and is increasingly used as a modern fairness benchmark. Our goal here is two-fold. First, we test whether the information-theoretic formulation remains meaningful in a structurally more complex tabular setting. Second, and more importantly, we test the non-binary setting in which both the target and the sensitive attribute have more than two categories.

In our main ACSOccupation experiment, \(Y\) is the occupation code restricted to the top \(20\) most frequent occupation categories, and \(Z\) is the multi-group race attribute RAC1P with \(9\) observed groups. The input \(X\) consists of demographic and socioeconomic ACS covariates, one-hot encoded into a \(55\)-dimensional feature vector. This main version uses a conservative complete-case preprocessing rule: rows with missing or structurally unavailable covariate values are removed before one-hot encoding. Because ACS contains many structurally unavailable covariate entries, this rule leaves \(2{,}576\) samples. Thus, unlike the legacy binary benchmarks, ACSOccupation tests separation in a genuinely multivariate label--group setting:
\[
    |\mathcal Y|=20,
    \qquad
    |\mathcal Z|=9,
    \qquad
    |\mathcal Y|\cdot|\mathcal Z|=180.
\]
To assess the sensitivity of the ACS conclusions to this conservative complete-case preprocessing, Appendix~\ref{append:acs_large} reports an additional large-scale ACSOccupation variant in which missing or structurally unavailable covariate values are retained as explicit categorical levels rather than removed.

For ACSOccupation, we compare against baselines that naturally support multi-class prediction and multi-group sensitive attributes. Binary post-processing methods such as ExpGrad and threshold-based equalized-odds repair are excluded because their standard implementations are designed for binary decisions and do not directly apply to multi-class, multi-group separation without additional reductions or post-processing choices. We also exclude FairDummy in this experiment. Although its high-level idea could be extended by sampling a categorical dummy sensitive attribute $\widetilde Z\sim P(Z\mid Y)$, a faithful multi-class/multi-group implementation would require nontrivial design choices and would become closely related to adversarial separation learning. To avoid introducing an additional nonstandard baseline with its own implementation choices, we reserve this extension for future work and use ACSOccupation as a clean comparison among methods that naturally extend to multi-class, multi-group separation.

\begin{figure*}[t]
    \centering
    \setlength{\tabcolsep}{1pt}
    
    \subfloat{\includegraphics[width=0.49\textwidth]{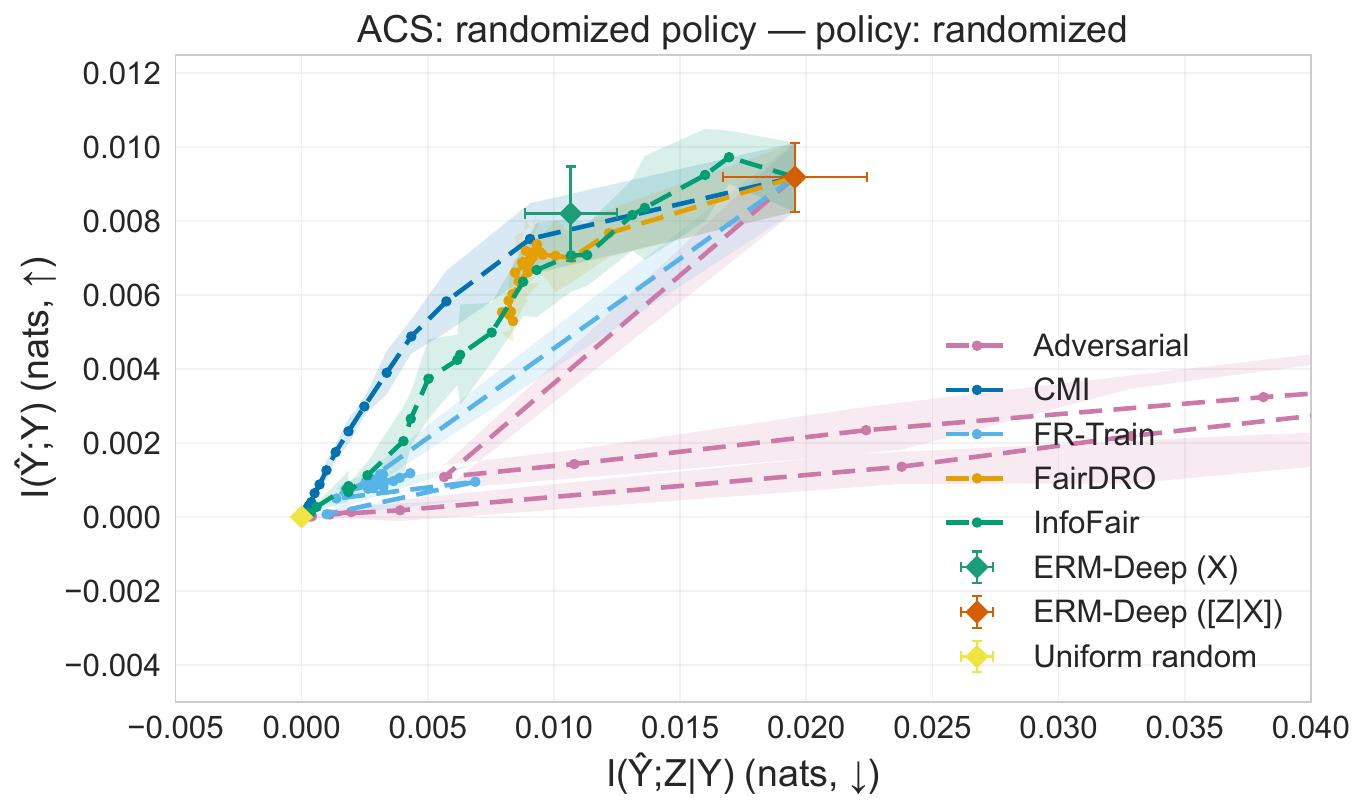}}
    \subfloat{\includegraphics[width=0.49\textwidth]{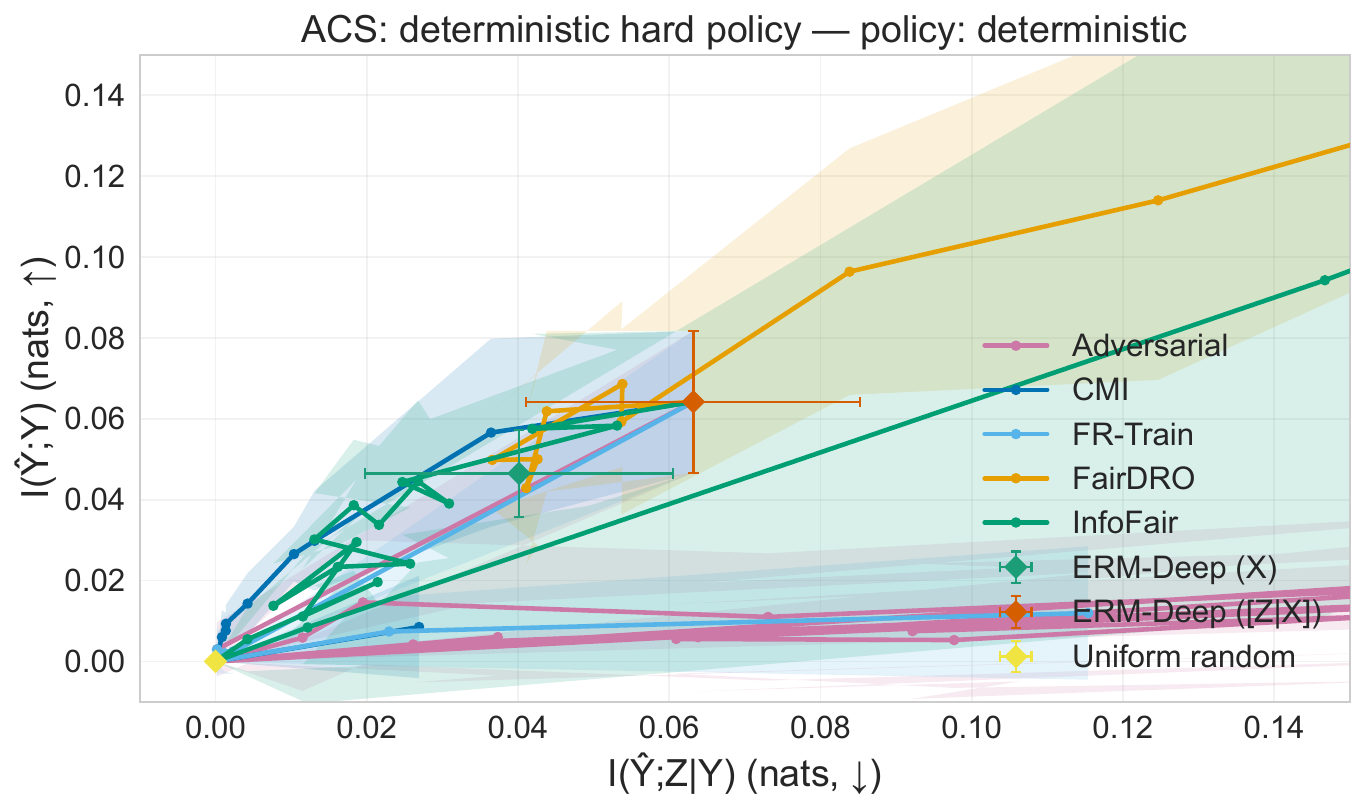}} \\
    \subfloat{\includegraphics[width=0.49\textwidth]{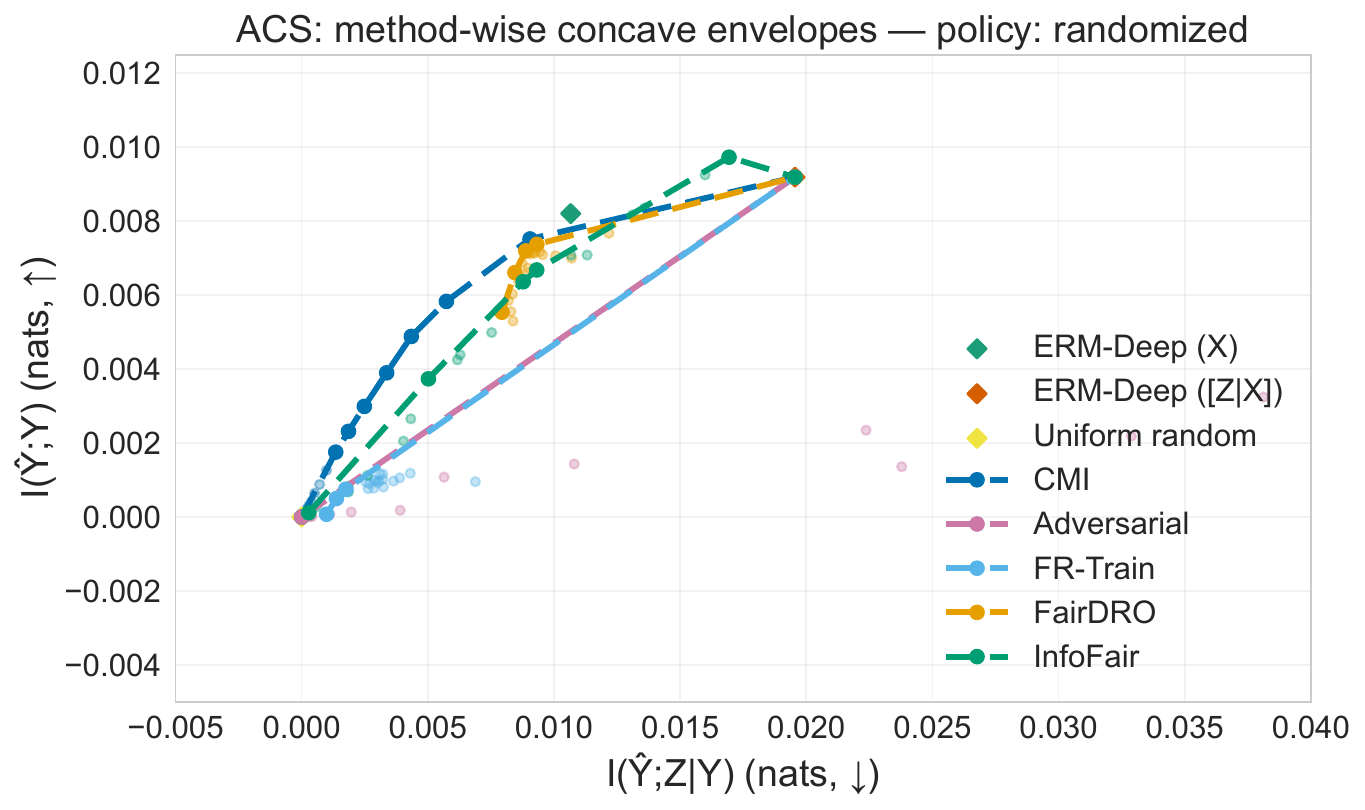}}
    \subfloat{\includegraphics[width=0.49\textwidth]{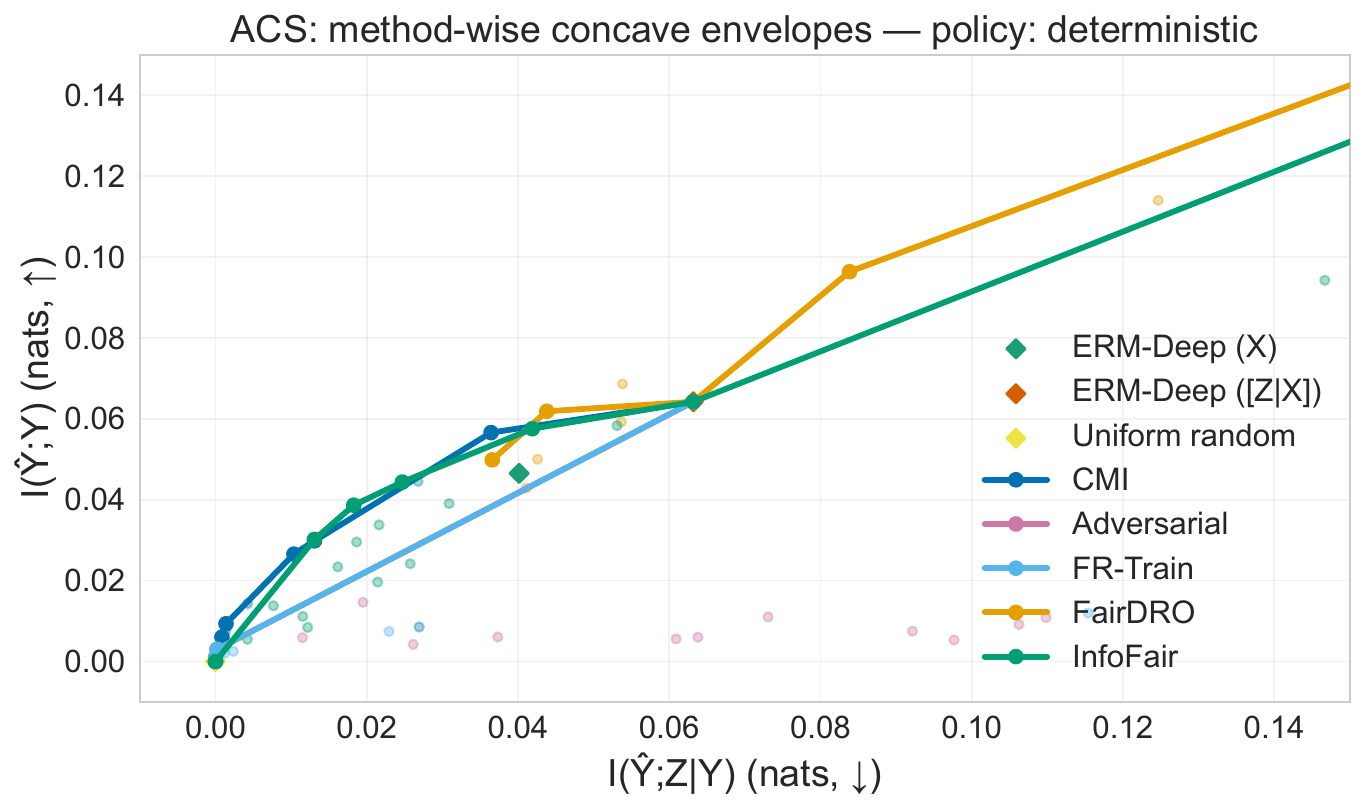}} \\
    \subfloat{\includegraphics[width=0.49\textwidth]{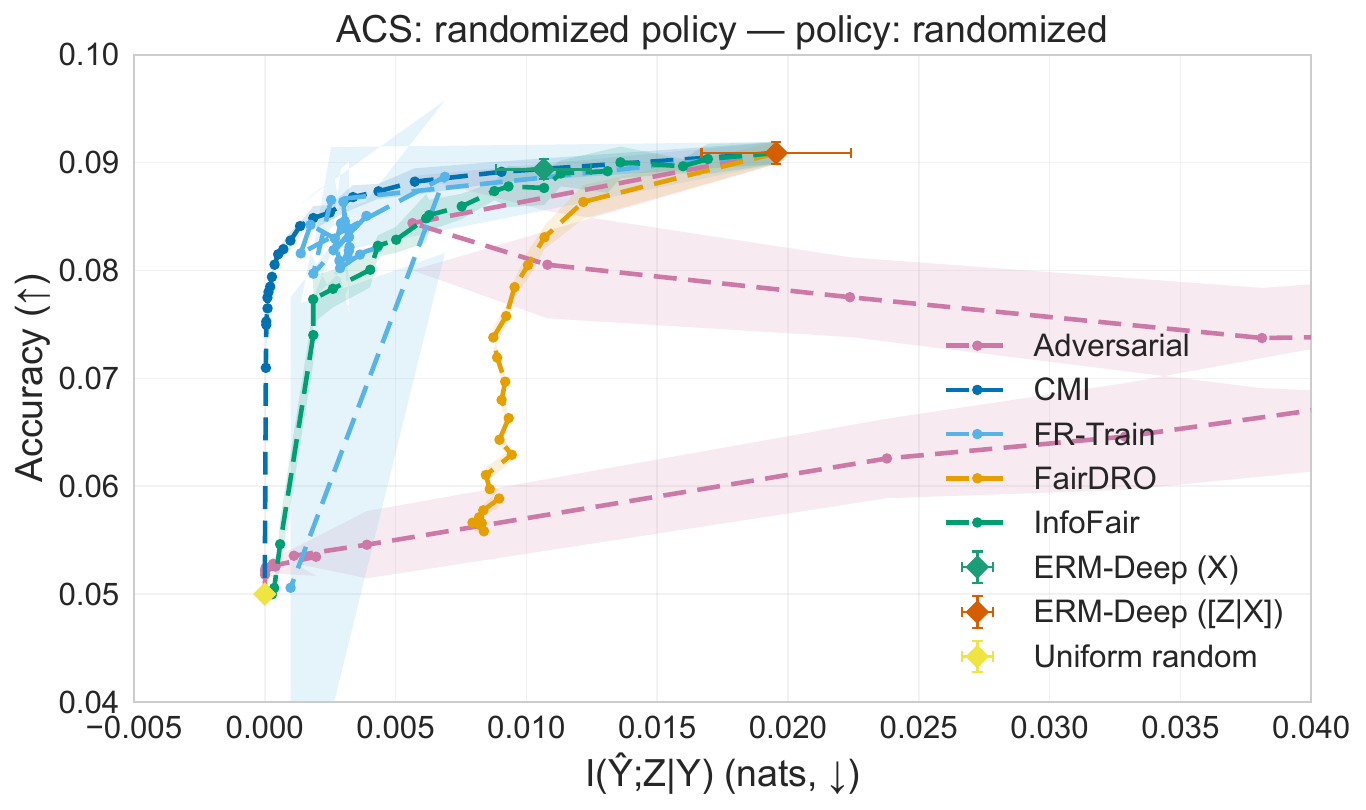}}
    \subfloat{\includegraphics[width=0.49\textwidth]{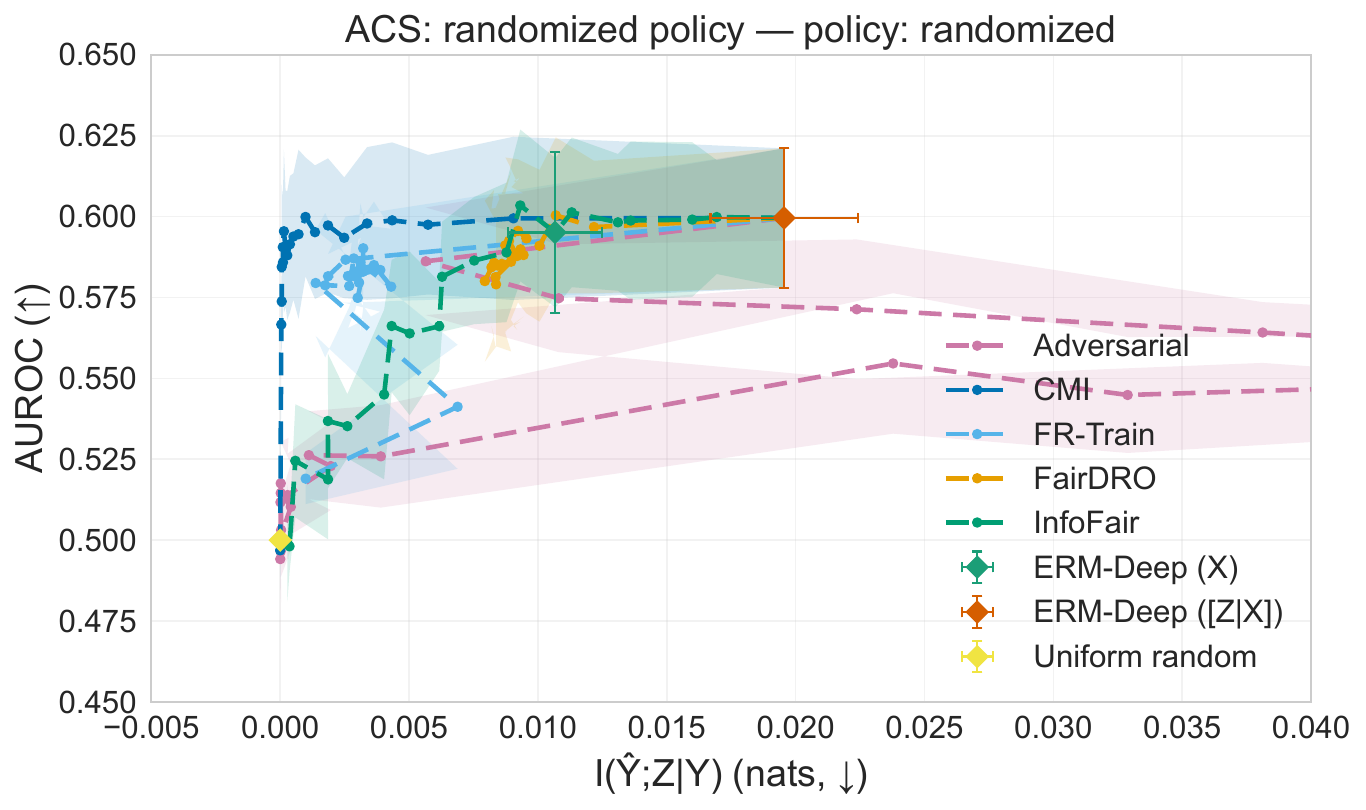}} \\
    \caption{\textbf{ACSOccupation Results: Multi-class and Multi-group Separation.}
    \textbf{Top:} Information-plane frontier approximations for randomized-policy (posterior-level) and deterministic-policy predictions, plotting utility $I(\hat Y;Y)$ against separation violation $I(\hat Y;Z\mid Y)$. CMI traces a stable and monotone frontier approximation while preserving relatively high utility, especially in the low-separation-violation region. In contrast, the other baselines exhibit lower utility, backtracking, or less targeted movement in the information plane.
    \textbf{Middle:} The method-wise upper concave envelopes summarize empirical approximations of the revealed-randomization frontier over the observed trained predictors. These plots provide a cleaner comparison of method-wise attainable frontiers in the multi-class and multi-group setting.
    \textbf{Bottom (Operational Transfer):} Expected randomized accuracy and macro-AUROC are plotted against the separation violation. CMI transfers its stable information-theoretic frontier traversal to deployment-metric trade-offs, maintaining strong accuracy and AUROC in the low-separation-violation region along a relatively smooth trajectory compared with the baselines.}
    \label{fig:ACS_combined}
\end{figure*}

Figure~\ref{fig:ACS_combined} shows the multi-class and multi-group results. The six panels evaluate complementary aspects of the learned frontier: raw information-plane behavior, method-restricted concave envelopes obtained from the observed trained predictors, and operational transfer to expected randomized accuracy and macro-AUROC.

\begin{itemize}[leftmargin=*]
    \item \textbf{Difficulty of the multi-class, multi-group task.}
    ACSOccupation is substantially harder than the binary benchmarks. The absolute values of $I(\hat Y;Y)$ and predictive accuracy are modest because the task requires predicting a $20$-class occupation label from coarse ACS covariates. Nevertheless, the ERM baselines and comparison methods achieve nontrivial performance above the uniform predictor, confirming that the task contains predictive signal while remaining challenging.

    \item \textbf{Natural generalization as both quantification and regularization.}
    The information-plane plots demonstrate that CMI remains effective in the genuinely multivariate setting. The CMI frontier moves from the unconstrained ERM endpoint toward the strict-separation region with a smooth and interpretable trajectory. This is precisely the setting in which the proposed information-theoretic formulation is most useful: no binary reduction, pairwise group comparison, or post-hoc threshold design is needed. It directly handles the multi-class outcome and the multi-group sensitive attribute. The envelope row further shows that the method-wise revealed-randomization diagnostic remains meaningful in this multi-class and multi-group setting.

    \item \textbf{Operational transfer.}
    The bottom row shows that posterior-level separation control does not merely minimize a formal information quantity. The CMI method preserves competitive expected randomized accuracy and macro-AUROC throughout the low-separation-violation region. In comparison, InfoFair obtains lower accuracy and AUROC performance, FR-Train suffers from instability and backtracking in low-separation-violation region, and FairDRO has difficulty in reducing the violation to zero. Adversarial training exhibits a less stable frontier approximation.
\end{itemize}

We report additional ACS ablations in Appendix~\ref{append:acs_ablation}, including batch-size sensitivity of the CMI estimator and raw versus gradient-normalized CMI training. The batch-size ablation evaluates whether the empirical CMI signal remains stable when the number of samples per $(Y,Z)$ cell changes, which is important because ACSOccupation has $|\mathcal Y|\cdot|\mathcal Z|=180$ conditional cells. The gradient-normalization ablation verifies that the proposed normalized objective yields a smoother and more reliable frontier than the unnormalized CMI penalty. Together, these ablations support that (1) the observed ACS trends are not artifacts of a particular batch size and (2) gradient normalization is important for stable and smooth frontier approximation.

\section{Conclusion}
\label{sec:conclusion}

This work connects fundamental fairness limits with practical optimization by giving a model-agnostic characterization of the separation-utility Pareto frontier. On the theoretical side, we formulated separation violation through conditional mutual information and proved that the optimal randomized frontier is the concave closure of the deterministic frontier via revealed randomization (Theorem~\ref{thm:timesharing-frontier-polish}). We also identified conditions under which improving utility beyond the best $X$-only predictor necessarily incurs positive separation violation. On the empirical side, we showed that, for discrete target and sensitive variables, complex learned proxies are often unnecessary: a direct plug-in estimator of conditional mutual information yields stable empirical frontiers, directly controls posterior-level separation violation, and transfers effectively to operational EO-Accuracy and EO-AUROC trade-offs across standard and multi-class/multi-group benchmarks.

\textbf{Open challenges.}
Despite these advances, critical questions remain.
\emph{Theoretical necessity:} While we identify sufficient conditions under which utility gains beyond the best $X$-only predictor necessarily incur positive separation violation, a full characterization of when such strict trade-offs are unavoidable for arbitrary distributions remains open.
\emph{Continuous and high-dimensional estimation:} Our direct estimator relies on the discreteness of the target and sensitive variables. Extending transparent, theoretically guaranteed CMI estimators to continuous or high-dimensional targets and sensitive attributes, without reverting to unstable adversarial or variational proxies, remains an important direction for future work.

\newpage

\bibliography{main}
\bibliographystyle{abbrvnat}

\appendix

\newpage

\addtocontents{toc}{\protect\setcounter{tocdepth}{2}}
\tableofcontents 

\newpage

\section{Appendix to Section \ref{sec:introduction}}

\subsection{Related Machine Learning Fairness Definitions}

We briefly review several standard statistical fairness definitions to situate separation relative to other commonly used notions. These definitions are included for background only. The main paper focuses on separation.

\textbf{1. Independence, statistical parity, or demographic parity.}
Independence requires the prediction to be statistically independent of the sensitive attribute:
\[
    \hat{Y} \perp Z.
\]
Equivalently, for any measurable event \(A_{\mathcal Y}\in\mathcal B_{\mathcal Y}\) and any \(z,z'\in\mathcal Z\),
\[
    P(\hat{Y}\in A_{\mathcal Y}\mid Z=z)
    =
    P(\hat{Y}\in A_{\mathcal Y}\mid Z=z').
\]
In binary classification, this reduces to
\[
    P(\hat{Y}=1\mid Z=z)=P(\hat{Y}=1),
    \qquad \forall z\in\mathcal Z.
\]
This criterion asks for equal prediction distributions across groups, without conditioning on the true outcome.

\textbf{2. Separation or equalized odds.}
Separation requires the prediction to be conditionally independent of the sensitive attribute given the true outcome:
\[
    \hat{Y} \perp Z\mid Y.
\]
Equivalently, for any measurable event \(A_{\mathcal Y}\in\mathcal B_{\mathcal Y}\), any \(y\in\mathcal Y\), and any \(z,z'\in\mathcal Z\),
\[
    P(\hat{Y}\in A_{\mathcal Y}\mid Y=y,Z=z)
    =
    P(\hat{Y}\in A_{\mathcal Y}\mid Y=y,Z=z').
\]
In binary classification, this is equalized odds:
\[
    P(\hat{Y}=1\mid Y=y,Z=z)
    =
    P(\hat{Y}=1\mid Y=y),
    \qquad y\in\{0,1\},\ z\in\mathcal Z.
\]
This criterion requires the conditional error behavior of the predictor to be comparable across sensitive groups. Common relaxations include equal opportunity, which matches true positive rates,
\[
    P(\hat{Y}=1\mid Y=1,Z=z)
    =
    P(\hat{Y}=1\mid Y=1),
\]
and predictive equality, which matches false positive rates,
\[
    P(\hat{Y}=1\mid Y=0,Z=z)
    =
    P(\hat{Y}=1\mid Y=0).
\]

\textbf{3. Sufficiency, predictive parity, or calibration within groups.}
Sufficiency requires the true outcome to be conditionally independent of the sensitive attribute given the prediction:
\[
    Y \perp Z\mid \hat{Y}.
\]
For a binary score or prediction, this corresponds to matching outcome rates across groups among individuals assigned the same prediction value:
\[
    P(Y=1\mid \hat{Y}=\hat y,Z=z)
    =
    P(Y=1\mid \hat{Y}=\hat y),
    \qquad \forall \hat y,\ z.
\]
For probabilistic scores, a related notion is calibration within groups:
\[
    P(Y=1\mid \widehat p(X)=s,Z=z)=s,
    \qquad \forall s\in[0,1],\ z\in\mathcal Z.
\]

\textbf{4. Other fairness notions.}
Other fairness definitions include individual fairness, subgroup fairness, treatment equality, and counterfactual fairness. These notions involve either metric assumptions, subgroup families, error-ratio constraints, or causal structure, and are therefore conceptually distinct from the conditional-independence notion of separation studied in this paper.

\section{Appendix to Section \ref{sec:theory-regularization}}
\subsection{Proof of Proposition \ref{prop:det-monotone}}\label{append:det-monotone}

\begin{proof}
If \(0\le v_1\le v_2\), then
\[
\{(v',u')\in \overline{\mathcal S_{\det}}: v'\le v_1\}
\subseteq
\{(v',u')\in \overline{\mathcal S_{\det}}: v'\le v_2\}.
\]
Taking suprema of the second coordinates over these two nested sets yields
\[
U^\star_{\det}(v_1)\le U^\star_{\det}(v_2).
\]
\end{proof}

\subsection{Proof of Theorem \ref{thm:timesharing-frontier-polish}}\label{append:timesharing-frontier-polish}

\begin{proof}
\textit{Step 1 (W.l.o.g. reduction to a uniform selector).} Any standard Borel probability space $(\mathcal B,\mu)$ can be represented by the push-forward of the unit interval $([0,1],\lambda)$ with Lebesgue measure $\lambda$ under a measurable map (See, for example, \cite[Theorem 1.8]{Kallenberg2002}). This means any revealed randomization scheme can be realized, without loss of generality, by a selector $B\sim\mathrm{Unif}[0,1]$ and a jointly measurable $F:\mathcal X\times\mathcal Z\times[0,1]\to\mathcal U$. We assume this simplified representation for the remainder of the proof.

\textit{Step 2 (Average identity).} Let $U_b:=f_b(X,Z)$ and $\widetilde U:=(B,U_B)$ with $B\sim\mathrm{Unif}[0,1]$. Since all spaces are Polish with Borel \(\sigma\)-algebras, the relevant regular conditional distributions exist. Therefore, the chain rule for mutual information gives
\begin{align*}
    I(\widetilde U;Y) & = I(B,U_B;Y) = I(B;Y) + I(U_B;Y\mid B) \\
    I(\widetilde U;Z\mid Y) & = I(B,U_B;Z\mid Y) = I(B;Z\mid Y) + I(U_B;Z\mid Y,B)
\end{align*}
By assumption, $B$ is independent of $(X,Y,Z)$, which implies $B \perp Y$ and $B \perp Z \mid Y$. Therefore, $I(B;Y)=0$ and $I(B;Z\mid Y)=0$. The chain rules now simplify to:
\begin{align*}
    I(\widetilde U;Y) &= I(U_B;Y\mid B) = \int_0^1 I(U_b;Y)\,db \\
    I(\widetilde U;Z\mid Y) &= I(U_B;Z\mid Y,B) = \int_0^1 I(U_b;Z\mid Y)\,db
\end{align*}
The last equality in each line follows from disintegration and the measurability of the maps \(b \mapsto I(U_b;Y)\) and \(b \mapsto I(U_b;Z\mid Y)\). Thus, every randomization point on the $(v,u)$ plane is the average of some deterministic point parametrized by the realization $B = b$:
\begin{equation}\label{eq:barycenter}
\bigl(I(\widetilde U;Z\mid Y),\,I(\widetilde U;Y)\bigr)=\int_0^1 \bigl(v_{f_b},u_{f_b}\bigr)\,db.
\end{equation}

\textit{Step 3 ($\mathrm{conv}(\mathcal{S}_{\det})\subseteq \mathcal{S}_{\rand}$).} We show that any finite convex combination $y = \sum_{j=1}^J\alpha_j(v_{f_j},u_{f_j})$ is in $\mathcal{S}_{\rand}$, so that $\mathrm{conv}(\mathcal{S}_{\det})\subseteq \mathcal{S}_{\rand}$. We construct a time-sharing scheme by partitioning the interval $[0,1]$ into disjoint intervals $A_j$ with lengths $|A_j| = \alpha_j$. We then define the measurable map $F$ such that $f_b(x,z) := f_j(x,z)$ for all $b \in A_j$.
Applying the integral identity \eqref{eq:barycenter} to this scheme:
\[
\int_0^1 (v_{f_b},u_{f_b})\,db = \sum_{j=1}^J \int_{A_j} (v_{f_j},u_{f_j})\,db = \sum_{j=1}^J |A_j| (v_{f_j},u_{f_j}) = y.
\]
Since $y$ is achieved by a valid scheme, $y \in \mathcal{S}_{\rand}$. Thus, $\mathrm{conv}(\mathcal{S}_{\det})\subseteq\mathcal{S}_{\rand}$.

\textit{Step 4 ($\mathcal{S}_{\rand}\subseteq \overline{\mathrm{conv}}(\mathcal{S}_{\det})$).}
Let $y$ be any point in $\mathcal{S}_{\rand}$ achieved by some randomized scheme
$(B,\widetilde U_B)$. By \eqref{eq:barycenter},
\[
    y=\int_0^1 s(b)\,db,
    \qquad
    s(b):=(v_{f_b},u_{f_b})\in \mathcal S_{\det}.
\]
Since \(s(b)=(v_{f_b},u_{f_b})\in\mathbb R_{\ge0}^2\) for almost every \(b\), we may use the \(\ell^1\)-norm on \(\mathbb R^2\), for which
\[
\|s(b)\|_1 = v_{f_b}+u_{f_b}.
\]
By the identities in Step~2,
\[
\int_0^1 u_{f_b}\,db = I(\widetilde U;Y)<\infty,
\qquad
\int_0^1 v_{f_b}\,db = I(\widetilde U;Z\mid Y)<\infty.
\]
Hence,
\[
\int_0^1 \|s(b)\|_1\,db
=
\int_0^1 \bigl(v_{f_b}+u_{f_b}\bigr)\,db
<\infty,
\]
and therefore \(s\in L^1([0,1];\mathbb R^2)\).

We claim that
\[
    \int_0^1 s(b)\,db
    \in
    \overline{\mathrm{conv}}(\mathcal S_{\det}).
\]
Indeed, suppose for contradiction that
\[
    y\notin \overline{\mathrm{conv}}(\mathcal S_{\det}).
\]
Since \(\overline{\mathrm{conv}}(\mathcal S_{\det})\) is a closed convex subset
of \(\mathbb R^2\), the separating hyperplane theorem gives a vector
\(a\in\mathbb R^2\) and a constant \(\alpha\in\mathbb R\) such that
\[
    a\cdot y>\alpha
    \qquad\text{and}\qquad
    a\cdot x\le \alpha
    \quad
    \text{for all }x\in \overline{\mathrm{conv}}(\mathcal S_{\det}).
\]
Since \(s(b)\in\mathcal S_{\det}\subseteq
\overline{\mathrm{conv}}(\mathcal S_{\det})\) for almost every \(b\), we have $a\cdot s(b)\le \alpha$ for almost every $b\in[0,1]$. Integrating this inequality gives
\[
    a\cdot y
    =
    a\cdot \int_0^1 s(b)\,db
    =
    \int_0^1 a\cdot s(b)\,db
    \le
    \alpha,
\]
which contradicts \(a\cdot y>\alpha\). Therefore,
\[
    y\in \overline{\mathrm{conv}}(\mathcal S_{\det}).
\]
Since \(y\in\mathcal S_{\rand}\) was arbitrary, we conclude that $\mathcal S_{\rand}\subseteq \overline{\mathrm{conv}}(\mathcal S_{\det})$.

By Step 3 and 4, we have $\mathrm{conv}(\mathcal{S}_{\det}) \subseteq \mathcal{S}_{\rand}\subseteq \overline{\mathrm{conv}}(\mathcal{S}_{\det})$. By taking closure, we have $\overline{\mathcal{S}_{\rand}} = \overline{\mathrm{conv}}(\mathcal{S}_{\det})$. Finally, it follows from the definition of $U^\star_{\rand}(v)$ and $U_{\overline{\operatorname{conv}}(\mathcal{S}_{\det})}(v)$ that
\begin{equation*}
    U^\star_{\rand}(v) = U_{\overline{\operatorname{conv}}(\mathcal{S}_{\det})}(v), \qquad \forall v\ge 0.
\end{equation*}
That completes the proof.
\end{proof}

\subsection{Proof of Proposition \ref{prop:CMI_characterization}}\label{a:CMI_characterization}

\begin{proof}
By the disintegration formula for conditional mutual information,
\[
I(U;Z\mid Y)
=
\int
D_{\mathrm{KL}}
\!\left(
P_{U,Z\mid Y=y}
\,\middle\|\,
P_{U\mid Y=y}\otimes P_{Z\mid Y=y}
\right)
\,P_Y(dy).
\]
Since KL divergence is nonnegative, \(I(U;Z\mid Y)=0\) holds if and only if
\[
D_{\mathrm{KL}}
\!\left(
P_{U,Z\mid Y=y}
\,\middle\|\,
P_{U\mid Y=y}\otimes P_{Z\mid Y=y}
\right)
=0
\]
for \(P_Y\)-almost every \(y\). Since \(D_{\mathrm{KL}}(P\|Q)=0\) if and only if
\(P=Q\), this is equivalent to
\[
P_{U,Z\mid Y=y}
=
P_{U\mid Y=y}\otimes P_{Z\mid Y=y}
\]
for \(P_Y\)-almost every \(y\). This is exactly
\[
U\perp Z\mid Y
\]
by definition. We are done.
\end{proof}

\subsection{Proof of Lemma \ref{l:MI_bound}}\label{a:proof_MI_bound}

\begin{proof}
Let \(h:\mathcal U\to\mathbb R^d\) and
\(g:\mathcal Z\to\mathbb R^d\) be bounded measurable functions. By replacing
\(h\) and \(g\) with \(h-\mathbb E[h(U)]\) and
\(g-\mathbb E[g(Z)]\), respectively, it suffices to prove the bound in the
centered case. Thus assume
\[
    \mathbb E[h(U)]=0,
    \qquad
    \mathbb E[g(Z)]=0.
\]
Then we have:
\begin{align*}
    \frac{\left|\langle h(U), g(Z)\rangle_{L^2}\right|}
    {\|h\|_{L^{\infty}}\|g\|_{L^{\infty}}}
    &= \frac{1}{\|h\|_{L^{\infty}}\|g\|_{L^{\infty}}}
    \left|
    \int_{\mathcal{U}\times\mathcal{Z}}
    \langle h(u),g(z)\rangle\,
    d\mathbb{P}_{(U, Z)}(u,z)
    \right|\\[5pt]
    &= \frac{1}{\|h\|_{L^{\infty}}\|g\|_{L^{\infty}}}
    \left|
    \int_{\mathcal{U}\times\mathcal{Z}}
    \langle h(u),g(z)\rangle\,
    d(\mathbb{P}_{(U,Z)} - \mathbb{P}_{U}\otimes\mathbb{P}_{Z})(u,z)
    \right|\\[5pt]
    &\leq \frac{1}{\|h\|_{L^{\infty}}\|g\|_{L^{\infty}}}
    \int_{\mathcal{U}\times\mathcal{Z}}
    |\langle h(u),g(z)\rangle|\,
    d|\mathbb{P}_{(U,Z)} - \mathbb{P}_{U}\otimes\mathbb{P}_{Z}|(u,z)\\[5pt]
    &\leq \int_{\mathcal{U}\times\mathcal{Z}}
    d|\mathbb{P}_{(U,Z)} - \mathbb{P}_{U}\otimes\mathbb{P}_{Z}|(u,z)\\[5pt]
    &= 2\,\|\mathbb{P}_{(U,Z)} - \mathbb{P}_{U}\otimes\mathbb{P}_{Z}\|_{\mathrm{TV}} \\[5pt]
    &\leq 2\,\sqrt{\frac{1}{2}D_{\mathrm{KL}}\left(\mathbb{P}_{(U,Z)}\,\|\,\mathbb{P}_{U}\otimes\mathbb{P}_{Z}\right)}\\[5pt]
    &= \sqrt{2\,I(U;Z)}.
\end{align*}
Here, the fourth line follows from
\[
    |\langle h(u),g(z)\rangle|
    \leq
    \|h\|_{\infty}\|g\|_{\infty},
\]
the fifth from the definition of the total variation norm, the sixth from
Pinsker's inequality, and the last from the definition of mutual information.
We are done.
\end{proof}

\subsection{Examples of Lemma \ref{l:MI_bound}}\label{a:Lemma_MI_Bound_Examples}

On direct application or corollary of the above result is that one can effectively upper bound the covariance between $f(\hat{Y})$ and $g(Z)$ for any bounded $f$ and $g$:
If $I(\hat{Y};Z)\le C$, then for any bounded $f,g$,
\begin{equation*}
    \big|\operatorname{Cov}(f(\hat{Y}),g(Z))\big| \;\le\;\|f(\hat{Y})\|_{L^\infty}\,\|g(Z)\|_{L^\infty}\,\sqrt{2C}.
\end{equation*}

That is, a bound on mutual information uniformly controls covariance for \emph{all} bounded queries. Below, we provide two concrete examples:

\begin{exam}[Binary Indicator]
    Let $f(\hat Y)=\mathbf{1}\{h(\hat Y)=1\}$ and $g(Z)=\mathbf{1}\{a(Z)=1\}$ for (possibly randomized) binary classifiers $h,a$. Then
    \begin{align*}
    \mathrm{Cov}\!\bigl(\mathbf{1}\{h(\hat Y)=1\},\,\mathbf{1}\{a(Z)=1\}\bigr)
    &= \mathbb{P}\!\bigl(h(\hat Y)=1,\,a(Z)=1\bigr)
   - \mathbb{P}\!\bigl(h(\hat Y)=1\bigr)\,\mathbb{P}\!\bigl(a(Z)=1\bigr).
    \end{align*}
    Moreover, by Lemma~\ref{l:MI_bound},
    \begin{align*}
    \bigl|\mathrm{Cov}(f(\hat Y),g(Z))\bigr| \;\le\; \|f\|_\infty\|g\|_\infty\sqrt{2\,I(\hat Y;Z)}.
    \end{align*}
    Since $\|f\|_\infty=\|g\|_\infty=1$ for indicators, we obtain the explicit bound
    \begin{equation*}
        \bigl|\mathbb{P}(h(\hat Y)=1,\,a(Z)=1)-\mathbb{P}(h(\hat Y)=1)\mathbb{P}(a(Z)=1)\bigr| \;\le\; \sqrt{2\,I(\hat Y;Z)}.
    \end{equation*}
\end{exam}

In Lemma $\ref{l:MI_bound}$, we assume the knowledge of the population mutual information. In practice, one would apply the sample or empirical estimation of mutual information to upper bound the sample or empirical estimation of the covariance between $f(\hat{Y})$ and $g(Z)$ for any (essentially) bounded $f$ and $g$. The next result applies concentration inequality to show the empirical estimation error and corresponding probability guarantee, given a sample size, so that together with Lemma \ref{l:MI_bound}, we can provide a uniform covariance control guarantee only based on sample estimations.

\begin{exam}[1-dimensional Covariance Control]\label{coro: empirical covariance control}
    Given i.i.d.\ samples $\{(\hat Y_t,Z_t)\}_{t=1}^n$ and bounded $f,g$ with $M_f:=\|f(\hat Y)\|_\infty$ and $M_g:=\|g(Z)\|_\infty$, define
    \begin{equation*}
        \widehat{\mathrm{Cov}}_n(f,g) := \frac{1}{n}\sum_{t=1}^n f(\hat Y_t)g(Z_t) - \Bigl(\frac{1}{n}\sum_{t=1}^n f(\hat Y_t)\Bigr)\Bigl(\frac{1}{n}\sum_{t=1}^n g(Z_t)\Bigr).
    \end{equation*}
    Then, for any $\delta\in(0,1)$, with probability at least $1-\delta$,
    \begin{equation}\label{eq:emp-cov-conc}
    \bigl|\widehat{\mathrm{Cov}}_n(f,g)-\mathrm{Cov}(f(\hat Y),g(Z))\bigr| \;\le\; \,M_f M_g\,\sqrt{\frac{2\log(6/\delta)}{n}}.
    \end{equation}
\begin{proof}
Write
\begin{equation*}
    \widehat{\mathrm{Cov}}_n(f,g)-\mathrm{Cov}(f(\hat Y),g(Z)) = \underbrace{\Bigl(\tfrac{1}{n}\sum_{t=1}^n f_t g_t - \mathbb{E}[fg]\Bigr)}_{\text{(A)}} \;-\; \underbrace{\Bigl(\overline f_n\,\overline g_n - \mathbb{E}f\,\mathbb{E}g\Bigr)}_{\text{(B)}},
\end{equation*}
where $f_t:=f(\hat Y_t)$, $g_t:=g(Z_t)$, $\overline f_n:=n^{-1}\sum_t f_t$, $\overline g_n:=n^{-1}\sum_t g_t$.
By Hoeffding’s inequality, since $|f_t g_t|\le M_f M_g$,
\[
\mathbb{P}\!\left(\bigl|\text{(A)}\bigr| \ge M_f M_g\sqrt{\frac{2\log(6/\delta)}{n}}\right)\le \frac{\delta}{3}.
\]
Also,
\[
\bigl|\text{(B)}\bigr| \le |\overline f_n-\mathbb{E}f|\,|\overline g_n| + |\mathbb{E}f|\,|\overline g_n-\mathbb{E}g|
\le M_g\,|\overline f_n-\mathbb{E}f| + M_f\,|\overline g_n-\mathbb{E}g|.
\]
Applying Hoeffding's inequality again to each mean (since $|f_t-\mathbb{E}f|\le 2M_f$, $|g_t-\mathbb{E}g|\le 2M_g$) and a union bound yields, with probability at least $1-\tfrac{2\delta}{3}$,
\[
|\overline f_n-\mathbb{E}f| \le M_f\sqrt{\frac{2\log(6/\delta)}{n}},\qquad
|\overline g_n-\mathbb{E}g| \le M_g\sqrt{\frac{2\log(6/\delta)}{n}}.
\]
Combining the three events and simplifying constants gives \eqref{eq:emp-cov-conc}.
\end{proof}
Together with the population bound
\[
|\mathrm{Cov}(f(\hat Y),g(Z))|
\le
\|f\|_\infty\|g\|_\infty\sqrt{2I(\hat Y;Z)},
\]
we obtain, with probability at least \(1-\delta\),
\[
\bigl|\widehat{\mathrm{Cov}}_n(f,g)\bigr|
\le
\|f\|_\infty\|g\|_\infty
\left(
\sqrt{2I(\hat Y;Z)}
+
\sqrt{\frac{2\log(6/\delta)}{n}}
\right).
\]
\end{exam}

\subsection{Counter-example for correlation}\label{a:counter_example_correlation}

\begin{exam}[Why not correlation?]
     Let $Y$ be a Bernoulli variable with $P(Y=0)=\epsilon$ (small) and $P(Y=1)=1-\epsilon$.  Set $Z=Y$ (perfectly correlated).  Then $$I(Y;Z)=H(Y)=-\epsilon\ln\epsilon -(1-\epsilon)\ln(1-\epsilon),$$
     which becomes approximately $\epsilon|\ln\epsilon|$ when $\epsilon\to0$. Now choose centered, bounded functions $$f(Y)=\frac{\mathbf{1}_{{Y=0}} - \epsilon}{\sqrt{\epsilon(1-\epsilon)}},\qquad g(Z)=\frac{\mathbf{1}_{{Z=0}} - \epsilon}{\sqrt{\epsilon(1-\epsilon)}}.$$ 
     By construction $E[f(Y)]=E[g(Z)]=0$, and one checks $$\Cov(f(Y),g(Z))=\frac{E[(\mathbf{1}_{Y=0}-\epsilon)^2]}{\epsilon(1-\epsilon)}=\frac{\epsilon(1-\epsilon)}{\epsilon(1-\epsilon)}=1.$$ In particular, with $L^2$ normalization we get $\frac{|\Cov(f,g)|}{||f||_2||g||_2}=1$.  But $\sqrt{2I(Y;Z)}\approx\sqrt{2\epsilon|\ln\epsilon|}\to0$ as $\epsilon\to0$.  Thus the $L^2$‐normalized covariance can be much larger than $\sqrt{2I}$, violating any bound of the form $|\Cov(f,g)|/(||f||_2||g||_2)\le \sqrt{2I(Y;Z)}$.
\end{exam}

\subsection{Proof of Theorem \ref{th:CMI_bound}}\label{a:proof_CMI_bound}

\begin{proof}
For \(P_Y\)-almost every \(y\), define the conditional mutual information density
\[
    I_y(U;Z)
    :=
    D_{\mathrm{KL}}\!\left(
    P_{U,Z\mid Y=y}
    \,\middle\|\,
    P_{U\mid Y=y}\otimes P_{Z\mid Y=y}
    \right).
\]
By the conditional centering assumptions,
\[
    \mathbb E[h(U)\mid Y=y]=0,
    \qquad
    \mathbb E[g(Z)\mid Y=y]=0
\]
for \(P_Y\)-almost every \(y\). Hence, applying Lemma~\ref{l:MI_bound}
under the conditional law \(P_{U,Z\mid Y=y}\), we obtain
\[
\frac{
\left|
\mathbb E\!\left[\langle h(U),g(Z)\rangle\mid Y=y\right]
\right|
}{
\|h\|_{L^\infty}\|g\|_{L^\infty}
}
\le
\sqrt{2I_y(U;Z)}
\]
for \(P_Y\)-almost every \(y\). Taking expectation over \(Y\) gives
\[
\frac{
\mathbb E_Y\!\left[
\left|
\mathbb E\!\left[\langle h(U),g(Z)\rangle\mid Y\right]
\right|
\right]
}{
\|h\|_{L^\infty}\|g\|_{L^\infty}
}
\le
\mathbb E_Y\!\left[
\sqrt{2I_Y(U;Z)}
\right].
\]
By Jensen's inequality,
\[
\mathbb E_Y\!\left[
\sqrt{2I_Y(U;Z)}
\right]
\le
\sqrt{
2\mathbb E_Y[I_Y(U;Z)]
}
=
\sqrt{2I(U;Z\mid Y)}.
\]
Taking the supremum over all bounded measurable \(h,g\) satisfying the
conditional centering assumptions completes the proof.
\end{proof}

\subsection{Proof of Lemma \ref{lem:budget}}\label{a:proof_budget}

\begin{proof}
Fix a generic predictor \(U\), it follows from the chain rule for mutual information that
\[
    I(U;(Y,Z))
    =
    I(U;Y)+I(U;Z\mid Y)
    =
    u+v.
\]

We next prove the data-processing upper bound. If \(U=f(X,Z)\) is deterministic, then
\[
    U \rightarrow (X,Z) \rightarrow (Y,Z)
\]
is a Markov chain, and hence
\[
    I(U;(Y,Z))
    \le
    I((X,Z);(Y,Z)).
\]
If \(U=(N,f(X,Z,N))\) is a revealed randomized predictor with
\(N\perp (X,Y,Z)\), then
\[
    U \rightarrow (X,Z,N) \rightarrow (Y,Z)
\]
is a Markov chain. Therefore,
\[
    I(U;(Y,Z))
    \le
    I((X,Z,N);(Y,Z)).
\]
Moreover, since \(N\perp (X,Y,Z)\),
\[
    I((X,Z,N);(Y,Z))
    =
    I((X,Z);(Y,Z)) + I(N;(Y,Z)\mid X,Z)
    =
    I((X,Z);(Y,Z)).
\]
Thus, in both cases, we have
\[
    u+v
    =
    I(U;(Y,Z))
    \le
    I((X,Z);(Y,Z)).
\]

Finally, applying the chain rule again gives
\[
    I((X,Z);(Y,Z))
    =
    I((X,Z);Y)+I((X,Z);Z\mid Y).
\]
Similarly, by data processing,
\[
    u=I(U;Y)\le I((X,Z);Y),
\]
and
\[
    v=I(U;Z\mid Y)\le I((X,Z);Z\mid Y),
\]
where, in the revealed randomized case, the same argument is applied with
\((X,Z,N)\) first and then uses \(I(N;Z\mid X,Z,Y)=0\).

Now, if \(Z\) is discrete, then we further have
\[
    I((X,Z);Z\mid Y)
    =
    H(Z\mid Y)-H(Z\mid X,Z,Y)
    =
    H(Z\mid Y),
\]
because \(H(Z\mid X,Z,Y)=0\). If \(Y\) is discrete and finite, then we also have
\[
    I((X,Z);Y)\le H(Y)\le \log|\mathcal Y|.
\]
That completes the proof.
\end{proof}

\subsection{Proof of Lemma \ref{lem:cond-law-matching}}\label{append:cond-law-matching}

\begin{proof}
Since the spaces are Polish, they are standard Borel spaces. Hence regular
conditional distributions exist.

For a fixed \(z\in\mathcal Z\), define the frozen predictor
\[
    U_z:=
    \begin{cases}
    f(X,z), & \text{if } U=f(X,Z),\\
    (N,f(X,z,N)), & \text{if } U=(N,f(X,Z,N)).
    \end{cases}
\]
Because \(N\perp (X,Y,Z)\), the conditional independence \(X\perp Z\mid Y\)
also implies \((X,N)\perp Z\mid Y\). Therefore, for each fixed \(z\), the
frozen predictor \(U_z\), being a measurable function of \((X,N)\), is
conditionally independent of \(Z\) given \(Y\). Hence, for every Borel set
\(A\subseteq\mathcal U\),
\[
    \mathbb P(U_z\in A\mid Y=y,Z=z)
    =
    \mathbb P(U_z\in A\mid Y=y)
\]
for \(P_{Y,Z}\)-almost every \((y,z)\). On the other hand, since \(I(U;Z\mid Y)=0\), we have \(U\perp Z\mid Y\).
Therefore,
\[
    \mathbb P(U\in A\mid Y=y,Z=z)
    =
    \mathbb P(U\in A\mid Y=y)
\]
for \(P_{Y,Z}\)-almost every \((y,z)\). Finally, by the definition of \(U\), for each fixed \(z\in\mathcal Z\), on the fiber \(\{Z=z\}\) the predictor \(U\) is given by the same measurable function of \((X,N)\) as the frozen predictor \(U_z\). Equivalently, for \(P_{Y,Z}\)-almost every \((y,z)\),
\[
    \mathcal L(U\mid Y=y,Z=z)
    =
    \mathcal L(U_z\mid Y=y,Z=z).
\]
Hence, for every Borel set \(A\subseteq\mathcal U\),
\[
    \mathbb P(U\in A\mid Y=y,Z=z)
    =
    \mathbb P(U_z\in A\mid Y=y,Z=z)
\]
for \(P_{Y,Z}\)-almost every \((y,z)\). Combining the three equations above gives
\[
    \mathbb P(U_z\in A\mid Y=y)
    =
    \mathbb P(U\in A\mid Y=y)
\]
for \(P_{Y,Z}\)-almost every \((y,z)\).

It remains to pass from a \(P_{Y,Z}\)-almost everywhere statement to a statement
holding for \(P_Z\)-almost every \(z\) and \(P_Y\)-almost every \(y\). Let
\(\mathcal A\) be a countable \(\pi\)-system generating
\(\mathcal B(\mathcal U)\). For each \(A\in\mathcal A\), define
\[
    N_A
    :=
    \left\{
    (y,z):
    \mathbb P(U_z\in A\mid Y=y)
    \neq
    \mathbb P(U\in A\mid Y=y)
    \right\}.
\]
From the previous paragraph, \(P_{Y,Z}(N_A)=0\). Since
\(P_Y\otimes P_Z\ll P_{Y,Z}\), it follows that
\[
    (P_Y\otimes P_Z)(N_A)=0.
\]
Let $N:=\bigcup_{A\in\mathcal A} N_A$. Since \(\mathcal A\) is countable, it follows from countable additivity that
\[
    (P_Y\otimes P_Z)(N)=0.
\]
Then it follows from Fubini's theorem that the set $Z_N:=\left\{z\in\mathcal Z:P_Y(\{y:(y,z)\in N\})>0\right\}$ satisfies \(P_Z(Z_N)=0\). Define
\[
    \mathcal Z_0:=\mathcal Z\setminus Z_N.
\]
Then \(P_Z(\mathcal Z_0)=1\), and for every \(z_0\in\mathcal Z_0\), the equality
\[
    \mathbb P(U_{z_0}\in A\mid Y=y)
    =
    \mathbb P(U\in A\mid Y=y)
\]
holds for every \(A\in\mathcal A\) and for \(P_Y\)-almost every \(y\). 
It now follows from the \(\pi\)-\(\lambda\) theorem that the two regular conditional laws agree on all Borel sets:
\[
    \mathcal L(U_{z_0}\mid Y=y)
    =
    \mathcal L(U\mid Y=y)
\]
for \(P_Y\)-almost every \(y\). Finally, integrating with respect to \(P_Y(dy)\), we have
\[
    \mathcal L(U_{z_0},Y)=\mathcal L(U,Y).
\]
Since mutual information depends only on the joint law, we also have
\[
    I(U_{z_0};Y)=I(U;Y).
\]
That completes the proof.
\end{proof}

\subsection{Proof of Theorem \ref{thm:necessary-tradeoff}}
\label{append:necessary-tradeoff}

\begin{proof}
We first prove the equality
\[
    \sup_U
    \Big\{
        I(U;Y): I(U;Z\mid Y)=0
    \Big\}
    =
    u_X^\star.
\]

To start, let \(U_X\) be any \(X\)-only predictor:
\[
    U_X\in \{g(X),(N,g(X,N))\},
    \qquad
    N\perp (X,Y,Z).
\]
Since \(X\perp Z\mid Y\), and since \(N\) is independent of \((X,Y,Z)\), we have $U_X\perp Z\mid Y$. Hence,
\[
    I(U_X;Z\mid Y)=0.
\]
Therefore, every \(X\)-only predictor is feasible for the perfect-separation
constraint. That implies
\[
    \sup_U
    \Big\{
        I(U;Y): I(U;Z\mid Y)=0
    \Big\}
    \ge
    \sup_{U_X} I(U_X;Y)
    =
    u_X^\star.
\]

For the other direction, let \(U\) be any generic predictor satisfying
\[
    I(U;Z\mid Y)=0.
\]
By Lemma~\ref{lem:cond-law-matching}, there exists \(z_0\in\mathcal Z_0\) such
that the frozen predictor \(U_{z_0}\) satisfies $\mathcal L(U_{z_0},Y)=\mathcal L(U,Y)$. Therefore,
\[
    I(U_{z_0};Y)=I(U;Y).
\]
Since \(U_{z_0}\) is an \(X\)-only predictor, possibly using the same independent
external noise \(N\), we have
\[
    I(U;Y)
    =
    I(U_{z_0};Y)
    \le
    u_X^\star.
\]
Taking the supremum over all perfectly separated generic predictors \(U\)
gives
\[
    \sup_U
    \Big\{
        I(U;Y): I(U;Z\mid Y)=0
    \Big\}
    \le
    u_X^\star.
\]
That completes the proof for the equality.

Now, we are ready to prove the rest of the claims. If a generic predictor \(U\) satisfies $I(U;Y)>u_X^\star$, then it follows from the equality above and the definition of supremum that $I(U;Z\mid Y)>0$. Indeed, if $I(U;Z\mid Y)=0$, then we have $u_X^\star \geq I(U;Y)>u_X^\star$, a contradiction. Therefore, we have
\[
    I(U;Z\mid Y)>0.
\]

Finally, suppose \(u_{XZ}^\star>u_X^\star\). By the definition of the supremum,
there exists a generic predictor \(U\) such that
\[
    I(U;Y)>u_X^\star.
\]
Hence, it must satisfy
\[
    I(U;Z\mid Y)>0.
\]
That completes the proof.
\end{proof}

\section{Appendix of Section \ref{sec:algorithm}}\label{Append:algorithm}

\subsection{Proof of Proposition \ref{prop:estimation_error}}\label{append:estimation_error}

\begin{proof}
Write the plug-in estimator as a stratified average
\[
\widehat I_B
=\sum_{y\in\mathcal Y}\frac{B_y}{|B|}\,\widehat I_y,
\qquad
\widehat I_y:=\widehat I(U;Z\mid Y=y),
\]
where $B_y:=|\{i:Y_i=y\}|$ is the number of samples with label $y$, and $\widehat I_y$ is the ordinary empirical mutual information computed from those $B_y$ samples.

\textbf{Part (1): Asymptotic bias.}
Fix $y\in\mathcal Y$ and condition on $\{B_y=n\}$ with $n\ge 1$. On this event, $\widehat I_y$ is the standard empirical MI estimator from $n$ i.i.d.\ samples drawn from the conditional distribution $P_{UZ\mid Y=y}$.
Under the full-support assumption ($P(u,z|y) \geq q_{\text{min}} > 0$), the Miller-Madow expansion for the empirical MI bias (see \citet{paninski2003estimation}) yields:
\begin{equation}\label{eq:mm_cond}
\mathbb E[\widehat I_y\mid B_y=n]
=
I(U;Z\mid Y=y)
+
\frac{(K_U-1)(K_Z-1)}{2n}
+
O(n^{-2}).
\end{equation}
Multiply \eqref{eq:mm_cond} by $B_y/|B|$ and take expectations over $B_y$:
\begin{align*}
\mathbb E\!\left[\frac{B_y}{|B|}\widehat I_y\right]
&=
\mathbb E\!\left[\frac{B_y}{|B|}\right] I(U;Z\mid Y=y)
+\frac{(K_U-1)(K_Z-1)}{2|B|}\,\mathbb E\!\left[\mathbbm{1}_{\{B_y\ge 1\}}\right]
+R_y,
\end{align*}
where the remainder satisfies
\[
|R_y|
\;\le\;
\mathbb E\!\left[\frac{B_y}{|B|}\,O(B_y^{-2})\,\mathbbm{1}_{\{B_y\ge 1\}}\right]
=
O\!\left(\frac{1}{|B|}\,\mathbb E\!\left[\frac{1}{B_y}\mathbbm{1}_{\{B_y\ge 1\}}\right]\right).
\]
Since $B_y\sim\mathrm{Binomial}(|B|,p_y)$ with $p_y:=P(Y=y)\ge p_{\min}$, a Chernoff bound gives $\mathbb{P}(B_y\le \tfrac12 |B|p_y)\le e^{-c|B|}$ for some $c=c(p_{\min})>0$. The inverse moment is dominated by the typical set:
\[
\mathbb E\!\left[\frac{1}{B_y}\mathbbm{1}_{\{B_y\ge 1\}}\right]
\le
\frac{2}{|B|p_y}+\mathbb{P}\!\left(B_y\le \tfrac12 |B|p_y\right)
=
O(|B|^{-1}).
\]
Therefore, $R_y=O(|B|^{-2})$. Also, $\mathbb E[B_y/|B|]=p_y$ and $\mathbb E[\mathbbm{1}_{\{B_y\ge 1\}}]=1-(1-p_y)^{|B|}=1-O(e^{-c|B|})$. Summing over $y\in\mathcal Y$ gives
\[
\mathbb E[\widehat I_B]
=
\sum_y p_y I(U;Z\mid Y=y)
+
\frac{K_Y (K_U-1)(K_Z-1)}{2|B|}
+
O(|B|^{-2}),
\]
and $\sum_y p_y I(U;Z\mid Y=y)=I(U;Z\mid Y)$. That proves part (1).

\textbf{Part (2): Concentration.}
We use a high-probability ``good'' event to avoid empty/near-empty strata (conditioned on $\{Y = y\}$) which cause instability in the entropy estimates. Define
\[
\mathcal G
:=
\Big\{\forall y:\ B_y\ge \tfrac12 |B| p_{\min}\Big\}
\cap
\Big\{\forall (u,z,y):\ N_{u,z,y}\ge \tfrac12 |B|\,p_{\min}q_{\min}\Big\},
\]
where $N_{u,z,y}$ is the count of samples equal to $(u,z,y)$.
By Chernoff bounds and a union bound over all strata and outcomes, there exists a constant \(c_0=c_0(p_{\min},q_{\min})>0\) such that:
\begin{equation}\label{eq:G_fail}
\epsilon_B := \mathbb{P}(\mathcal G^c)
\le
(K_Y + K_YK_UK_Z)\exp\!\big(-c_0 |B| p_{\min}q_{\min}\big).
\end{equation}
For sufficiently large $|B|$, $\epsilon_B \le \delta/2$.

Next, we establish the bounded difference property on $\mathcal G$. Let $F(D):=\widehat I_B$ viewed as a function of the dataset $D$. Consider two datasets $D,D'$ differing by exactly one sample. This change affects at most two strata (if the $Y$-label changes). On \(\mathcal G\), for every \((u,z,y)\),
\[
\widehat P(u,z\mid y)=\frac{N_{u,z,y}}{B_y}
\ge
\frac{\tfrac12 |B|p_{\min}q_{\min}}{|B|}
= \frac12 p_{\min}q_{\min}.
\]
In particular, all empirical conditional joint probabilities are uniformly bounded away from zero on \(\mathcal G\). The function \(h(t)=-t\log t\) has derivative $|h'(t)|=|1+\log t|$, which is uniformly bounded on \([\eta,1]\) for every \(\eta>0\). On \(\mathcal G\), the empirical conditional cell probabilities are bounded below by $\eta:=\tfrac12 p_{\min}q_{\min}$, so
\[
|h'(t)|\le 1+\big|\log \eta\big|=:L_0
\qquad\text{for all } t\in[\eta,1].
\] Thus, the entropy functional is $L_0$-Lipschitz with respect to the $\ell_1$ norm.
Changing one sample in stratum $y$ changes the empirical distribution by at most $2/B_y$. Consequently, the change in the local MI estimate $\widehat I_y$ is bounded:
\[
|\Delta \widehat I_y| \le 3L_0 \|\Delta \widehat P_{UZ|y}\|_1 \le \frac{6L_0}{B_y}.
\]
The total change in the weighted sum $\widehat I_B = \sum \frac{B_y}{|B|} \widehat I_y$ involves changes to the weights $B_y/|B|$ and the values $\widehat I_y$. Since \(B_y \ge \frac12 |B|p_{\min}\) on \(\mathcal G\), we have
\[
|\Delta \widehat I_y|
\le
\frac{12L_0}{|B|p_{\min}}.
\]
Moreover, the weights \(B_y/|B|\) can change by at most \(1/|B|\) for at most two strata when one sample is modified, and each \(\widehat I_y\) is bounded by
\[
M_{\max}:=\log(K_UK_Z).
\]
Hence, there exists a constant
\[
c=c(K_U,K_Z,p_{\min},q_{\min})
\]
such that
\[
|F(D)-F(D')|
\le
\frac{c}{|B|}
\qquad
\forall D,D'\in\mathcal G \text{ with } d_H(D,D')=1.
\]
Here, $d_H$ is the Hamming distance. To apply concentration bounds rigorously despite the conditioning, we consider a Lipschitz extension. Since $F$ satisfies the bounded difference property with constant $c/|B|$ on $\mathcal{G}$, there exists an extension \(F^\star\) defined on the entire domain
\((\mathcal U\times\mathcal Y\times\mathcal Z)^{|B|}\) that coincides with $F$ on $\mathcal{G}$ and preserves the bounded differences property globally. Explicitly, we use the McShane-Whitney extension with Hamming distance $d_H$:
\[
F^\star(D) := \inf_{D' \in \mathcal{G}} \left( F(D') + \frac{c}{|B|} d_H(D, D') \right).
\]
By construction, $F^\star(D)=F(D)$ whenever $D \in \mathcal{G}$, and $F^\star$ satisfies the bounded difference condition with constant $c/|B|$ everywhere. Applying McDiarmid's inequality to $F^\star$:
\[
\mathbb{P}\!\left(\big|F^\star(D) - \mathbb{E}[F^\star]\big| \ge t\right)
\le
2\exp\!\left(-\frac{2|B|t^2}{c^2}\right).
\]
Now, since $\{F \neq F^\star\} \subseteq \mathcal G^c$, and the global range is bounded by $M_{\max}$, we have $$|\mathbb{E}[F^\star] - \mathbb{E}[F]| \le M_{\max} \mathbb{P}(\mathcal G^c) = M_{\max} \epsilon_B.$$
Thus, for any $t > 0$:
\[
\mathbb{P}\big(|F - \mathbb{E}[F]| \ge t + M_{\max}\epsilon_B\big)
\le
2\exp\!\left(-\frac{2|B|t^2}{c^2}\right) + \epsilon_B.
\]
Setting $t = c\sqrt{\frac{\log(4/\delta)}{2|B|}}$, the first term becomes $\le \delta/2$. For $|B|$ large, $\epsilon_B \le \delta/2$ and the exponentially decaying bias term $M_{\max}\epsilon_B$ is negligible compared to the concentration term (specifically, $M_{\max}\epsilon_B \le \frac{c}{\sqrt{2}}\sqrt{\frac{\log(2/\delta)}{|B|}}$).
Summing the terms and simplifying, we obtain a constant
\[
C=C(K_U,K_Y,K_Z,p_{\min},q_{\min})
\]
such that
\[
\big|\widehat I_B-\mathbb E[\widehat I_B]\big|
\le
C\sqrt{\frac{\log(2/\delta)}{|B|}}
\quad \text{with probability at least }1-\delta.
\]
This completes the proof of Part~(2).
\end{proof}

\newpage

\section{Appendix to Section \ref{sec:experiments}}

\subsection{Datasets Information}\label{append:data_info}

See Table \ref{tab:datasets}.

\begin{table}[t]
\centering
\footnotesize
\setlength{\tabcolsep}{4pt}
\renewcommand{\arraystretch}{1.15}
\newcolumntype{Y}{>{\centering\arraybackslash}X} 

\begin{tabularx}{\linewidth}{@{}l r Y Y l@{}}
\toprule
\textbf{Dataset} & \textbf{\# Samples} & \textbf{Input Features ($X$)} & \textbf{Target ($Y$)} & \textbf{Sensitive ($Z$)} \\
\midrule
Adult  & 32{,}561  & Tabular demographic and employment features & Income $\ge$ 50K & Gender (M/F) \\

COMPAS & 5{,}278   & Tabular criminal-history and demographic features & Recidivism within 2 years & Race (Black/White) \\

Bank   & 45{,}211  & Tabular client and campaign features & Term-deposit subscription & Marital status, binarized \\

CelebA & 202{,}599 & Face images; fixed ResNet18 embeddings used in experiments & Smiling & Gender (M/F) \\

ACSOccupation & 2{,}576 & ACS demographic and socioeconomic covariates; complete-case preprocessing & Occupation, top 20 classes & Race, 9 groups \\

ACSOccupation-Large & 575{,}989 & ACS demographic and socioeconomic covariates; missing values retained as categorical levels & Occupation, top 20 classes & Race, 9 groups \\
\bottomrule
\end{tabularx}
\caption{Summary of datasets used in experiments.}
\label{tab:datasets}
\end{table}

\subsection{Metrics Information}\label{append:metrics}

See Table~\ref{tab:metrics}

\begin{table*}[t]
\centering
\footnotesize
\begin{tabularx}{\textwidth}{@{}l l X@{}}
\toprule
\textbf{Metric} & \textbf{Category} & \textbf{Definition} \\
\midrule
Accuracy  & Utility
& $\frac{1}{N}\sum_{i=1}^N \mathbf{1}\{\hat y_i = y_i\}$ \\

Expected randomized accuracy & Utility
& $\frac{1}{N}\sum_{i=1}^N Q_\theta(y_i\mid x_i,z_i)$, used for randomized-policy evaluation. \\

AUROC & Utility
& Area under the receiver operating characteristic curve, computed from the positive-class score for binary tasks. \\

Macro-AUROC & Utility
& For multi-class tasks,
$\displaystyle
\mathrm{AUROC}_{\mathrm{macro}}
=
\frac{1}{|\mathcal Y|}
\sum_{y\in\mathcal Y}
\mathrm{AUROC}\!\left(
\mathbf 1\{Y=y\},
Q_\theta(y\mid X,Z)
\right).
$ \\

MI ($u$)  & Utility
& $\displaystyle
\sum_{\hat y,y}\hat p(\hat y,y)
\log\!\frac{\hat p(\hat y,y)}{\hat p(\hat y)\hat p(y)}
$ \quad (plug-in estimator). \\
\midrule
EO-gap & Violation
& $\displaystyle
\tfrac12 \left[
(\max_z \mathrm{FPR}_z - \min_z \mathrm{FPR}_z)
+
(\max_z \mathrm{FNR}_z - \min_z \mathrm{FNR}_z)
\right],
$
reported for binary tasks. \\

CMI ($v$) & Violation
& $\displaystyle
\sum_y \hat p(y)
\sum_{\hat y,z}
\hat p(\hat y,z\mid y)
\log
\frac{
\hat p(\hat y,z\mid y)
}{
\hat p(\hat y\mid y)\hat p(z\mid y)
}
$
\quad (plug-in estimator of $I(\hat Y;Z\mid Y)$). \\
\bottomrule
\end{tabularx}
\caption{
Evaluation metrics. Empirical probabilities $\hat p$ are computed on test folds in cross-validation. Binary AUROC and EO-gap are reported for binary benchmarks; macro-AUROC and expected randomized accuracy are reported for ACSOccupation.
}
\label{tab:metrics}
\end{table*}

\subsection{Comparison Methods \& Implementations}
\label{appen:methods}

\begin{table*}[t]
\centering
\footnotesize
\setlength{\tabcolsep}{4pt}
\renewcommand{\arraystretch}{1.2}
\begin{tabularx}{\textwidth}{@{}l c X@{}}
\toprule
\textbf{Method} & \textbf{Type} & \textbf{Implementation Mechanism} \\
\midrule
ERM & -- &
Unconstrained cross-entropy training with the shared MLP backbone. We use the gradient-normalized cross entropy for consistency. \\

Threshold Optimizer~\citep{hardt2016equality} & Post &
Fairlearn~\cite{bird2020fairlearn} post-processing method that learns a group-dependent randomized decision rule from a trained score function to satisfy equalized-odds constraints. \\

EG Reductions~\citep{agarwal2018reductions} & In &
Fairlearn exponentiated-gradient reduction with equalized-odds constraints, implemented using a sample-weight-aware neural backbone adapter. \\

Adversarial~\cite{madras2018learning} & In &
Min-max training in which a discriminator predicts \(Z\) from the learned representation, while the predictor maximizes the discriminator loss. \\

FR-Train~\citep{roh2020frtrain} & In &
$Y$-conditional adversarial baseline: the discriminator predicts \(Z\) from \((\widehat p_\theta,Y)\), where \(\widehat p_\theta\) is the predictive posterior. \\

FairDRO~\citep{fairdro2024} & In &
Group-loss reweighting over \((Y,Z)\)-groups, interpolating empirical group weights with softmax worst-group weights. \\

InfoFair~\citep{kang2022infofair} & In &
Conditional contrastive MI proxy: an InfoNCE-style~\cite{oord2018cpc} term is computed within \(Y\)-condition/slice on a learned latent representation. \\

Fair Dummy~\citep{romano2021fair} & In &
Distribution-matching baseline using a dummy sensitive attribute \(\widetilde Z\sim P(Z\mid Y)\) and a discriminator distinguishing real from dummy sensitive inputs. \\

\textbf{CMI-NN (Ours)} & In &
\textbf{Direct differentiable plug-in estimate of \(I(\hat{Y};Z\mid Y)\), trained with gradient-norm balancing.} \\
\bottomrule
\end{tabularx}
\caption{
Comparison methods. ``In'' denotes in-processing methods and ``Post'' denotes post-processing methods. All neural methods use the same MLP backbone, optimizer family, training horizon, and cross-validation splits whenever applicable.
}
\label{tab:baselines}
\end{table*}

Our baselines are chosen to cover the main existing methods that enforce separation from different approaches: constrained optimization, information-theoretic penalties, distributional robustness, and distribution matching. This choice has two goals: (i) to make our results indirectly comparable to prior and future methods within each line of work, and (ii) to make the differences among these algorithmic approaches visible in the experiments.

For implementation, we use authors' public code when directly applicable. Otherwise, we follow the corresponding paper pseudocode as closely as possible and implement controlled adaptations under the same model backbone, data splits, and preprocessing.

\textbf{Why these methods.}
\emph{EG Reductions}~\cite{agarwal2018reductions} represents the rate-constrained, reductions-based in-processing family with provable feasibility in expectation.
\emph{Adversarial Debiasing}~\cite{madras2018learning} represents the standard minimax approach that removes sensitive information through an adversary predicting \(Z\).
\emph{FR-Train}~\cite{roh2020frtrain} and \emph{InfoFair}~\cite{kang2022infofair} represent variational- and learning-based information-theoretic approaches: FR-Train uses a \(Y\)-conditional adversarial proxy for separation, while our InfoFair implementation uses an InfoNCE-based variational proxy with \(Y\)-condition to target conditional sensitive information.
\emph{FairDRO}~\cite{fairdro2024} captures distributionally robust optimization through worst-group or high-loss group reweighting.
\emph{Fair Dummy}~\cite{romano2021fair} exemplifies distribution matching by constructing a dummy sensitive attribute and aligning the real and dummy class-conditional distributions. This is the Monte Carlo approach to separation.
Our method \emph{CMI} directly penalizes an empirical conditional mutual information, providing a single differentiable objective aligned with separation.

\paragraph{Shared neural backbone and input convention.}
All neural in-processing methods use the same two-hidden-layer MLP backbone. The network returns both logits and the penultimate representation \(h\), allowing methods to apply penalties or adversaries to logits, predictive posteriors, latent representations, or \(h\). Unless otherwise stated, methods are trained with the same Adam optimizer, number of epochs, and minibatch protocol. When a method is trained with sensitive input, we concatenate \(Z\) to the feature vector and train on \([Z|X]\); otherwise, the method is trained on \(X\) alone. The corresponding evaluation wrapper preserves the same input convention at test time.

\paragraph{ERM.}
The ERM baseline minimizes cross-entropy with no fairness penalty. In the experiments, we use a gradient-normalized version of cross-entropy consistently across the comparison methods when applicable,
\[
    \mathcal L_{\mathrm{ERM}}
    =
    \frac{\mathcal L_{\mathrm{ce}}}
    {\|\partial_h \mathcal L_{\mathrm{ce}}\|+\epsilon},
\]
where the gradient norm is detached and acts only as a scalar normalization factor. This makes ERM directly comparable to the gradient-normalized CMI and InfoFair objectives. ERM serves as the unconstrained reference for utility and separation violation.

\paragraph{Threshold Optimizer.}
Threshold Optimizer represents the post-processing approach of Hardt et al.~\citep{hardt2016equality}. We implement it using Fairlearn's \texttt{ThresholdOptimizer} with equalized-odds constraints \cite{bird2020fairlearn}. A base score model \(s(X)\) is first trained, and the post-processor is then fit on a held-out validation set. The learned rule defines a group-dependent randomized policy
\[
    \hat{Y} \sim q_z(\cdot\mid s(X)),
\]
chosen to satisfy empirical equalized-odds constraints while minimizing classification cost. At deployment time, the rule depends only on \(s(X)\) and \(Z\), not on the unknown true label \(Y\). In evaluation, we report policy-level quantities from the randomized post-processing distribution and also consider the deterministic argmax of that policy. For ranking metrics such as AUROC, we use the base model posterior score.

\paragraph{Exponentiated Gradient Reductions.}
EG Reductions~\citep{agarwal2018reductions} is implemented using Fairlearn's \texttt{ExponentiatedGradient} with the \texttt{EqualizedOdds} constraint. To make the comparison compatible with the neural methods, we use a sample-weight-aware adapter around the shared PyTorch backbone as the base estimator. The adapter trains the backbone with weighted cross-entropy losses, as required by the reductions framework. The resulting EG model represents a randomized policy over classifiers. We evaluate both randomized policy probabilities and deterministic argmax decisions when appropriate.

\paragraph{Adversarial Debiasing.}
The adversarial baseline \cite{madras2018learning} trains a predictor and a sensitive-attribute discriminator in alternating steps. The discriminator receives the learned representation \(h\) and predicts \(Z\). The predictor is trained to minimize task loss while making \(Z\) difficult to recover from \(h\):
\[
    \mathcal L_{\mathrm{pred}}
    =
    \mathcal L_{\mathrm{ce}}
    -
    \lambda \mathcal L_{\mathrm{adv}}.
\]
Each minibatch performs \(k_D\) discriminator updates followed by one predictor update. In the multi-group case, the discriminator uses a multi-class cross-entropy loss; in the binary case, this reduces to the corresponding binary discriminator loss. This baseline represents representation-level adversarial removal of sensitive information.

\paragraph{FR-Train.}
Our FR-Train baseline \cite{roh2020frtrain} follows the $Y$-conditional adversarial idea used to target equalized-odds-type dependence. Instead of predicting \(Z\) from \(h\) alone, the discriminator receives the model posterior and the true label:
\[
    \bigl(\widehat p_\theta(\cdot\mid X,Z),\ \mathrm{onehot}(Y)\bigr).
\]
It is trained to predict \(Z\), while the predictor maximizes this discriminator loss:
\[
    \mathcal L_{\mathrm{pred}}
    =
    \mathcal L_{\mathrm{ce}}
    -
    \lambda \mathcal L_{\mathrm{disc}}.
\]
Thus, this baseline acts as a learned class-conditional proxy for reducing dependence between \(\hat{Y}\) and \(Z\) given \(Y\). It is included as a FR-Train implementation that is suitable for our comparison purpose.

\paragraph{FairDRO.}
FairDRO \cite{fairdro2024} is implemented as group-loss reweighting over \((Y,Z)\)-groups. We define a group identifier
\[
    g = Y\cdot |\mathcal Z| + Z,
\]
so that groups correspond to the label-sensitive cells relevant for equalized-odds-type constraints. For each minibatch, we compute the average loss \(L_g\) over the observed groups. The final group weight interpolates between the empirical group weight \(w_{\mathrm{erm}}\) and a soft worst-group weight
\[
    w_{\mathrm{dro}}(g)
    =
    \frac{\exp(\eta L_g)}
    {\sum_{g'}\exp(\eta L_{g'})}.
\]
The resulting objective is
\[
    \mathcal L_{\mathrm{FairDRO}}
    =
    \sum_g
    \Bigl((1-\lambda)w_{\mathrm{erm}}(g)
    +
    \lambda w_{\mathrm{dro}}(g)\Bigr)
    L_g.
\]
This baseline compares direct CMI regularization with a robust optimization strategy that emphasizes high-loss \((Y,Z)\)-groups.

\paragraph{InfoFair.}
InfoFair \cite{kang2022infofair} formulates group-fair learning through information-theoretic regularization, primarily for statistical-parity-type constraints. Since our target notion is separation, we adapt this idea by applying a conditional contrastive proxy within each \(Y\)-stratum. Concretely, the network maps logits to a normalized latent representation, and an InfoNCE-style \cite{oord2018cpc} objective is computed among samples with the same label \(Y\), using sensitive labels \(Z\) to define the contrastive structure. This gives a minibatch-based variational proxy for conditional sensitive information in the learned representation. We train it with the same gradient-norm balancing used by CMI:
\[
    \mathcal L_{\mathrm{InfoFair}}
    =
    (1-\lambda)
    \frac{\mathcal L_{\mathrm{ce}}}
    {\|\partial_h\mathcal L_{\mathrm{ce}}\|+\epsilon}
    +
    \lambda
    \frac{\mathcal L_{\mathrm{InfoFair}}}
    {\|\partial_h\mathcal L_{\mathrm{InfoFair}}\|+\epsilon}.
\]
If a minibatch contains no valid contrastive signal within a label stratum, the contrastive term is set to zero for that batch. This baseline represents a variational information-proxy approach to separation.

\paragraph{Fair Dummy.}
Fair Dummy~\citep{romano2021fair} is implemented as a distribution-matching adversarial baseline. For each training sample, we generate a dummy sensitive attribute
\[
    \widetilde Z \sim \widehat P(Z\mid Y),
\]
estimated from the training data. A discriminator is trained to distinguish tuples involving the real sensitive attribute \(Z\) from tuples involving the dummy attribute \(\widetilde Z\). In our implementation, the discriminator input includes the label \(Y\), either the predictive signal or learned representation, and the real or dummy sensitive attribute. The predictor is then trained to minimize cross-entropy while making the real and dummy distributions difficult to distinguish:
\[
    \mathcal L_{\mathrm{pred}}
    =
    \mathcal L_{\mathrm{ce}}
    -
    \lambda\mathcal L_{\mathrm{dummy}}.
\]
This provides a comparison between direct CMI regularization and distribution-matching approaches to class-conditional group fairness.

\paragraph{CMI-NN (Ours).}
Our method directly computes a differentiable plug-in estimate of conditional mutual information from the minibatch predictive posterior. For each minibatch, we form soft empirical counts for \((\hat{Y},Y,Z)\) using the model posterior \(\widehat p_\theta(\cdot\mid X,Z)\), and compute
\[
    \widehat I_{\mathrm{CMI}}
    \approx
    I(\hat{Y};Z\mid Y).
\]
The main training objective is
\[
    \mathcal L_{\mathrm{CMI}}
    =
    (1-\lambda)
    \frac{\mathcal L_{\mathrm{ce}}}
    {\|\partial_h\mathcal L_{\mathrm{ce}}\|+\epsilon}
    +
    \lambda
    \frac{\widehat I_{\mathrm{CMI}}}
    {\|\partial_h\widehat I_{\mathrm{CMI}}\|+\epsilon}.
\]
The gradient norms are detached and act as adaptive scale factors. Unlike adversarial, contrastive, or distributionally robust baselines, the CMI penalty is computed directly from the empirical soft-count distribution associated with the separation violation itself.

\paragraph{CMI without gradient normalization.}
We also include a raw-CMI ablation in which the training objective is
\[
    \mathcal L_{\mathrm{raw}}
    =
    (1-\lambda)\mathcal L_{\mathrm{ce}}
    +
    \lambda\,c\,\widehat I_{\mathrm{CMI}},
\]
where \(c\) is a fixed scale factor. This ablation isolates the effect of gradient-norm balancing from the effect of directly penalizing empirical CMI.

\subsection{Compact Numerical Summary of Low-separation-violation Operating Points}
\label{append:compact_operating_points}

The main figures show the full empirical separation–utility frontier, while the compact tables report operational summaries at representative low-violation thresholds. For each method and threshold, we select the trained operating point with the highest mean utility among those satisfying the corresponding violation constraint. This avoids comparing methods at arbitrary values of their fairness parameter and instead asks which method achieves the best predictive performance under the same fairness budget.

For the binary benchmarks, we report the best Accuracy and AUROC subject to EO-gap constraints:
\[
    \mathrm{EO\text{-}gap} \leq \epsilon,
    \qquad
    \epsilon \in \{0.01,0.02,0.05\}.
\]
For ACSOccupation, where the operational plots use posterior-level separation violation directly, we report the best expected randomized Accuracy and macro-AUROC subject to CMI constraints:
\[
    I(\hat Y;Z\mid Y) \leq \epsilon,
    \qquad
    \epsilon \in \{0.0025,0.005,0.01\}.
\]
All entries are computed from mean performance over cross-validation folds. The tables provide compact numerical summaries at representative low-separation-violation thresholds, while the plots show the full empirical separation-utility trade-off traced by the trained operating points.

\begin{table*}[t]
\centering
\caption{
\textbf{Adult compact operating-point summary.}
For each method, policy, constraint threshold, and target metric, we report the best mean target value among trained operating points satisfying the corresponding EO-gap constraint.
}
\label{tab:adult_compact_operating_points}
\resizebox{\textwidth}{!}{
\begin{tabular}{lllrlrlrl}
\toprule
Method & Policy & $\lambda$ & Constraint & Threshold & Constraint value & Metric & Target value & Folds \\
\midrule
Adversarial & randomized & 100 & EO gap & 0.01 & 0.0008638 $\pm$ 0.00069 & Accuracy & 0.7251 $\pm$ 0.024 & 5 \\
FairDummy & randomized & 1.678e+07 & EO gap & 0.01 & 0.008873 $\pm$ 0.0061 & Accuracy & 0.6816 $\pm$ 0.003 & 5 \\
CMI & randomized & 0.8 & EO gap & 0.01 & 0.007937 $\pm$ 0.0031 & Accuracy & 0.68 $\pm$ 0.0027 & 5 \\
FR-Train & randomized & 3.162 & EO gap & 0.01 & 0.0027 $\pm$ 0.0054 & Accuracy & 0.5466 $\pm$ 0.28 & 5 \\
CMI & randomized & 0.8 & EO gap & 0.01 & 0.007937 $\pm$ 0.0031 & AUROC & 0.7934 $\pm$ 0.0059 & 5 \\
FairDummy & randomized & 1.678e+07 & EO gap & 0.01 & 0.008873 $\pm$ 0.0061 & AUROC & 0.7364 $\pm$ 0.01 & 5 \\
Adversarial & randomized & 3.162 & EO gap & 0.01 & 0.001216 $\pm$ 0.00079 & AUROC & 0.5994 $\pm$ 0.014 & 5 \\
FR-Train & randomized & 31.62 & EO gap & 0.01 & 0.004166 $\pm$ 0.0053 & AUROC & 0.5108 $\pm$ 0.16 & 5 \\
Adversarial & randomized & 100 & EO gap & 0.02 & 0.0008638 $\pm$ 0.00069 & Accuracy & 0.7251 $\pm$ 0.024 & 5 \\
ERM-Deep (X) & randomized & -- & EO gap & 0.02 & 0.01204 $\pm$ 0.0087 & Accuracy & 0.7125 $\pm$ 0.0026 & 5 \\
ExpGrad & randomized & -- & EO gap & 0.02 & 0.01204 $\pm$ 0.0087 & Accuracy & 0.7125 $\pm$ 0.0026 & 5 \\
CMI & randomized & 0.7 & EO gap & 0.02 & 0.01511 $\pm$ 0.0033 & Accuracy & 0.6959 $\pm$ 0.0017 & 5 \\
FairDummy & randomized & 2.097e+06 & EO gap & 0.02 & 0.0128 $\pm$ 0.0063 & Accuracy & 0.6879 $\pm$ 0.0043 & 5 \\
FR-Train & randomized & 3.162 & EO gap & 0.02 & 0.0027 $\pm$ 0.0054 & Accuracy & 0.5466 $\pm$ 0.28 & 5 \\
CMI & randomized & 0.7 & EO gap & 0.02 & 0.01511 $\pm$ 0.0033 & AUROC & 0.7936 $\pm$ 0.0072 & 5 \\
ERM-Deep (X) & randomized & -- & EO gap & 0.02 & 0.01204 $\pm$ 0.0087 & AUROC & 0.7928 $\pm$ 0.0043 & 5 \\
ExpGrad & randomized & -- & EO gap & 0.02 & 0.01204 $\pm$ 0.0087 & AUROC & 0.7928 $\pm$ 0.0043 & 5 \\
FairDummy & randomized & 2.097e+06 & EO gap & 0.02 & 0.0128 $\pm$ 0.0063 & AUROC & 0.7535 $\pm$ 0.014 & 5 \\
Adversarial & randomized & 3.162 & EO gap & 0.02 & 0.001216 $\pm$ 0.00079 & AUROC & 0.5994 $\pm$ 0.014 & 5 \\
FR-Train & randomized & 31.62 & EO gap & 0.02 & 0.004166 $\pm$ 0.0053 & AUROC & 0.5108 $\pm$ 0.16 & 5 \\
Adversarial & randomized & 100 & EO gap & 0.05 & 0.0008638 $\pm$ 0.00069 & Accuracy & 0.7251 $\pm$ 0.024 & 5 \\
CMI & randomized & 0.3 & EO gap & 0.05 & 0.0493 $\pm$ 0.0055 & Accuracy & 0.7132 $\pm$ 0.0023 & 5 \\
ERM-Deep (X) & randomized & -- & EO gap & 0.05 & 0.01204 $\pm$ 0.0087 & Accuracy & 0.7125 $\pm$ 0.0026 & 5 \\
ExpGrad & randomized & -- & EO gap & 0.05 & 0.01204 $\pm$ 0.0087 & Accuracy & 0.7125 $\pm$ 0.0026 & 5 \\
FairDummy & randomized & 4096 & EO gap & 0.05 & 0.03471 $\pm$ 0.015 & Accuracy & 0.7103 $\pm$ 0.0061 & 5 \\
FairDRO & randomized & 0.2 & EO gap & 0.05 & 0.03878 $\pm$ 0.0091 & Accuracy & 0.6665 $\pm$ 0.0057 & 5 \\
FR-Train & randomized & 3.162 & EO gap & 0.05 & 0.0027 $\pm$ 0.0054 & Accuracy & 0.5466 $\pm$ 0.28 & 5 \\
InfoFair & randomized & 1 & EO gap & 0.05 & 0.02015 $\pm$ 0.015 & Accuracy & 0.4376 $\pm$ 0.14 & 5 \\
FairDummy & randomized & 4096 & EO gap & 0.05 & 0.03471 $\pm$ 0.015 & AUROC & 0.8014 $\pm$ 0.003 & 5 \\
CMI & randomized & 0.3 & EO gap & 0.05 & 0.0493 $\pm$ 0.0055 & AUROC & 0.7983 $\pm$ 0.005 & 5 \\
FairDRO & randomized & 0.2 & EO gap & 0.05 & 0.03878 $\pm$ 0.0091 & AUROC & 0.7954 $\pm$ 0.0036 & 5 \\
ERM-Deep (X) & randomized & -- & EO gap & 0.05 & 0.01204 $\pm$ 0.0087 & AUROC & 0.7928 $\pm$ 0.0043 & 5 \\
ExpGrad & randomized & -- & EO gap & 0.05 & 0.01204 $\pm$ 0.0087 & AUROC & 0.7928 $\pm$ 0.0043 & 5 \\
Adversarial & randomized & 3.162 & EO gap & 0.05 & 0.001216 $\pm$ 0.00079 & AUROC & 0.5994 $\pm$ 0.014 & 5 \\
FR-Train & randomized & 31.62 & EO gap & 0.05 & 0.004166 $\pm$ 0.0053 & AUROC & 0.5108 $\pm$ 0.16 & 5 \\
InfoFair & randomized & 1 & EO gap & 0.05 & 0.02015 $\pm$ 0.015 & AUROC & 0.4499 $\pm$ 0.16 & 5 \\
Adversarial & deterministic & 1 & EO gap & 0.01 & 0.004927 $\pm$ 0.0031 & Accuracy & 0.7461 $\pm$ 0.0054 & 5 \\
FairDummy & deterministic & 1.074e+09 & EO gap & 0.01 & 0.002352 $\pm$ 0.0053 & Accuracy & 0.2904 $\pm$ 0.11 & 5 \\
FairDummy & deterministic & 512 & EO gap & 0.01 & 0.008689 $\pm$ 0.01 & AUROC & 0.8103 $\pm$ 0.0029 & 5 \\
Adversarial & deterministic & 3.162 & EO gap & 0.01 & 0.003908 $\pm$ 0.0021 & AUROC & 0.5994 $\pm$ 0.014 & 5 \\
CMI & deterministic & 1 & EO gap & 0.02 & 0.01796 $\pm$ 0.018 & Accuracy & 0.7509 $\pm$ 0.007 & 5 \\
Adversarial & deterministic & 1 & EO gap & 0.02 & 0.004927 $\pm$ 0.0031 & Accuracy & 0.7461 $\pm$ 0.0054 & 5 \\
InfoFair & deterministic & 1 & EO gap & 0.02 & 0.01596 $\pm$ 0.013 & Accuracy & 0.7266 $\pm$ 0.033 & 5 \\
FairDummy & deterministic & 2.621e+05 & EO gap & 0.02 & 0.01453 $\pm$ 0.01 & Accuracy & 0.3988 $\pm$ 0.22 & 5 \\
FR-Train & deterministic & 316.2 & EO gap & 0.02 & 0.01487 $\pm$ 0.018 & Accuracy & 0.3799 $\pm$ 0.22 & 5 \\
FairDummy & deterministic & 512 & EO gap & 0.02 & 0.008689 $\pm$ 0.01 & AUROC & 0.8103 $\pm$ 0.0029 & 5 \\
Adversarial & deterministic & 3.162 & EO gap & 0.02 & 0.003908 $\pm$ 0.0021 & AUROC & 0.5994 $\pm$ 0.014 & 5 \\
FR-Train & deterministic & 31.62 & EO gap & 0.02 & 0.01447 $\pm$ 0.015 & AUROC & 0.5108 $\pm$ 0.16 & 5 \\
CMI & deterministic & 1 & EO gap & 0.02 & 0.01796 $\pm$ 0.018 & AUROC & 0.4935 $\pm$ 0.1 & 5 \\
InfoFair & deterministic & 1 & EO gap & 0.02 & 0.01596 $\pm$ 0.013 & AUROC & 0.4499 $\pm$ 0.16 & 5 \\
FairDRO & deterministic & 0.1 & EO gap & 0.05 & 0.03358 $\pm$ 0.034 & Accuracy & 0.7916 $\pm$ 0.0067 & 5 \\
CMI & deterministic & 0.2 & EO gap & 0.05 & 0.03149 $\pm$ 0.034 & Accuracy & 0.7914 $\pm$ 0.0062 & 5 \\
ERM-Deep (X) & deterministic & -- & EO gap & 0.05 & 0.03604 $\pm$ 0.018 & Accuracy & 0.79 $\pm$ 0.0058 & 5 \\
InfoFair & deterministic & 0.8 & EO gap & 0.05 & 0.03269 $\pm$ 0.023 & Accuracy & 0.7892 $\pm$ 0.0066 & 5 \\
ExpGrad & deterministic & -- & EO gap & 0.05 & 0.03228 $\pm$ 0.022 & Accuracy & 0.7886 $\pm$ 0.0047 & 5 \\
FR-Train & deterministic & 0.3162 & EO gap & 0.05 & 0.03425 $\pm$ 0.026 & Accuracy & 0.7835 $\pm$ 0.011 & 5 \\
Adversarial & deterministic & 1 & EO gap & 0.05 & 0.004927 $\pm$ 0.0031 & Accuracy & 0.7461 $\pm$ 0.0054 & 5 \\
FairDummy & deterministic & 2.621e+05 & EO gap & 0.05 & 0.01453 $\pm$ 0.01 & Accuracy & 0.3988 $\pm$ 0.22 & 5 \\
FairDummy & deterministic & 8 & EO gap & 0.05 & 0.02287 $\pm$ 0.011 & AUROC & 0.8177 $\pm$ 0.0028 & 5 \\
InfoFair & deterministic & 0.8 & EO gap & 0.05 & 0.03269 $\pm$ 0.023 & AUROC & 0.814 $\pm$ 0.0048 & 5 \\
FR-Train & deterministic & 0.3162 & EO gap & 0.05 & 0.03425 $\pm$ 0.026 & AUROC & 0.8055 $\pm$ 0.0032 & 5 \\
CMI & deterministic & 0.2 & EO gap & 0.05 & 0.03149 $\pm$ 0.034 & AUROC & 0.8028 $\pm$ 0.0047 & 5 \\
FairDRO & deterministic & 0.1 & EO gap & 0.05 & 0.03358 $\pm$ 0.034 & AUROC & 0.8027 $\pm$ 0.0049 & 5 \\
ERM-Deep (X) & deterministic & -- & EO gap & 0.05 & 0.03604 $\pm$ 0.018 & AUROC & 0.7928 $\pm$ 0.0043 & 5 \\
ExpGrad & deterministic & -- & EO gap & 0.05 & 0.03228 $\pm$ 0.022 & AUROC & 0.7928 $\pm$ 0.0043 & 5 \\
Adversarial & deterministic & 3.162 & EO gap & 0.05 & 0.003908 $\pm$ 0.0021 & AUROC & 0.5994 $\pm$ 0.014 & 5 \\
\bottomrule
\end{tabular}

}
\end{table*}

\begin{table*}[t]
\centering
\caption{
\textbf{COMPAS compact operating-point summary.}
For each method, policy, constraint threshold, and target metric, we report the best mean target value among trained operating points satisfying the corresponding EO-gap constraint.
}
\label{tab:compas_compact_operating_points}
\resizebox{\textwidth}{!}{
\begin{tabular}{lllrlrlrl}
\toprule
Method & Policy & $\lambda$ & Constraint & Threshold & Constraint value & Metric & Target value & Folds \\
\midrule
CMI & randomized & 0.6 & EO gap & 0.01 & 0.009576 $\pm$ 0.0055 & Accuracy & 0.5455 $\pm$ 0.0055 & 5 \\
FairDummy & randomized & 1.074e+09 & EO gap & 0.01 & 0.009044 $\pm$ 0.0081 & Accuracy & 0.5293 $\pm$ 0.0039 & 5 \\
Adversarial & randomized & 1 & EO gap & 0.01 & 4.172e-08 $\pm$ 2.7e-08 & Accuracy & 0.4999 $\pm$ 8e-05 & 5 \\
CMI & randomized & 0.7 & EO gap & 0.01 & 0.00714 $\pm$ 0.0044 & AUROC & 0.6966 $\pm$ 0.0035 & 5 \\
FairDummy & randomized & 1.074e+09 & EO gap & 0.01 & 0.009044 $\pm$ 0.0081 & AUROC & 0.6881 $\pm$ 0.012 & 5 \\
Adversarial & randomized & 1 & EO gap & 0.01 & 4.172e-08 $\pm$ 2.7e-08 & AUROC & 0.5 $\pm$ 0 & 5 \\
FairDRO & randomized & 0.5 & EO gap & 0.02 & 0.01965 $\pm$ 0.011 & Accuracy & 0.5609 $\pm$ 0.0035 & 5 \\
CMI & randomized & 0.3 & EO gap & 0.02 & 0.01904 $\pm$ 0.012 & Accuracy & 0.5598 $\pm$ 0.0038 & 5 \\
FairDummy & randomized & 4096 & EO gap & 0.02 & 0.01399 $\pm$ 0.0079 & Accuracy & 0.5532 $\pm$ 0.0035 & 5 \\
InfoFair & randomized & 0.9 & EO gap & 0.02 & 0.01339 $\pm$ 0.015 & Accuracy & 0.5334 $\pm$ 0.008 & 5 \\
Adversarial & randomized & 1 & EO gap & 0.02 & 4.172e-08 $\pm$ 2.7e-08 & Accuracy & 0.4999 $\pm$ 8e-05 & 5 \\
FairDummy & randomized & 4096 & EO gap & 0.02 & 0.01399 $\pm$ 0.0079 & AUROC & 0.6996 $\pm$ 0.0035 & 5 \\
CMI & randomized & 0.7 & EO gap & 0.02 & 0.00714 $\pm$ 0.0044 & AUROC & 0.6966 $\pm$ 0.0035 & 5 \\
InfoFair & randomized & 0.9 & EO gap & 0.02 & 0.01339 $\pm$ 0.015 & AUROC & 0.6955 $\pm$ 0.0055 & 5 \\
FairDRO & randomized & 0.9 & EO gap & 0.02 & 0.0142 $\pm$ 0.0025 & AUROC & 0.6942 $\pm$ 0.0038 & 5 \\
Adversarial & randomized & 1 & EO gap & 0.02 & 4.172e-08 $\pm$ 2.7e-08 & AUROC & 0.5 $\pm$ 0 & 5 \\
FR-Train & randomized & 0.2 & EO gap & 0.05 & 0.04463 $\pm$ 0.021 & Accuracy & 0.5666 $\pm$ 0.0016 & 5 \\
Adversarial & randomized & 0.2 & EO gap & 0.05 & 0.03988 $\pm$ 0.021 & Accuracy & 0.5664 $\pm$ 0.0016 & 5 \\
CMI & randomized & 0.1 & EO gap & 0.05 & 0.04654 $\pm$ 0.022 & Accuracy & 0.5654 $\pm$ 0.0029 & 5 \\
FairDRO & randomized & 0.2 & EO gap & 0.05 & 0.04462 $\pm$ 0.017 & Accuracy & 0.5652 $\pm$ 0.0036 & 5 \\
FairDummy & randomized & 512 & EO gap & 0.05 & 0.03927 $\pm$ 0.017 & Accuracy & 0.5643 $\pm$ 0.0036 & 5 \\
InfoFair & randomized & 0.7 & EO gap & 0.05 & 0.02932 $\pm$ 0.017 & Accuracy & 0.56 $\pm$ 0.0053 & 5 \\
Adversarial & randomized & 0.6 & EO gap & 0.05 & 0.03976 $\pm$ 0.02 & AUROC & 0.703 $\pm$ 0.0063 & 5 \\
FairDummy & randomized & 512 & EO gap & 0.05 & 0.03927 $\pm$ 0.017 & AUROC & 0.7011 $\pm$ 0.0038 & 5 \\
InfoFair & randomized & 0.8 & EO gap & 0.05 & 0.02847 $\pm$ 0.02 & AUROC & 0.6985 $\pm$ 0.0069 & 5 \\
CMI & randomized & 0.7 & EO gap & 0.05 & 0.00714 $\pm$ 0.0044 & AUROC & 0.6966 $\pm$ 0.0035 & 5 \\
FR-Train & randomized & 0.2 & EO gap & 0.05 & 0.04463 $\pm$ 0.021 & AUROC & 0.6959 $\pm$ 0.0036 & 5 \\
FairDRO & randomized & 0.2 & EO gap & 0.05 & 0.04462 $\pm$ 0.017 & AUROC & 0.6951 $\pm$ 0.0046 & 5 \\
FR-Train & deterministic & 1 & EO gap & 0.01 & 0 $\pm$ 0 & Accuracy & 0.5001 $\pm$ 0.0036 & 5 \\
Adversarial & deterministic & 1 & EO gap & 0.01 & 0 $\pm$ 0 & Accuracy & 0.4999 $\pm$ 0.0036 & 5 \\
InfoFair & deterministic & 1 & EO gap & 0.01 & 0.002394 $\pm$ 0.0054 & Accuracy & 0.4988 $\pm$ 0.0022 & 5 \\
Adversarial & deterministic & 1 & EO gap & 0.01 & 0 $\pm$ 0 & AUROC & 0.5 $\pm$ 0 & 5 \\
InfoFair & deterministic & 1 & EO gap & 0.01 & 0.002394 $\pm$ 0.0054 & AUROC & 0.4248 $\pm$ 0.072 & 5 \\
FairDummy & deterministic & 1.074e+09 & EO gap & 0.02 & 0.01374 $\pm$ 0.01 & Accuracy & 0.5231 $\pm$ 0.028 & 5 \\
FR-Train & deterministic & 1 & EO gap & 0.02 & 0 $\pm$ 0 & Accuracy & 0.5001 $\pm$ 0.0036 & 5 \\
Adversarial & deterministic & 1 & EO gap & 0.02 & 0 $\pm$ 0 & Accuracy & 0.4999 $\pm$ 0.0036 & 5 \\
InfoFair & deterministic & 1 & EO gap & 0.02 & 0.002394 $\pm$ 0.0054 & Accuracy & 0.4988 $\pm$ 0.0022 & 5 \\
FairDummy & deterministic & 1.678e+07 & EO gap & 0.02 & 0.01073 $\pm$ 0.012 & AUROC & 0.6884 $\pm$ 0.013 & 5 \\
Adversarial & deterministic & 1 & EO gap & 0.02 & 0 $\pm$ 0 & AUROC & 0.5 $\pm$ 0 & 5 \\
InfoFair & deterministic & 1 & EO gap & 0.02 & 0.002394 $\pm$ 0.0054 & AUROC & 0.4248 $\pm$ 0.072 & 5 \\
FairDummy & deterministic & 512 & EO gap & 0.05 & 0.02826 $\pm$ 0.017 & Accuracy & 0.5787 $\pm$ 0.071 & 5 \\
FR-Train & deterministic & 0.9 & EO gap & 0.05 & 0.0474 $\pm$ 0.059 & Accuracy & 0.558 $\pm$ 0.078 & 5 \\
CMI & deterministic & 1 & EO gap & 0.05 & 0.02069 $\pm$ 0.018 & Accuracy & 0.5018 $\pm$ 0.004 & 5 \\
Adversarial & deterministic & 1 & EO gap & 0.05 & 0 $\pm$ 0 & Accuracy & 0.4999 $\pm$ 0.0036 & 5 \\
InfoFair & deterministic & 1 & EO gap & 0.05 & 0.002394 $\pm$ 0.0054 & Accuracy & 0.4988 $\pm$ 0.0022 & 5 \\
FairDummy & deterministic & 8 & EO gap & 0.05 & 0.0234 $\pm$ 0.027 & AUROC & 0.7026 $\pm$ 0.0041 & 5 \\
FR-Train & deterministic & 0.9 & EO gap & 0.05 & 0.0474 $\pm$ 0.059 & AUROC & 0.6386 $\pm$ 0.053 & 5 \\
CMI & deterministic & 1 & EO gap & 0.05 & 0.02069 $\pm$ 0.018 & AUROC & 0.5247 $\pm$ 0.071 & 5 \\
Adversarial & deterministic & 1 & EO gap & 0.05 & 0 $\pm$ 0 & AUROC & 0.5 $\pm$ 0 & 5 \\
InfoFair & deterministic & 1 & EO gap & 0.05 & 0.002394 $\pm$ 0.0054 & AUROC & 0.4248 $\pm$ 0.072 & 5 \\
\bottomrule
\end{tabular}

}
\end{table*}

\begin{table*}[t]
\centering
\caption{
\textbf{Bank compact operating-point summary.}
We report randomized-policy operating points only. For each method, EO-gap threshold, and target metric, we report the best mean target value among trained operating points satisfying the corresponding EO-gap constraint.
}
\label{tab:bank_compact_operating_points}
\resizebox{\textwidth}{!}{
\begin{tabular}{lllrlrlrl}
\toprule
Method & Policy & $\lambda$ & Constraint & Threshold & Constraint value & Metric & Target value & Folds \\
\midrule
Adversarial & randomized & 31.62 & EO gap & 0.01 & 0.0005504 $\pm$ 0.00028 & Accuracy & 0.8809 $\pm$ 0.00043 & 5 \\
FairDummy & randomized & 1.342e+08 & EO gap & 0.01 & 0.004024 $\pm$ 0.003 & Accuracy & 0.8644 $\pm$ 0.0092 & 5 \\
CMI & randomized & 0.8 & EO gap & 0.01 & 0.006713 $\pm$ 0.0034 & Accuracy & 0.8303 $\pm$ 0.011 & 5 \\
FR-Train & randomized & 10 & EO gap & 0.01 & 0.00246 $\pm$ 0.0055 & Accuracy & 0.7168 $\pm$ 0.34 & 5 \\
CMI & randomized & 0.8 & EO gap & 0.01 & 0.006713 $\pm$ 0.0034 & AUROC & 0.8798 $\pm$ 0.0054 & 5 \\
FairDummy & randomized & 1.678e+07 & EO gap & 0.01 & 0.004767 $\pm$ 0.0016 & AUROC & 0.623 $\pm$ 0.099 & 5 \\
Adversarial & randomized & 100 & EO gap & 0.01 & 0.002735 $\pm$ 0.0049 & AUROC & 0.5181 $\pm$ 0.078 & 5 \\
FR-Train & randomized & 31.62 & EO gap & 0.01 & 0.003749 $\pm$ 0.0057 & AUROC & 0.5072 $\pm$ 0.01 & 5 \\
Adversarial & randomized & 31.62 & EO gap & 0.02 & 0.0005504 $\pm$ 0.00028 & Accuracy & 0.8809 $\pm$ 0.00043 & 5 \\
CMI & randomized & 0.1 & EO gap & 0.02 & 0.01743 $\pm$ 0.0028 & Accuracy & 0.8645 $\pm$ 0.0032 & 5 \\
FairDummy & randomized & 1.342e+08 & EO gap & 0.02 & 0.004024 $\pm$ 0.003 & Accuracy & 0.8644 $\pm$ 0.0092 & 5 \\
FR-Train & randomized & 0.3162 & EO gap & 0.02 & 0.01397 $\pm$ 0.0032 & Accuracy & 0.8622 $\pm$ 0.014 & 5 \\
InfoFair & randomized & 0.7 & EO gap & 0.02 & 0.01789 $\pm$ 0.0044 & Accuracy & 0.834 $\pm$ 0.0016 & 5 \\
CMI & randomized & 0.1 & EO gap & 0.02 & 0.01743 $\pm$ 0.0028 & AUROC & 0.914 $\pm$ 0.0034 & 5 \\
FR-Train & randomized & 0.3162 & EO gap & 0.02 & 0.01397 $\pm$ 0.0032 & AUROC & 0.9086 $\pm$ 0.0048 & 5 \\
FairDummy & randomized & 4096 & EO gap & 0.02 & 0.01878 $\pm$ 0.0069 & AUROC & 0.9084 $\pm$ 0.011 & 5 \\
Adversarial & randomized & 0.3162 & EO gap & 0.02 & 0.01478 $\pm$ 0.0063 & AUROC & 0.9007 $\pm$ 0.012 & 5 \\
InfoFair & randomized & 0.7 & EO gap & 0.02 & 0.01789 $\pm$ 0.0044 & AUROC & 0.8435 $\pm$ 0.027 & 5 \\
\bottomrule
\end{tabular}

}
\end{table*}

\begin{table*}[t]
\centering
\caption{
\textbf{CelebA compact operating-point summary.}
For each method, policy, constraint threshold, and target metric, we report the best mean target value among trained operating points satisfying the corresponding EO-gap constraint.
}
\label{tab:celeba_compact_operating_points}
\resizebox{\textwidth}{!}{
\begin{tabular}{lllrlrlrl}
\toprule
Method & Policy & $\lambda$ & Constraint & Threshold & Constraint value & Metric & Target value & Folds \\
\midrule
CMI & randomized & 0.9 & EO gap & 0.01 & 0.004135 $\pm$ 0.002 & Accuracy & 0.5446 $\pm$ 0.01 & 5 \\
InfoFair & randomized & 1 & EO gap & 0.01 & 3.771e-42 $\pm$ 8.4e-42 & Accuracy & 0.5036 $\pm$ 0.02 & 5 \\
CMI & randomized & 0.9 & EO gap & 0.01 & 0.004135 $\pm$ 0.002 & AUROC & 0.8942 $\pm$ 0.0051 & 5 \\
InfoFair & randomized & 1 & EO gap & 0.01 & 3.771e-42 $\pm$ 8.4e-42 & AUROC & 0.5 $\pm$ 7.9e-07 & 5 \\
CMI & randomized & 0.8 & EO gap & 0.02 & 0.01955 $\pm$ 0.008 & Accuracy & 0.7102 $\pm$ 0.015 & 5 \\
InfoFair & randomized & 1 & EO gap & 0.02 & 3.771e-42 $\pm$ 8.4e-42 & Accuracy & 0.5036 $\pm$ 0.02 & 5 \\
CMI & randomized & 0.8 & EO gap & 0.02 & 0.01955 $\pm$ 0.008 & AUROC & 0.9091 $\pm$ 0.0022 & 5 \\
InfoFair & randomized & 1 & EO gap & 0.02 & 3.771e-42 $\pm$ 8.4e-42 & AUROC & 0.5 $\pm$ 7.9e-07 & 5 \\
CMI & randomized & 0.5 & EO gap & 0.05 & 0.04533 $\pm$ 0.022 & Accuracy & 0.7736 $\pm$ 0.0032 & 5 \\
FairDRO & randomized & 0.9 & EO gap & 0.05 & 0.04862 $\pm$ 0.011 & Accuracy & 0.7617 $\pm$ 0.0029 & 5 \\
InfoFair & randomized & 1 & EO gap & 0.05 & 3.771e-42 $\pm$ 8.4e-42 & Accuracy & 0.5036 $\pm$ 0.02 & 5 \\
CMI & randomized & 0.5 & EO gap & 0.05 & 0.04533 $\pm$ 0.022 & AUROC & 0.9201 $\pm$ 0.0019 & 5 \\
FairDRO & randomized & 0.9 & EO gap & 0.05 & 0.04862 $\pm$ 0.011 & AUROC & 0.9194 $\pm$ 0.0012 & 5 \\
InfoFair & randomized & 1 & EO gap & 0.05 & 3.771e-42 $\pm$ 8.4e-42 & AUROC & 0.5 $\pm$ 7.9e-07 & 5 \\
Adversarial & deterministic & 9 & EO gap & 0.01 & 0.007608 $\pm$ 0.0063 & Accuracy & 0.6433 $\pm$ 0.028 & 5 \\
FR-Train & deterministic & 1e+08 & EO gap & 0.01 & 0.001792 $\pm$ 0.0024 & Accuracy & 0.5258 $\pm$ 0.024 & 5 \\
FairDummy & deterministic & 512 & EO gap & 0.01 & 0.002576 $\pm$ 0.0014 & Accuracy & 0.5236 $\pm$ 0.019 & 5 \\
CMI & deterministic & 1 & EO gap & 0.01 & 0.0002272 $\pm$ 0.00027 & Accuracy & 0.4821 $\pm$ 5.9e-05 & 5 \\
InfoFair & deterministic & 1 & EO gap & 0.01 & 0 $\pm$ 0 & Accuracy & 0.4821 $\pm$ 5.7e-06 & 5 \\
FairDummy & deterministic & 4096 & EO gap & 0.01 & 0.00448 $\pm$ 0.003 & AUROC & 0.9221 $\pm$ 0.0017 & 5 \\
FR-Train & deterministic & 1e+08 & EO gap & 0.01 & 0.001792 $\pm$ 0.0024 & AUROC & 0.9212 $\pm$ 0.0012 & 5 \\
Adversarial & deterministic & 9 & EO gap & 0.01 & 0.007608 $\pm$ 0.0063 & AUROC & 0.921 $\pm$ 0.00098 & 5 \\
CMI & deterministic & 1 & EO gap & 0.01 & 0.0002272 $\pm$ 0.00027 & AUROC & 0.5 $\pm$ 6.1e-05 & 5 \\
InfoFair & deterministic & 1 & EO gap & 0.01 & 0 $\pm$ 0 & AUROC & 0.5 $\pm$ 7.9e-07 & 5 \\
Adversarial & deterministic & 1e+08 & EO gap & 0.02 & 0.013 $\pm$ 0.0071 & Accuracy & 0.6637 $\pm$ 0.03 & 5 \\
FR-Train & deterministic & 9 & EO gap & 0.02 & 0.01238 $\pm$ 0.0097 & Accuracy & 0.6584 $\pm$ 0.12 & 5 \\
FairDummy & deterministic & 8 & EO gap & 0.02 & 0.01695 $\pm$ 0.0096 & Accuracy & 0.6507 $\pm$ 0.092 & 5 \\
CMI & deterministic & 1 & EO gap & 0.02 & 0.0002272 $\pm$ 0.00027 & Accuracy & 0.4821 $\pm$ 5.9e-05 & 5 \\
InfoFair & deterministic & 1 & EO gap & 0.02 & 0 $\pm$ 0 & Accuracy & 0.4821 $\pm$ 5.7e-06 & 5 \\
FairDummy & deterministic & 4096 & EO gap & 0.02 & 0.00448 $\pm$ 0.003 & AUROC & 0.9221 $\pm$ 0.0017 & 5 \\
Adversarial & deterministic & 1e+08 & EO gap & 0.02 & 0.013 $\pm$ 0.0071 & AUROC & 0.922 $\pm$ 0.0012 & 5 \\
FR-Train & deterministic & 1e+08 & EO gap & 0.02 & 0.001792 $\pm$ 0.0024 & AUROC & 0.9212 $\pm$ 0.0012 & 5 \\
CMI & deterministic & 1 & EO gap & 0.02 & 0.0002272 $\pm$ 0.00027 & AUROC & 0.5 $\pm$ 6.1e-05 & 5 \\
InfoFair & deterministic & 1 & EO gap & 0.02 & 0 $\pm$ 0 & AUROC & 0.5 $\pm$ 7.9e-07 & 5 \\
ExpGrad & deterministic & -- & EO gap & 0.05 & 0.04641 $\pm$ 0.014 & Accuracy & 0.843 $\pm$ 0.0021 & 5 \\
FairDRO & deterministic & 0.7 & EO gap & 0.05 & 0.04355 $\pm$ 0.014 & Accuracy & 0.8396 $\pm$ 0.0027 & 5 \\
CMI & deterministic & 0.7 & EO gap & 0.05 & 0.04134 $\pm$ 0.007 & Accuracy & 0.8347 $\pm$ 0.0023 & 5 \\
FR-Train & deterministic & 2.333 & EO gap & 0.05 & 0.04868 $\pm$ 0.012 & Accuracy & 0.8081 $\pm$ 0.014 & 5 \\
Adversarial & deterministic & 1.5 & EO gap & 0.05 & 0.0447 $\pm$ 0.01 & Accuracy & 0.8056 $\pm$ 0.0073 & 5 \\
InfoFair & deterministic & 0.9 & EO gap & 0.05 & 0.04636 $\pm$ 0.022 & Accuracy & 0.7912 $\pm$ 0.02 & 5 \\
FairDummy & deterministic & 8 & EO gap & 0.05 & 0.01695 $\pm$ 0.0096 & Accuracy & 0.6507 $\pm$ 0.092 & 5 \\
Adversarial & deterministic & 2.333 & EO gap & 0.05 & 0.03128 $\pm$ 0.0019 & AUROC & 0.9223 $\pm$ 0.00083 & 5 \\
FR-Train & deterministic & 2.333 & EO gap & 0.05 & 0.04868 $\pm$ 0.012 & AUROC & 0.9221 $\pm$ 0.00077 & 5 \\
FairDummy & deterministic & 4096 & EO gap & 0.05 & 0.00448 $\pm$ 0.003 & AUROC & 0.9221 $\pm$ 0.0017 & 5 \\
ExpGrad & deterministic & -- & EO gap & 0.05 & 0.04641 $\pm$ 0.014 & AUROC & 0.9216 $\pm$ 0.001 & 5 \\
FairDRO & deterministic & 0.7 & EO gap & 0.05 & 0.04355 $\pm$ 0.014 & AUROC & 0.9207 $\pm$ 0.0017 & 5 \\
CMI & deterministic & 0.7 & EO gap & 0.05 & 0.04134 $\pm$ 0.007 & AUROC & 0.9148 $\pm$ 0.0018 & 5 \\
InfoFair & deterministic & 0.9 & EO gap & 0.05 & 0.04636 $\pm$ 0.022 & AUROC & 0.8597 $\pm$ 0.041 & 5 \\
\bottomrule
\end{tabular}

}
\end{table*}

\begin{table*}[t]
\centering
\caption{
\textbf{ACSOccupation compact operating-point summary.}
We report randomized-policy operating points only. For each method, CMI threshold, and target metric, we report the best mean target value among trained operating points satisfying the corresponding CMI constraint.
}
\label{tab:acs_compact_operating_points}
\resizebox{\textwidth}{!}{
\begin{tabular}{lllrlrlrl}
\toprule
Method & Policy & $\lambda$ & Constraint & Threshold & Constraint value & Metric & Target value & Folds \\
\midrule
CMI & randomized & 0.15 & $I(\hat Y;Z\mid Y)$ & 0.005 & 0.004346 $\pm$ 0.00025 & Accuracy & 0.08734 $\pm$ 0.00069 & 5 \\
FR-Train & randomized & 0.05 & $I(\hat Y;Z\mid Y)$ & 0.005 & 0.002536 $\pm$ 0.00055 & Accuracy & 0.08654 $\pm$ 0.0049 & 5 \\
InfoFair & randomized & 0.65 & $I(\hat Y;Z\mid Y)$ & 0.005 & 0.004325 $\pm$ 0.0013 & Accuracy & 0.08229 $\pm$ 0.00098 & 5 \\
Adversarial & randomized & 0.4 & $I(\hat Y;Z\mid Y)$ & 0.005 & 0.003904 $\pm$ 0.0069 & Accuracy & 0.05458 $\pm$ 0.0031 & 5 \\
Uniform random & randomized & -- & $I(\hat Y;Z\mid Y)$ & 0.005 & -6.19e-17 $\pm$ 1.1e-16 & Accuracy & 0.05 $\pm$ 4.8e-17 & 5 \\
CMI & randomized & 0.4 & $I(\hat Y;Z\mid Y)$ & 0.005 & 0.000983 $\pm$ 0.00014 & AUROC & 0.5999 $\pm$ 0.018 & 5 \\
FR-Train & randomized & 0.3 & $I(\hat Y;Z\mid Y)$ & 0.005 & 0.003222 $\pm$ 0.00029 & AUROC & 0.5902 $\pm$ 0.008 & 5 \\
InfoFair & randomized & 0.65 & $I(\hat Y;Z\mid Y)$ & 0.005 & 0.004325 $\pm$ 0.0013 & AUROC & 0.5662 $\pm$ 0.02 & 5 \\
Adversarial & randomized & 0.45 & $I(\hat Y;Z\mid Y)$ & 0.005 & 0.001113 $\pm$ 0.0012 & AUROC & 0.5263 $\pm$ 0.014 & 5 \\
Uniform random & randomized & -- & $I(\hat Y;Z\mid Y)$ & 0.005 & -6.19e-17 $\pm$ 1.1e-16 & AUROC & 0.5 $\pm$ 0 & 5 \\
CMI & randomized & 0.05 & $I(\hat Y;Z\mid Y)$ & 0.01 & 0.009051 $\pm$ 0.0012 & Accuracy & 0.08915 $\pm$ 0.00097 & 5 \\
FR-Train & randomized & 0.95 & $I(\hat Y;Z\mid Y)$ & 0.01 & 0.006884 $\pm$ 0.0051 & Accuracy & 0.08869 $\pm$ 0.0071 & 5 \\
InfoFair & randomized & 0.35 & $I(\hat Y;Z\mid Y)$ & 0.01 & 0.009316 $\pm$ 0.00072 & Accuracy & 0.0878 $\pm$ 0.0019 & 5 \\
Adversarial & randomized & 0.05 & $I(\hat Y;Z\mid Y)$ & 0.01 & 0.005646 $\pm$ 0.00064 & Accuracy & 0.08439 $\pm$ 0.0044 & 5 \\
FairDRO & randomized & 0.2 & $I(\hat Y;Z\mid Y)$ & 0.01 & 0.009551 $\pm$ 0.00037 & Accuracy & 0.07846 $\pm$ 0.00076 & 5 \\
Uniform random & randomized & -- & $I(\hat Y;Z\mid Y)$ & 0.01 & -6.19e-17 $\pm$ 1.1e-16 & Accuracy & 0.05 $\pm$ 4.8e-17 & 5 \\
InfoFair & randomized & 0.35 & $I(\hat Y;Z\mid Y)$ & 0.01 & 0.009316 $\pm$ 0.00072 & AUROC & 0.6034 $\pm$ 0.023 & 5 \\
CMI & randomized & 0.4 & $I(\hat Y;Z\mid Y)$ & 0.01 & 0.000983 $\pm$ 0.00014 & AUROC & 0.5999 $\pm$ 0.018 & 5 \\
FairDRO & randomized & 0.25 & $I(\hat Y;Z\mid Y)$ & 0.01 & 0.009227 $\pm$ 0.00033 & AUROC & 0.5956 $\pm$ 0.022 & 5 \\
FR-Train & randomized & 0.3 & $I(\hat Y;Z\mid Y)$ & 0.01 & 0.003222 $\pm$ 0.00029 & AUROC & 0.5902 $\pm$ 0.008 & 5 \\
Adversarial & randomized & 0.05 & $I(\hat Y;Z\mid Y)$ & 0.01 & 0.005646 $\pm$ 0.00064 & AUROC & 0.5862 $\pm$ 0.017 & 5 \\
Uniform random & randomized & -- & $I(\hat Y;Z\mid Y)$ & 0.01 & -6.19e-17 $\pm$ 1.1e-16 & AUROC & 0.5 $\pm$ 0 & 5 \\
\bottomrule
\end{tabular}

}
\end{table*}

\subsection{Large-Scale ACSOccupation Preprocessing Variant}
\label{append:acs_large}

The main ACSOccupation experiment uses a conservative complete-case preprocessing rule: rows with missing or structurally unavailable ACS covariate values are removed before one-hot encoding. This yields a clean multi-class/multi-group benchmark with no retained feature missingness, but it also substantially reduces the sample size because many ACS covariates contain structurally unavailable entries. In the processed main benchmark, this leaves \(2{,}576\) samples.

As an additional robustness check, we evaluate a large-scale ACSOccupation preprocessing variant. In this version, rows are removed only when the target occupation code, sensitive attribute, or age eligibility variable is invalid. Missing or structurally unavailable covariate values are retained as explicit categorical levels before one-hot encoding. This produces a much larger benchmark with \(575{,}989\) samples, \(55\) input features, \(20\) occupation classes, and \(9\) observed RAC1P groups.

Because retaining missing or structurally unavailable values as categorical levels may introduce additional distributional heterogeneity, we report this large-scale experiment as an appendix robustness check rather than as the main ACS benchmark. The goal is to test how the methods behave when the ACS task is scaled up substantially under a less aggressive preprocessing rule.

\begin{figure*}[t]
    \centering
    \setlength{\tabcolsep}{1pt}

    \subfloat[Information plane.]{
        \includegraphics[width=0.49\textwidth]{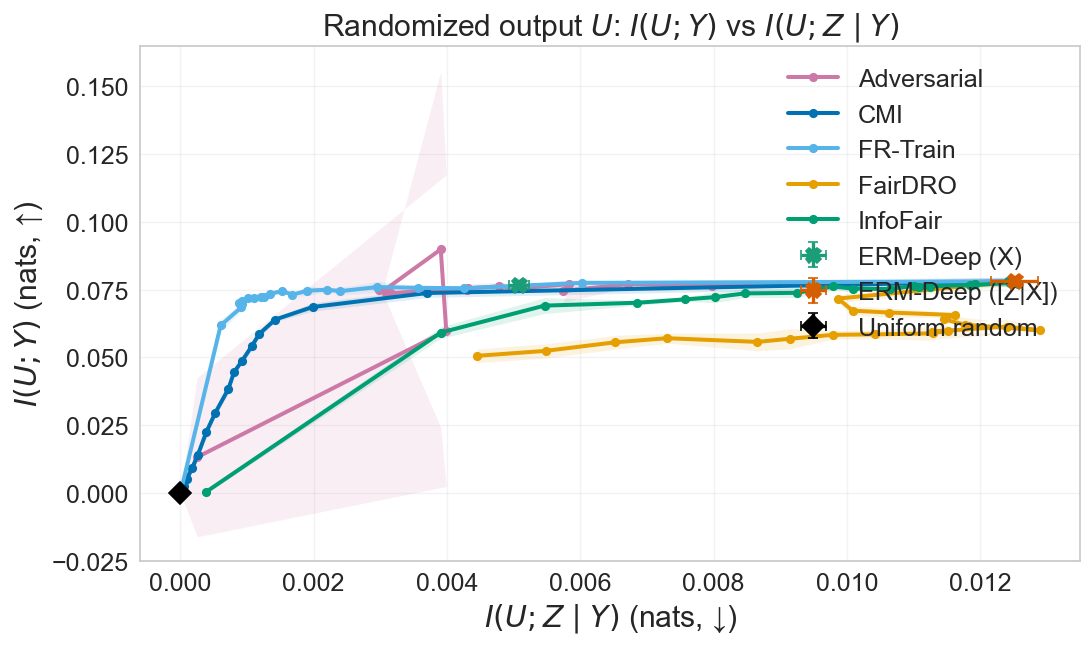}
    }
    \subfloat[Macro-AUROC transfer.]{
        \includegraphics[width=0.49\textwidth]{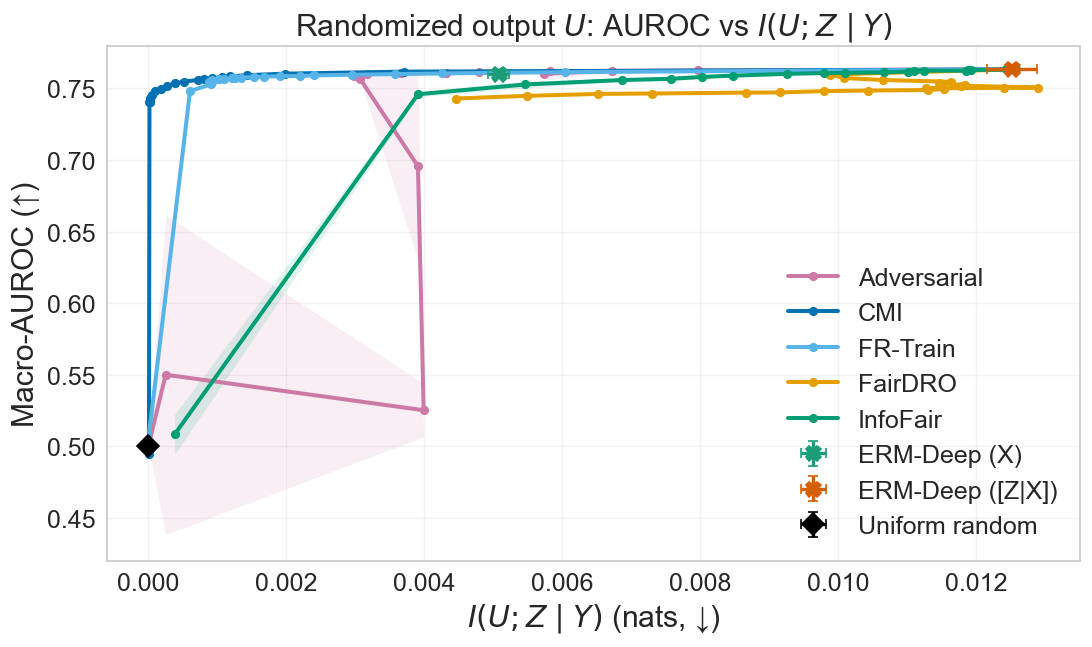}
    }

    \caption{
    \textbf{Large-scale ACSOccupation preprocessing variant.}
    We report randomized-policy results for the large ACSOccupation variant in which missing or structurally unavailable covariate values are retained as explicit categorical levels. 
    \textbf{Left:} Information-plane frontier, plotting \(I(\hat{Y};Y)\) against \(I(\hat{Y};Z\mid Y)\).
    \textbf{Right:} Macro-AUROC as a function of \(I(\hat{Y};Z\mid Y)\).
    CMI remains competitive in the high-utility, low-violation region. FR-Train is also a strong baseline on this large-scale variant, indicating that class-conditional adversarial training can be effective when the discriminator is well aligned with separation. The main distinction is that CMI directly optimizes the plug-in CMI violation, whereas FR-Train uses an auxiliary adversarial proxy.
    }
    \label{fig:ACS_large}
    \vspace{-10pt}
\end{figure*}

\subsection{ACS Ablation Studies}
\label{append:acs_ablation}

In this appendix, we provide additional ablation studies for the ACSOccupation experiment in Section~\ref{subsubsec:acs}. The goal is to verify that the observed behavior of Normalized CMI is not an artifact of a particular batch-size choice or loss-scaling convention. Since the main ACSOccupation experiment has \(|\mathcal Y|=20\) occupation classes and \(|\mathcal Z|=9\) race groups, the empirical CMI estimator must aggregate information across \(|\mathcal Y|\cdot|\mathcal Z|=180\) conditional cells. These ablations therefore test the numerical stability of the CMI regularizer in a genuinely multi-class and multi-group setting.

\paragraph{Batch-size sensitivity.}
Figure~\ref{fig:ACS_batchsize} reports the CMI frontier under different training batch sizes.
Across batch sizes \(512\), \(1024\), \(2048\), and \(4096\), the resulting curves exhibit the same qualitative behavior: increasing the trade-off parameter moves the model along the separation-utility trade-off from lower-violation, lower-utility points toward higher-utility points with larger separation violation. The frontiers are broadly consistent across batch sizes, suggesting that the empirical CMI signal remains stable even though each minibatch must estimate dependence across many $(Y,Z)$ cells. Larger batches provide somewhat smoother estimates of the conditional distributions, while smaller batches still recover the same overall separation-utility trend. This supports the robustness of the proposed estimator in the ACSOccupation setting.

\begin{figure*}[t]
    \centering
    \includegraphics[width=0.75\textwidth]{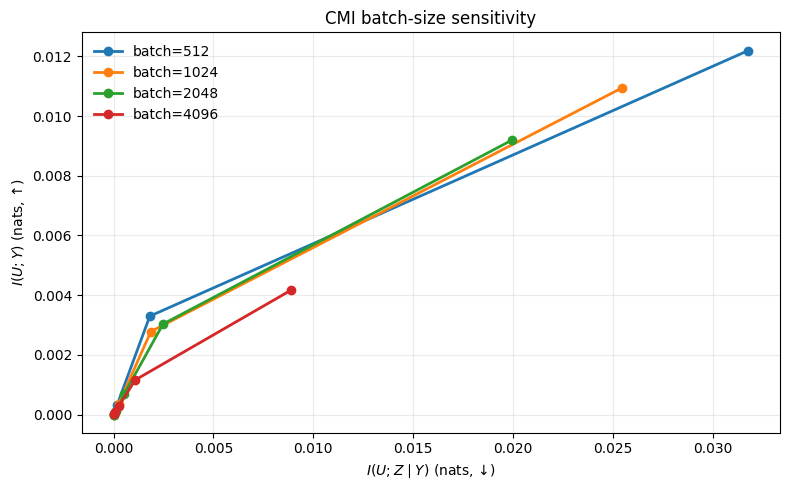}
    \caption{
    \textbf{Batch-size sensitivity of CMI on ACSOccupation.}
    We compare the information-plane frontiers obtained by training Normalized CMI with batch sizes $512$, $1024$, $2048$, and $4096$. The curves show consistent separation-utility behavior across batch sizes, indicating that the empirical CMI estimator is stable in the multi-class/multi-group ACS setting. Larger batches produce slightly smoother trajectories, while smaller batches still recover the same qualitative Pareto trend.
    }
    \label{fig:ACS_batchsize}
    \vspace{-10pt}
\end{figure*}

\paragraph{Effect of gradient normalization.}
Figure~\ref{fig:ACS_gradnorm} compares the proposed gradient-normalized CMI objective with a raw, unnormalized CMI penalty. The left panel shows the information-plane frontier, and the right panel shows the corresponding transfer to macro-AUROC. Gradient normalization substantially improves the usable range of the Pareto frontier. In the information plane, CMI with gradient normalization reaches much higher utility $I(\hat Y;Y)$ while maintaining a smooth monotone increase in separation violation. By contrast, the raw CMI penalty remains close to the origin and produces a compressed frontier, indicating that the unnormalized CMI term is poorly scaled relative to the prediction loss.

The macro-AUROC plot confirms the same conclusion in deployment-oriented metrics. With gradient normalization, the method rapidly reaches strong macro-AUROC while keeping $I(\hat Y;Z\mid Y)$ small, and then traces a stable high-utility frontier. The raw CMI objective achieves lower macro-AUROC and covers a much narrower part of the frontier. This ablation supports the design choice in Section~\ref{sec:algorithm}: normalizing the prediction and CMI gradients balances the two objectives and prevents the regularizer from being either numerically dominated by, or numerically dominating, the supervised loss.

\begin{figure*}[t]
    \centering
    \setlength{\tabcolsep}{1pt}
    
    \subfloat[Information plane.]{
        \includegraphics[width=0.49\textwidth]{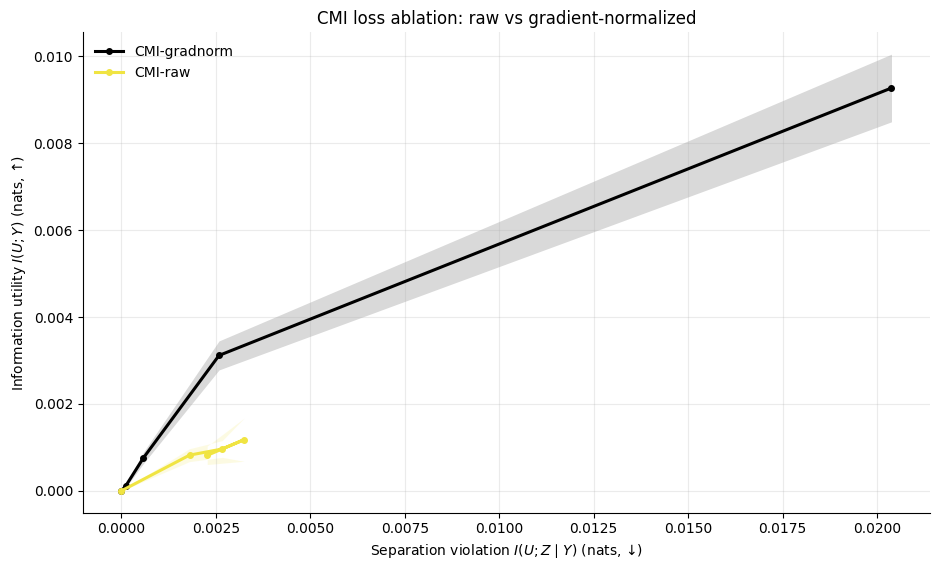}
    }
    \subfloat[Transfer to macro-AUROC.]{
        \includegraphics[width=0.49\textwidth]{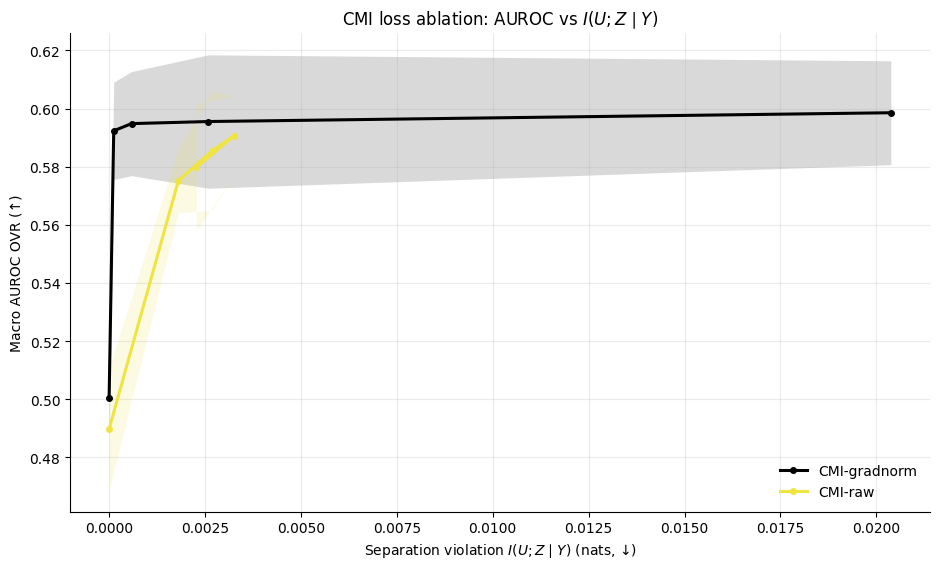}
    }

    \caption{
    \textbf{Ablation of gradient normalization on ACSOccupation.}
    We compare the proposed gradient-normalized CMI objective against a raw, unnormalized CMI penalty. 
    \textbf{Left:} Gradient normalization produces a much broader and smoother information-plane frontier, reaching higher utility $I(\hat Y;Y)$ while controlling $I(\hat Y;Z\mid Y)$.
    \textbf{Right:} The same improvement transfers to macro-AUROC: the gradient-normalized objective achieves stronger ranking performance in the low-violation regime and covers a larger useful range of the frontier. These results show that gradient normalization is important for stable optimization in the multi-class/multi-group setting.
    }
    \label{fig:ACS_gradnorm}
    \vspace{-10pt}
\end{figure*}

Together, these ablations support the robustness of the ACSOccupation results reported in Figure~\ref{fig:ACS_combined}. The batch-size study shows that the empirical CMI estimator remains stable across different minibatch resolutions, while the gradient-normalization study shows that the proposed normalized objective is essential for obtaining a smooth and useful Pareto traversal in the multi-class/multi-group regime.

\end{document}